\documentclass[12pt]{article}
\usepackage{amsmath}
\usepackage{graphicx}
\usepackage{enumerate}
\usepackage{natbib}
\usepackage{url} % not crucial - just used below for the URL 
\usepackage{hyperref}
\usepackage{amsmath}
\usepackage{amssymb}
\usepackage{amsthm}
\usepackage{bm,bbm}
\usepackage[ruled,vlined]{algorithm2e}
\usepackage{multirow}
\usepackage{enumitem}
\usepackage{float}
\usepackage{graphicx}
\usepackage{color}
\usepackage{graphicx}
\usepackage{wrapfig}
\usepackage{caption}
\usepackage{subcaption}

\newtheorem{theorem}{Theorem}
\newtheorem{lemma}{Lemma}
\newtheorem{corollary}{Corollary}
\newtheorem{assumption}{Assumption}
\newtheorem{proposition}{Proposition}

\newtheorem{remark}{Remark}

\DeclareMathOperator{\Var}{Var}
\DeclareMathOperator{\Cov}{Cov}
\DeclareMathOperator{\VEC}{Vec}

\newcommand{\blind}{0}

\begin{document}

\def\spacingset#1{\renewcommand{\baselinestretch}%
{#1}\small\normalsize} \spacingset{1}

%%%%%%%%%%%%%%%%%%%%%%%%%%%%%%%%%%%%%%%%%%%%%%%%%%%%%%%%%%%%%%%%%%%%%%%%%%%%%%

\if0\blind
{
  \title{\bf Covariance Estimators for the ROOT-SGD Algorithm in Online Learning}
  \author{Yiling Luo,
  %\thanks{The authors gratefully acknowledge \textit{please remember to list all relevant funding sources in the unblinded version}}
  %$^{\diamond}$ 
  Xiaoming Huo,
    %\thanks{The authors gratefully acknowledge \textit{please remember to list all relevant funding sources in the unblinded version}}
    %$^{\diamond}$
    Yajun Mei
    %\thanks{The authors gratefully acknowledge \textit{please remember to list all relevant funding sources in the unblinded version}} 
    %$^{\diamond}$\hspace{.2cm}
    \\
    School of Industrial and Systems Engineering%$^{\diamond}$
    \\
    Georgia Institute of Technology}
\date{}
  \maketitle
} \fi

\if1\blind
{
  \bigskip
  \bigskip
  \bigskip
  \begin{center}
    {\LARGE\bf Covariance Estimators for the ROOT-SGD Algorithm in Online Learning}
\end{center}
  \medskip
} \fi

\bigskip
\begin{abstract}
Online learning naturally arises in many statistical and machine learning problems. 
The most widely used methods in online learning are stochastic first-order algorithms. 
Among this family of algorithms, there is a recently developed algorithm, Recursive One-Over-T SGD (ROOT-SGD). 
ROOT-SGD is advantageous in that it converges at a non-asymptotically fast rate, and its estimator further converges to a normal distribution. 
However, this normal distribution has unknown asymptotic covariance; thus cannot be directly applied to measure the uncertainty. 
To fill this gap, we develop two estimators for the asymptotic covariance of ROOT-SGD. 
Our covariance estimators are useful for statistical inference in ROOT-SGD. 
Our first estimator adopts the idea of plug-in. 
For each unknown component in the formula of the asymptotic covariance, we substitute it with its empirical counterpart. 
The plug-in estimator converges at the rate $\mathcal{O}(1/\sqrt{t})$, where $t$ is the sample size. 
Despite its quick convergence, the plug-in estimator has the limitation that it relies on the Hessian of the loss function, which might be unavailable in some cases. 
Our second estimator is a Hessian-free estimator that overcomes the aforementioned limitation. 
The Hessian-free estimator uses the random-scaling technique, and we show that it is an asymptotically consistent estimator of the true covariance. 
\end{abstract}

\noindent%
\vfill

\newpage
\spacingset{1.9} % DON'T change the spacing!

\section{Introduction}

%This paper aims at efficiently estimating the (co)variance of an estimator in online learning.
%In this section, we first introduce what is online learning, and explain why we need covariance estimation in such a setting.
%Then we review the exiting literature on this problem.
%Finally, we summarize our contribution based on the literature.
%
%\subsection{Problem of Interest}

In statistics, a parameter estimation problem often leads to an optimization problem. 
Let $\bm{\theta}^*$ be the true model parameters. 
In many widely-adopted models, $\bm{\theta}^*$ is the minimizer of a convex population risk function $F(\bm{\theta}):= \mathbbm{E}_{\bm{x}\sim P} f(\bm{\theta};\bm{x})$, where $\bm{x}\in\mathcal{R}^d$ denotes the random sample from distribution $P$, and $f(\bm{\theta};\bm{x})$ is the loss function at $\bm{x}$. 
This is the assumption that we uphold throughout this paper. 
Mathematically, we have that
\begin{equation}\label{eq:obj1}
   \bm{\theta}^* = \arg\min_{\bm{\theta}\in\mathcal{R}^p }F(\bm{\theta}).
\end{equation}

In \textit{offline learning}, one is given a fixed number of samples $\bm{x}_1,\cdots, \bm{x}_t$, for parameter estimation. 
In general, the population risk function $F$ is unknown, so people use the empirical risk function as an estimate for $F$. 
This gives us the Empirical Risk Minimization (ERM) objective:
\begin{equation}\label{eq:erm}
   \arg\min_{\bm{\theta}} \frac{1}{t}\sum_{i=1}^t f(\bm{\theta};\bm{x}_i).
\end{equation}
Despite its popularity in statistics and machine learning, it is clear from \eqref{eq:erm} that when the data size is large, working with ERM will be both computational and memory expensive. 
In this case, the online learning framework can be considered. 

In \textit{Online learning}, one has i.i.d. data samples $\{\bm{x}_i, i=1,2,\cdots\}$ come in a sequence. 
Given the data stream, one need to sequentially update the current parameter estimate $\widehat{\bm{\theta}}_{i}$ based on the past estimate $\widehat{\bm{\theta}}_{i-1}$ and the incoming data sample $\bm{x}_i$. 
The goal is that the parameter estimate will approximate $\bm{\theta}^*$ in \eqref{eq:obj1} more accurately as we get more and more data. 

The most well-known algorithm in online learning is the Robbins-Monro algorithm, also known as the Stochastic Gradient Descent (SGD) \citep{robbins1951stochastic}.
The SGD algorithm updates the parameter estimate as follows:
\begin{equation}\label{eq:oSGD}
\widehat{\bm{\theta}}_{i} = \widehat{\bm{\theta}}_{i-1} - \eta_i \nabla_{\bm{\theta}} f(\widehat{\bm{\theta}}_{i-1};\bm{x}_{i}),    
\end{equation}
where $\nabla_{\bm{\theta}} f(\cdot;\cdot)$ is the first-order derivative of $f(\cdot;\cdot)$ with respect to the first argument; we also simplify $\nabla_{\bm{\theta}} f(\cdot;\cdot)$ as $\nabla f(\cdot;\cdot)$ in the rest of this paper. 
And $\eta_i$ is the step size. 
In general, one needs to take a diminishing sequence of step sizes such that $\eta_i \stackrel{i \to\infty}{\to} 0 $ to make SGD converge. 
Note that the SGD updates the parameter estimate using only the first-order derivative of the stochastic risk function, i.e., $\nabla_{\bm{\theta}} f(\widehat{\bm{\theta}}_{i-1};\bm{x}_{i})$. 
One refers to the algorithms that only use the stochastic first-order information for parameter update as stochastic first-order algorithms. 

In the family of stochastic first-order algorithms, a lot of advancements have been made over the SGD algorithm by using the techniques such as averaging \citep{polyak1992acceleration}, acceleration \citep{kingma2014adam,dozat2016incorporating}, and variance reduction \citep{johnson2013accelerating}. 
There are two main goals in those algorithms: the value of the objective function converges to the minimum quickly, as illustrated by a tight upper bound of the quantity $\mathbbm{E} F(\widehat{\bm{\theta}}_t) - F(\bm{\theta}^*)$, where $\widehat{\bm{\theta}}_t$ is the estimator from the algorithm given $t$ data samples;  
the estimator $\widehat{\bm{\theta}}_t$ converges in distribution, and one can develop the theoretical expression of such a distribution. 
Recently, there is a new algorithm --- ROOT-SGD \citep{li2022root} --- that achieves both goals. 
We will discuss ROOT-SGD in the following. 

The main advantage of ROOT-SGD, as compared to other online first-order algorithms, is that it achieves non-asymptotically quick convergence and asymptotically normal distribution in a single algorithm.  
Here we review such properties of ROOT-SGD for a strongly-convex population risk function that are established in \citep{li2022root}. 
The non-asymptotic convergence result says that the value of the objective function in \eqref{eq:obj1} converges at rate $\mathcal{O}(1/t)$ for the sample size $t$, which matches the optimal statistical rate up to a constant factor. 
The asymptotic convergence result says that the estimator $\widehat{\bm{\theta}}_{t}$ converges in distribution to a multivariate normal distribution. 
In particular, one has $\sqrt{t}(\widehat{\bm{\theta}}_t - \bm{\theta}^*) \stackrel{d}{\to} \mathcal{N}(0,\Sigma)$, where the theoretical expression of the asymptotic covariance $\Sigma$ has been derived. 

The asymptotic covariance of the ROOT-SGD estimator, $\Sigma$, is unknown since it depends on the unknown data distribution $P$ and the unknown population risk function $F$. 
However, the asymptotic covariance is important for uncertainty quantification. 
We develop two estimators for the asymptotic covariance, $\Sigma$, in ROOT-SGD. 
We show that our estimators are consistent. 
Using our estimators, one has reliable uncertainty measurement. 
With such uncertainty measurement, one will be able to conduct statistical inference, e.g., hypothesis testing for the true parameter $\bm{\theta}^*$, or constructing confidence intervals. 

We now give a high-level summary of the two estimators and discuss their pros and cons. 
Our first estimator uses the plug-in technique. 
For each component in $\Sigma$, we plug in the data samples to get the empirical version of it. 
We also made several adjustments to the empirical estimator so that it is well-defined and can be computed in an online fashion. 
Our final plug-in estimator converges to the true asymptotic covariance at a rate $\mathcal{O}(1/\sqrt{t})$, where $t$ is the number of data samples. 
The plug-in estimator has the advantage of fast convergence. 
However, it has the limitation of high computation burden and requires the stochastic Hessian. 
Thus the plug-in estimator might be hard or even impossible to compute in some cases. 
Our second estimator uses the random-scaling technique from the martingale theory. 
It only requires the ROOT-SGD point estimator sequences, so the estimator is Hessian-free. 
We call it a Hessian-free estimator. 
The Hessian-free estimator can also be computed in an online way and is asymptotically consistent. 
When compared with the plug-in estimator, the Hessian-free estimator has the advantage of lower computational cost. 
However, the convergence rate of the Hessian-free estimator is unknown. 
Thus there is no guarantee of a comparable convergence rate.

\subsection{Existing Covariance Estimators}\label{sec:1.1}
There is no work for estimating the covariance of the ROOT-SGD algorithm. 
But for other online stochastic algorithms, such as the Averaging SGD (ASGD) algorithm, there is some literature on covariance estimation. 
In this section, we review the existing covariance estimation techniques. 
We also review the covariance estimators for the ASGD algorithm in detail. 

To provide a context of covariance estimation for the ASGD algorithm, we first quickly review the ASGD algorithm and its asymptotic distribution. 
Then we summarize the techniques in the current literature for covariance estimation in online stochastic algorithms and their detailed application in ASGD algorithms. 
These techniques motivate us to develop covariance estimators for ROOT-SGD. 

The ASGD estimator is as follows. 
Recall that SGD updates as 
\[\widehat{\bm{\theta}}_{i} = \widehat{\bm{\theta}}_{i-1} - \eta_i \nabla_{\bm{\theta}} f(\widehat{\bm{\theta}}_{i-1};\bm{x}_{i}),\]
then ASGD estimator, denoted by $\Bar{\bm{\theta}}_i$'s, initialize as $\Bar{\bm{\theta}}_1 = \widehat{\bm{\theta}}_{1}$ and update as
\begin{equation}\label{eq:asgd_recursive}
\Bar{\bm{\theta}}_i = \frac{i-1}{i}\Bar{\bm{\theta}}_{i-1} + \frac{1}{i}\widehat{\bm{\theta}}_{i}.   
\end{equation}
One may notice that such update is equivalent to $\Bar{\bm{\theta}}_i = \frac{1}{i}\sum_{j=1}^i \widehat{\bm{\theta}}_{j}$, i.e., the ASGD estimator is the averaging of all past SGD estimators. 
In equation \eqref{eq:asgd_recursive}, the ASGD estimator is updated recursively: the current estimator $\Bar{\bm{\theta}}_i$ only depends on the most recent estimator $\Bar{\bm{\theta}}_{i-1}$ and the current SGD estimator $\widehat{\bm{\theta}}_{i}$. 
Such recursive update avoids storing all past SGD estimators, so the ASGD is memory efficient and compatible with the online learning scheme. 

The asymptotic distribution of the ASGD estimator is established for a special scheme of decreasing step size as follows. 
When $\eta_i = \eta i^{-\alpha}$ for $\alpha\in (1/2,1)$, Theorem 2 of \citep{polyak1992acceleration} derives the asymptotic distribution of the ASGD estimator as:
    \begin{equation}
    \sqrt{t}(\bar{\bm{\theta}}_t - \bm{\theta}^* ) \stackrel{d}{\rightarrow} \mathcal{N}(0,\Sigma_{ASGD}),
    \end{equation}
where the asymptotic covariance matrix $\Sigma_{ASGD} = A^{-1}S A^{-1}$ for $A = \nabla^2 F(\bm{\theta}^*) $ (Hessian of $F(\cdot)$ evaluated at $\bm{\theta}^*$), $S = \mathbbm{E}_{x\sim P}[\nabla f (\bm{\theta}^*;\bm{x}) \nabla f (\bm{\theta}^*;\bm{x})^T ]$. \footnote{Note that S is a column vector times a row vector, which gives a square and symmetric matrix.} 

In the current literature, there are two classes of estimators for the asymptotic covariance, $\Sigma_{ASGD}$, based on the information required. 
We review these two classes for the ASGD algorithm. 

The first class is the \textit{plug-in} estimator \citep{chen2020statistical,chen2021online}, which requires computing the stochastic Hessian $\nabla_{\bm{\theta}}^2 f(\cdot;\cdot)$. 
The plug-in estimator estimates each component in the formula of the asymptotic covariance by their empirical counterparts. 
For example, in the asymptotic covariance $\Sigma_{ASGD} = A^{-1}S A^{-1}$, the plug-in estimator replaces $A$ and $S$ by their estimators. 
Since $\Sigma_{ASGD}$ includes $A = \nabla^2 F(\cdot)$, the plug-in estimator then requires $\nabla_{\bm{\theta}}^2 f(\widehat{\bm{\theta}}_{i-1};\bm{x}_{i})$ to compute the empirical counterpart of $A$. 

The plug-in covariance estimator for the ASGD is as follows. 
For $\Sigma_{ASGD} = A^{-1}S A^{-1}$, \citep{chen2020statistical} estimates $A$ and $S$ by their empirical counterparts as: 
\begin{equation}\label{eq:estiAS}
    \widehat{A}_t = (1/t)\sum_{i=1}^t \nabla^2 f(\widehat{\bm{\theta}}_{i-1};\bm{x}_i), \quad
\widehat{S}_t = (1/t)\sum_{i=1}^t \nabla f(\widehat{\bm{\theta}}_{i-1};\bm{x}_i)\nabla f(\widehat{\bm{\theta}}_{i-1};\bm{x}_i)^T. 
\end{equation} 
Note that $\widehat{A}_t$ and $\widehat{S}_t$ plug in the estimator sequence $\widehat{\bm{\theta}}_{i}$'s from SGD, instead of the ASGD sequence $\bar{\bm{\theta}}_i$'s. 
The plug-in estimator for $\Sigma_{ASGD}$ is $\widehat{\Sigma}_{t,ASGD} = (\widehat{A}_t)^{-1} \widehat{S}_t (\widehat{A}_t)^{-1}$.
Both $\widehat{A}_t$ and $\widehat{S}_t$ can be computed in an online fashion, so this plug-in covariance estimator is an online estimator.
Paper \citep{chen2020statistical} further shows that the plug-in estimator $\widehat{\Sigma}_{t,ASGD}$ converges to $\Sigma_{ASGD}$ as
\begin{equation}\label{eq:covar_sgd}
    \mathbbm{E}\|\Sigma_{ASGD} - \widehat{\Sigma}_{t,ASGD}\|_2\lesssim t^{-\alpha/2}.
\end{equation}
Since we have $\alpha\in (1/2,1)$, the convergence rate is strictly slower than $\mathcal{O}(t^{-1/2})$. 

The second class is the \textit{Hessian-free} estimator \citep{chen2020statistical,zhu2021online,chen2021online,su2018uncertainty,lee2021fast,kiefer2000simple} which does not use the stochastic Hessian $\nabla_{\bm{\theta}}^2 f(\cdot;\cdot)$. 
The Hessian-free estimator can be computed purely based on the estimator sequences ($\widehat{\bm{\theta}}_i$'s or $\bar{\bm{\theta}}_i$'s in the ASGD algorithm). 
The idea is similar to that of sample covariance:
Suppose $\bm{a}_i, i=1,\cdots,t$ are i.i.d. samples from a distribution with finite second moment, then $(1/t)\sum_{i=1}^t (\bm{a}_i - \bar{\bm{a}}_t)(\bm{a}_i - \bar{\bm{a}}_t)^T$ is an asymptotically consistent estimator for the population covariance, where $\bar{\bm{a}}_t = (1/t)\sum_{i=1}^t \bm{a}_i$ is the sample mean. 
Then, for $\widehat{\bm{\theta}}_i$ (or $\Bar{\bm{\theta}}_i$), one can also construct a sample covariance estimator by treating each $\widehat{\bm{\theta}}_i$ (or $\Bar{\bm{\theta}}_i$) as a sample.
However, since $\widehat{\bm{\theta}}_i$ and $\widehat{\bm{\theta}}_{i-1}$ are highly correlated, a vanilla sample covariance as $(1/t)\sum_{i=1}^t (\widehat{\bm{\theta}}_i - \bar{\bm{\theta}}_t)(\widehat{\bm{\theta}}_i - \bar{\bm{\theta}}_t)^T$ will not be asymptotically consistent. 
To solve this issue, certain modifications will be required. 
We review them in the following. 

There are two types of Hessian-free estimators for ASGD, which use different ways to deal with the autocorrelation structure in the $\{\widehat{\bm{\theta}}_i\}$ sequence.

The first approach is the batch-mean estimator that uses batch mean as decorrelated samples in place of the SGD estimators, $\widehat{\bm{\theta}}_i$'s, to calculate the sample covariance \citep{chen2020statistical,zhu2021online}.  
The idea of batch-mean originates from the covariance estimation for time-homogeneous Markov-chain, which is an auto-correlated sequence \citep{politis1999subsampling,lahiri2003resampling}.
The batch-mean estimator is constructed as follows:
Divide $\{\widehat{\bm{\theta}}_i\}_{i=1,\cdots,t}$ in to $m$ batches as $\underbrace{\widehat{\bm{\theta}}_1,\cdots,\widehat{\bm{\theta}}_{b_1}},\underbrace{\widehat{\bm{\theta}}_{b_1+1},\cdots,\widehat{\bm{\theta}}_{b_1 + b_2}}, \cdots,\underbrace{\widehat{\bm{\theta}}_{b_1+\cdots+b_{m-1} + 1},\cdots,\widehat{\bm{\theta}}_{b_1+\cdots+b_{m}}}$, where $b_j > 0$ is the batch size of the $j$th batch, and $\sum_{j=1}^m b_j = t$. 
Denote the batch mean for the $j$th batch as $\Bar{\bm{\theta}}_j^{\flat} := (1/b_j)\sum_{i=b_1 + \cdots + b_{j-1}+ 1} ^{b_1 + \cdots + b_{j}} \widehat{\bm{\theta}}_i$. 
Finally, the batch-mean covariance estimator is calculated as the (weighted) sample covariance of batch means $\Bar{\bm{\theta}}_j^{\flat}$'s, denoted as $\widehat{\Sigma}_{t,BM}$.
When the batch size is large enough, the batch means are sufficiently decorrelated, thus the $\widehat{\Sigma}_{t,BM}$ will be a consistent estimator of $\Sigma_{ASGD}$ as shown by \citep{chen2020statistical,zhu2021online}. 

The second approach is the random-scaling estimator that uses the estimators $\widehat{\bm{\theta}}_i$'s (or $\Bar{\bm{\theta}}_i$'s for ASGD algorithm) 
to calculate a sample covariance, but carefully analyze how the sample autocorrelation affects the asymptotic distribution \citep{chen2021online,lee2021fast}.
The random-scaling technique is also referred to as robust testing \citep{kiefer2000simple,abadir2002simple}. 
For the ASGD algorithm, \citep{lee2021fast} computes a random-scaling estimator as $\bar{V}_{t,RS} = \frac{1}{t^2}\sum_{i=1}^t i^2 (\bar{\bm{\theta}}_i - \bar{\bm{\theta}}_t)(\bar{\bm{\theta}}_i - \bar{\bm{\theta}}_t)^T$. 
Compared with a consistent covariance estimator $\widehat{\Sigma}_t$ for $\Sigma_{ASGD}$ such that $\sqrt{t}(\widehat{\Sigma}_t)^{-1/2}(\bar{\bm{\theta}}_t - \bm{\theta}^* ) \stackrel{d}{\rightarrow} \mathcal{N}(0,I)$, one instead has that $\sqrt{t}(\bar{V}_{t,RS})^{-1/2}(\bar{\bm{\theta}}_t - \bm{\theta}^* )$ converges to another distribution. 
Paper \citep{lee2021fast} further shows that such distribution is well-defined and independent of the specific problem setting (i.e., does not depend on $f(\cdot;\cdot)$ or $P$). 
In this way, one can construct a consistent covariance estimator of $\sqrt{t}(\bar{\bm{\theta}}_t - \bm{\theta}^* )$ using the random-scaling estimator $\bar{V}_{t,RS}$.

\subsection{Overview of the ROOT-SGD Algorithm}\label{sec:1.2} 
In this section, we review the ROOT-SGD algorithm. 
We also review the asymptotic distribution of the estimator from the ROOT-SGD algorithm. 
The asymptotic distribution is crucial to our analysis since our covariance estimators are based on the covariance of such distribution. 

The ROOT-SGD algorithm  \citep{li2022root} applies the idea of variance reduction \citep{johnson2013accelerating,NIPS2014_ede7e2b6,pmlr-v70-nguyen17b,NEURIPS2018_1543843a} to online learning, so it converges fast for a fixed step size. 
The ROOT-SGD algorithm is given in Algorithm \ref{algo:rootsgd}.

\begin{algorithm}[ht]
\textbf{Initialization}: $\widehat{\bm{\theta}}_0$; set $v_1$ to be $\nabla f(\widehat{\bm{\theta}}_0;\bm{x}_1)$; 
choose a step size $\eta$ and burn-in period $B$\;
\For{i = 1,$\cdots$,B-1}{
$$\bm{v}_i = \nabla f(\widehat{\bm{\theta}}_{i-1};\bm{x}_i) + \frac{i-1}{i} (\bm{v}_{i-1} - \nabla f(\widehat{\bm{\theta}}_{i-2};\bm{x}_i)),$$
$$\widehat{\bm{\theta}}_i = \widehat{\bm{\theta}}_{i-1}.$$
}
\For{i = B,$\cdots$,t}{
$$\bm{v}_i = \nabla f(\widehat{\bm{\theta}}_{i-1};\bm{x}_i) + \frac{i-1}{i} (\bm{v}_{i-1} - \nabla f(\widehat{\bm{\theta}}_{i-2};\bm{x}_i)),$$
$$\widehat{\bm{\theta}}_i = \widehat{\bm{\theta}}_{i-1} - \eta \bm{v}_i .$$
}
\textbf{Output}: $\widehat{\bm{\theta}}_t$.
\caption{ROOT-SGD \citep{li2022root}}
\label{algo:rootsgd}
\end{algorithm}

We explain why ROOT-SGD is a variance-reduced algorithm. 
In Algorithm \ref{algo:rootsgd}, the intermediate quantity $\bm{v}_i$ serves as a variance-reduced estimator of the population gradient. 
The estimator $\bm{v}_i$ is developed as follows. 
One starts with an estimator $\bm{v}_i^\prime$ of $\nabla F(\widehat{\bm{\theta}}_{i-1})$ such that the error $\bm{v}_i^\prime - \nabla F(\widehat{\bm{\theta}}_{i-1})$ is the average of the errors $\{\nabla f(\widehat{\bm{\theta}}_{j-1}; \bm{x}_j) - \nabla F(\widehat{\bm{\theta}}_{j-1})\}_{j=1,\cdots,i}$. 
Mathematically, 
\begin{equation}\label{eq:esti_1}
    \bm{v}_i^\prime - \nabla F(\widehat{\bm{\theta}}_{i-1}) = \frac{1}{i}\sum_{j=1}^i (\nabla f(\widehat{\bm{\theta}}_{j-1}; \bm{x}_j) - \nabla F(\widehat{\bm{\theta}}_{j-1})). 
\end{equation}
Then the term $\bm{v}_i'$ is an unbiased estimator of $\nabla F(\widehat{\bm{\theta}}_{i-1})$. 
Moreover, assume the conditional covariance $\Cov(\nabla f(\widehat{\bm{\theta}}_{j-1}; \bm{x}_j) | \widehat{\bm{\theta}}_{j-1})\prec \Sigma'$ almost surely, then $\Cov(\bm{v}_i^\prime)\prec \Sigma'/i$, which is a reduced covariance. 
One can further rewrite \eqref{eq:esti_1} so that $\bm{v}_i^\prime - \nabla F(\widehat{\bm{\theta}}_{i-1})$ only depends on $\bm{v}_{i-1}^\prime - \nabla F(\widehat{\bm{\theta}}_{i-2})$ and $\nabla f(\widehat{\bm{\theta}}_{i-1};\bm{x}_i) - \nabla F(\widehat{\bm{\theta}}_{i-1})$ as follows:
\begin{align}
    \bm{v}_i^\prime - \nabla F(\widehat{\bm{\theta}}_{i-1}) &= \frac{1}{i} (\nabla f(\widehat{\bm{\theta}}_{i-1};\bm{x}_i) - \nabla F(\widehat{\bm{\theta}}_{i-1})) + \frac{1}{i}\sum_{j=1}^{i-1} (\nabla f(\widehat{\bm{\theta}}_{j-1};\bm{x}_j) - \nabla F(\widehat{\bm{\theta}}_{j-1}))\nonumber\\
    %\Longrightarrow& \bm{v}_t^\prime - \nabla F(\widehat{\bm{\theta}}_{t-1}) 
    &= \frac{1}{i} (\nabla f(\widehat{\bm{\theta}}_{i-1};\bm{x}_i) - \nabla F(\widehat{\bm{\theta}}_{i-1})) + \frac{i - 1}{i}(\bm{v}_{i-1}^\prime - \nabla F(\widehat{\bm{\theta}}_{i-2})).\label{eq:esti_2}
\end{align}
Since $\nabla F(\widehat{\bm{\theta}}_{i-1})$ and $\nabla F(\widehat{\bm{\theta}}_{i-2})$ in \eqref{eq:esti_2} are unknown, \citep{li2022root} replaces them with the unbiased estimators $\nabla f(\widehat{\bm{\theta}}_{i-1};\bm{x}_i)$ and $\nabla f(\widehat{\bm{\theta}}_{i-2};\bm{x}_i)$, respectively. 
This gives us the gradient estimator in the ROOT-SGD algorithm that
\begin{equation}
    \bm{v}_i = \nabla f(\widehat{\bm{\theta}}_{i-1};\bm{x}_i) + \frac{i-1}{i} (\bm{v}_{i-1} - \nabla f(\widehat{\bm{\theta}}_{i-2};\bm{x}_i)). \label{eq:esti_3}
\end{equation}
Then the ROOT-SGD performs update on parameter $\widehat{\bm{\theta}}_i$ using $\bm{v}_i$. 

We now review the asymptotic distribution of ROOT-SGD in \citep{li2022root}. 
Our covariance estimators are all based on this asymptotic distribution. 
Under certain regularity conditions (see details in Lemma \ref{lem:asymptotic_root}), for a proper step size $\eta$ and burn-in period $B$, \citep{li2022root} proves that  
\begin{equation}\label{eq:asymptotic_root}
    \sqrt{t}(\widehat{\bm{\theta}}_t - \bm{\theta}^*) \stackrel{d}{\to} \mathcal{N}(0,\Sigma)
\end{equation}
for $\Sigma = A^{-1} (S + \mathbbm{E}[\Xi_{\bm{x}}(\bm{\theta}^*) \Lambda \Xi_{\bm{x}}(\bm{\theta}^*)]) A^{-1}$, where $\Xi_{\bm{x}}(\bm{\theta}) = \nabla^2 f(\bm{\theta};\bm{x}) - \nabla^2 F(\bm{\theta})$, $A = \nabla^2 F(\bm{\theta}^*) $, $S = \mathbbm{E}_{x\sim P}[\nabla f (\bm{\theta}^*;\bm{x}) \nabla f (\bm{\theta}^*;\bm{x})^T ]$ and $\Lambda$ is determined by solving the following matrix equation of $\Lambda$ (a.k.a. \textit{the modified Lyapunov equation}):
\begin{equation}
    \Lambda A + A\Lambda - \eta \mathbbm{E}[\Xi_{\bm{x}}(\bm{\theta}^*)\Lambda \Xi_{\bm{x}}(\bm{\theta}^*)] - \eta A \Lambda A = \eta S. \label{eq:lambda}
\end{equation}

\subsection{Convergence of ROOT-SGD for Two Examples}\label{sec:1.3}
Besides its fast convergence and asymptotic normality, ROOT-SGD is also advantageous in that we can explicitly illustrate its convergence for some learning examples. 
In this section, we explain the convergence property of the ROOT-SGD algorithm for two learning examples: the normal mean estimation and the natural parameter estimation for the exponential family distribution.

\paragraph{Normal Mean Estimation.} 
We now describe the normal mean estimation problem and show the convergence of the ROOT-SGD in this situation. 

Consider the problem of estimating the mean vector of a multivariate normal distribution. 
Suppose the data $\bm{x}_i \stackrel{i.i.d.}{\sim} N(\bm{\bm{\theta}}^*; I_p)$. 
Let the loss function be the negative log-likelihood:
\[f(\bm{\theta};\bm{x}) = \frac{1}{2}\| \bm{\theta} - \bm{x}\|_2^2.\]
Then we have $\nabla f(\bm{\theta};\bm{x}) = \bm{\theta} - \bm{x}$. 
It is not hard to check that the ROOT-SGD update is: 
\begin{equation}\label{eq:root_normal_mean_update}
    \widehat{\bm{\theta}}_t = \widehat{\bm{\theta}}_{t-1} - \eta(\widehat{\bm{\theta}}_{t-1} - \frac{1}{t}\sum_{i = 1}^t \bm{x}_i) = (1 - \eta)\widehat{\bm{\theta}}_{t-1} + \frac{\eta}{t} \sum_{i = 1}^t \bm{x}_i.
\end{equation}
We now compare the ROOT-SGD estimator with the common estimator of sample mean $\bar{\bm{x}}_t = \frac{1}{t}\sum_{i = 1}^t \bm{x}_i$, which is a consistent estimator for the true mean vector $\bm{\theta}^*$. 
Then \eqref{eq:root_normal_mean_update} can be rewritten as:
\begin{equation}\label{eq:root_normal_mean_update2}
    \widehat{\bm{\theta}}_t - \bar{\bm{x}}_t = (1 - \eta)(\widehat{\bm{\theta}}_{t-1}-\bar{\bm{x}}_t).
\end{equation}
Thus, for the step size $\eta \in (0,1]$, the current estimator will be closer to the sample mean than the last estimator. 
When the sample mean $\bar{\bm{x}}_t$ is close to the true parameter $\bm{\theta}^*$, we will have nearly linear convergence of the ROOT-SGD estimator $\widehat{\bm{\theta}}_t$. 
Our calculation above explicitly shows an advantage of the ROOT-SGD algorithm: it converges fast to an optimal estimator.

\paragraph{Exponential Family Model.} 
We now describe the problem of parameter estimation in the exponential family model and show the convergence of the ROOT-SGD algorithm in this case.

Consider estimating the natural parameter $\bm{\theta}\in R^p$ in an exponential family model, where the probability density function is given by:
\[L(\bm{x};\bm{\theta}) = h(\bm{x})\exp[\langle\bm{\theta}, \bm{T}(\bm{x})\rangle - B(\bm{\theta})].\]
Let the risk function be the negative log-likelihood
$f(\bm{\theta};\bm{x}) = -\langle\bm{\theta}, \bm{T}(\bm{x})\rangle + B(\bm{\theta})$, then the gradient is $\nabla f(\bm{\theta};\bm{x}) = - \bm{T}(\bm{x}) + \nabla B(\bm{\theta})$. 

Then we write the ROOT-SGD update for the exponential family model. 
For the intermediate quantity $\bm{v}_t$, we can solve its recursive relationship to get:
\begin{align*}
\bm{v}_t &= \nabla B(\widehat{\bm{\theta}}_{t-1}) - \bm{T}(\bm{x}_t) + \frac{t-1}{t} (\bm{v}_{t-1} - \nabla B(\widehat{\bm{\theta}}_{t-2}) + \bm{T}(\bm{x}_t))\\
&= \nabla B(\widehat{\bm{\theta}}_{t-1}) - \frac{1}{t}\bm{T}(\bm{x}_t) + \frac{t-1}{t} (\bm{v}_{t-1} - \nabla B(\widehat{\bm{\theta}}_{t-2}))\\
&= \nabla B(\widehat{\bm{\theta}}_{t-1}) - \frac{1}{t}\sum_{i=1}^t \bm{T}(\bm{x}_i),\end{align*}
The parameter update using such $\bm{v}_t$ is
\begin{equation}\label{eq:exponential_update}
    \widehat{\bm{\theta}}_t = \widehat{\bm{\theta}}_{t-1} - \eta \bm{v}_t = \widehat{\bm{\theta}}_{t-1} - \eta \left(\nabla B(\widehat{\bm{\theta}}_{t-1}) - \frac{1}{t}\sum_{i=1}^t \bm{T}(\bm{x}_i)\right).
\end{equation}

Finally, we compare $\widehat{\bm{\theta}}_t$ with the maximum likelihood estimator (MLE). 
The MLE is equivalent to the estimator in \eqref{eq:erm}, and it is asymptotically efficient. 
In particular, the MLE converges to the true parameter at rate $\mathcal{O}(1/t)$. 
For exponential family model, the MLE $\bm{\theta}^*_t$ satisfies 
\[\nabla B(\bm{\theta}^*_t) = \frac{1}{t}\sum_{i = 1}^t \bm{T}(\bm{x}_i).\]
Thus the ROOT-SGD update in \eqref{eq:exponential_update} can be rewritten as 
\[\widehat{\bm{\theta}}_t = \widehat{\bm{\theta}}_{t-1} - \eta \left(\nabla B(\widehat{\bm{\theta}}_{t-1}) - \nabla B(\bm{\theta}^*_t)\right).\]
Under some regularity conditions on $B(\cdot)$, we have Proposition \ref{prop:root_sgd_convergence_exponential} holds. 
By Proposition \ref{prop:root_sgd_convergence_exponential}, when the MLE $\bm{\theta}^*_t$ is close to the true parameter $\bm{\theta}^*$, we will have nearly linear convergence of the ROOT-SGD estimator. 
In this way, we have a straightforward explanation for the convergence of ROOT-SGD that could not be applied to (A)SGD. 

\begin{proposition}\label{prop:root_sgd_convergence_exponential}
When $B(\bm{\theta})$ is $\mu$-strongly convex and $l$-smooth for $0<\mu<l$, we can take $\eta \in (0,\frac{2\mu}{l(\mu+l)}]$, so that $\exists \alpha \in (0, 1/\eta)$ s.t. 
\begin{equation}\label{eq:linear_conv_exponential_family}
    \|\widehat{\bm{\theta}}_t - \bm{\theta}^*_t\|_2 \leq (1 - \eta\alpha)\|\widehat{\bm{\theta}}_{t-1} - \bm{\theta}^*_t\|_2. 
\end{equation}
\end{proposition}
The proof for Proposition \ref{prop:root_sgd_convergence_exponential} is in Appendix \ref{app:prof_root_sgd_convergence_exponential}. 

%Thus the convergence of $\widehat{\bm{\theta}}_t$ to $\bm{\theta}^*$ is dominated by the convergence of $\bm{\theta}^*_t$ to $\bm{\theta}^*$, which is at rate $\mathcal{O}(1/t)$. 

\textbf{Notations and Paper Organization. } 
We define some notations that we will use throughout this paper. 
For matrices $X \in \mathbb{R}^{n\times o}, Y \in \mathbb{R} ^{p\times q}$: 
let $ X\otimes Y$ denote the standard Kronecker product; 
let $\|X\|_2$ be the operator norm of $X$ and $\|X\|_F$ be the Frobenius norm of $X$; 
denote the vectorization of $X$ as $\VEC(X)=(X_{11},...,X_{n1},X_{12},...,X_{n2},...,X_{1o},...,X_{no})^{T}$. 
For two nonnegative real values $a(t)$ and $b(t)$, denote $a(t) \lesssim b(t)$ if $\exists c_1 > 0$ such that $a(t) \leq c_1 b(t)$. 

The remainder of this paper is organized as follows. 
In Section \ref{sec:02}, we propose our plug-in estimator for the asymptotic covariance and show the convergence of the plug-in estimator to the true covariance. 
In Section \ref{sec:03}, we propose our Hessian-free estimator for the asymptotic covariance and prove its asymptotic consistency. 
In Section \ref{sec:04}, some numerical experiments are reported. 
In Section \ref{sec:05}, we discuss the findings of this paper and propose some future work.

\section{A Plug-in Estimator and Its Convergence}\label{sec:02}
In this section, we describe our plug-in covariance estimator in the ROOT-SGD algorithm, and we prove the convergence rate of the aforementioned plug-in covariance estimator. 

\subsection{A Plug-in Estimator with Thresholding}
Our goal is to provide an online estimator for the asymptotic covariance of ROOT-SGD. 
Recall that the asymptotic covariance is $\Sigma = A^{-1} (S + \mathbbm{E}[\Xi_{\bm{x}}(\bm{\theta}^*) \Lambda \Xi_{\bm{x}}(\bm{\theta}^*)]) A^{-1}$. 
For the plug-in estimator, we assume that we have access to the stochastic second order term $\nabla^2 f(\widehat{\bm{\theta}}_{t-1};\bm{x}_t)$. 

There are two steps in the plug-in estimator: plug-in and thresholding. 
The plug-in step approximates each component in $\Sigma$ by its empirical version, that is, estimate $A, S$ and $\mathbbm{E}[\Xi_{\bm{x}}(\bm{\theta}^*) \Lambda \Xi_{\bm{x}}(\bm{\theta}^*)]$ using the parameter estimates and data samples; 
the thresholding step performs spectral thresholding on the empirical estimators. 
We will explain spectral thresholding in detail later in our thresholding step. 

\paragraph{Plug-in Step.} 
For $A, S, \mathbbm{E}[\Xi_{\bm{x}}(\bm{\theta}^*) \Lambda \Xi_{\bm{x}}(\bm{\theta}^*)]$ in the asymptotic covariance, we approximate them by their empirical counterparts as follows. 
For 
\[
    A  = \nabla^2 F(\bm{\theta}^*)  ,\]
we approximate it by 
\begin{equation}\label{eq:a_approx}
    \widehat{A}_t = \frac{1}{t - B}\sum_{i=B+1}^{t}\nabla^2 f(\widehat{\bm{\theta}}_{i-1};\bm{x}_i).
\end{equation}
For 
    \[S = \mathbbm{E}_{x\sim P}[\nabla f (\bm{\theta}^*;\bm{x}) \nabla f (\bm{\theta}^*;\bm{x})^T ],\]
we approximate it by 
\begin{equation}\label{eq:s_approx}
    \widehat{S}_t = \frac{1}{t - B}\sum_{i=B+1}^{t}\nabla f(\widehat{\bm{\theta}}_{i-1};\bm{x}_i) \nabla f(\widehat{\bm{\theta}}_{i-1};\bm{x}_i)^T.
\end{equation}
For 
\[\mathbbm{E}[\Xi_{\bm{x}}(\bm{\theta}^*)\Lambda \Xi_{\bm{x}}(\bm{\theta}^*)] =  \mathbbm{E}[\nabla^2 f(\bm{\theta}^*;\bm{x})\Lambda \nabla^2 f(\bm{\theta}^*;\bm{x})] - A\Lambda A\]
we approximate it by
\begin{equation}\label{eq:exi_approx}
\frac{1}{t - B}\sum_{i=B+1}^{t}\nabla^2 f(\widehat{\bm{\theta}}_{i-1};\bm{x}_i)\Lambda \nabla^2 f(\widehat{\bm{\theta}}_{i-1};\bm{x}_i) - A\Lambda A.    
\end{equation} 
 In the estimator \eqref{eq:exi_approx}, we need to further estimate $\Lambda$. 
Since $\Lambda$ is the solution of \eqref{eq:lambda}, we can replace each unknown term in \eqref{eq:lambda} by their empirical estimators, and solve the ``perturbed version" of \eqref{eq:lambda} in $\Lambda$ as:
    \begin{align}\label{eq:plug_lambda}
        \Lambda \widehat{A}_t + \widehat{A}_t\Lambda - \eta \frac{1}{t - B}\sum_{i=B+1}^{t}\nabla^2 f(\widehat{\bm{\theta}}_{i-1};\bm{x}_i)\Lambda \nabla^2 f(\widehat{\bm{\theta}}_{i-1};\bm{x}_i) = \eta \widehat{S}_t,
    \end{align}
and we denote the solution as $\widehat{\Lambda}_t$. 
Vectorize \eqref{eq:plug_lambda} we have
 \begin{align*}
\eta \VEC(\widehat{S}_t) =& 
(\widehat{A}_t\otimes I )\VEC(\widehat{\Lambda}_t) + (I\otimes \widehat{A}_t)\VEC(\widehat{\Lambda}_t)\\
& - \eta \underbrace{\frac{1}{t - B}\sum_{i=B+1}^{t}\nabla^2 f(\widehat{\bm{\theta}}_{i-1};\bm{x}_i)\otimes \nabla^2 f(\widehat{\bm{\theta}}_{i-1};\bm{x}_i)}_{:=\widehat{P}_t}\VEC(\widehat{\Lambda}_t).
 \end{align*}
That is,
\begin{equation}\label{eq:lambda_approx}
    \VEC(\widehat{\Lambda}_t) = \eta \{\widehat{A}_t\otimes I + I\otimes \widehat{A}_t - \eta \widehat{P}_t\}^{-1} \VEC(\widehat{S}_t).
\end{equation}
Using the estimators $\widehat{A}_t$,$\widehat{S}_t$ and $\widehat{\Lambda}_t$ in equations \eqref{eq:a_approx}, \eqref{eq:s_approx}, \eqref{eq:lambda_approx}, the plug-in estimator of the asymptotic covariance $\Sigma = A^{-1} (S + \mathbbm{E}[\Xi_{\bm{x}}(\bm{\theta}^*) \Lambda \Xi_{\bm{x}}(\bm{\theta}^*)]) A^{-1} = A^{-1} (\frac{1}{\eta}\Lambda A + \frac{1}{\eta}A\Lambda - A\Lambda A) A^{-1}$ is 
\begin{equation}
    \widehat{\Sigma}_t = \widehat{A}_t^{-1} \left(\frac{1}{\eta}\widehat{\Lambda}_t \widehat{A}_t + \frac{1}{\eta}\widehat{A}_t\widehat{\Lambda}_t -  \widehat{A}_t \widehat{\Lambda}_t\widehat{A}_t\right) \widehat{A}_t^{-1}.\label{eq:esti_sigma}
\end{equation}

We add notes on the computation of the plug-in estimator as follows. 
First, the plug-in estimator is an online estimator. 
Since each of $\widehat{A}_t, \widehat{P}_t, \widehat{S}_t$ can be computed in an online fashion, the estimator $\widehat{\Lambda}_t$ in \eqref{eq:lambda_approx} is also an online estimator. 
The plug-in estimator in \eqref{eq:esti_sigma} then can be computed with the online estimators $\widehat{A}_t$ and $\widehat{\Lambda}_t$. 
The plug-in estimator is thus online. 
Second, among all components in the plug-in estimator, the term $\widehat{\Lambda}_t$ in \eqref{eq:lambda_approx} is the hardest to compute. 
Luckily, there are some cases in which we can avoid computing $\widehat{\Lambda}_t$. 
For example, when $\Xi_{\bm{x}}(\bm{\theta}^*) = 0$ (which holds for the exponential family model), one has $\Sigma = A^{-1} S A^{-1}$.
In this case, the plug-in estimator is simply $\widehat{\Sigma}_t = \widehat{A}_t^{-1} \widehat{S}_t \widehat{A}_t^{-1}$. 
Avoiding computing $\widehat{\Lambda}_t$ can reduce the computational burden in those cases. 

\paragraph{Thresholding Step.} 
Note that in equations \eqref{eq:lambda_approx} and \eqref{eq:esti_sigma} we need to invert $\widehat{A}_t\otimes I + I\otimes \widehat{A}_t - \eta \widehat{P}_t$ and $\widehat{A}_t$, which are random quantities that might be poorly conditioned or even non-invertible. 
To deal with this issue, we do spectral thresholding for $\widehat{A}_t$ and $\widehat{P}_t$. 

We briefly explain the spectral thresholding technique as follows. 
For a positive semi-definite matrix $M$, spectral thresholding involves two steps. 
First, one performs eigen-decomposition on $M$, denote $M = UDU^T$. 
Second, depending on whether one needs upper thresholding or lower thresholding on the matrix $M$, one performs upper or lower truncation on the diagonal matrix $D$. 
For example, when upper thresholding using the threhold value $C$, one computes a diagonal matrix $\widetilde{D}$ such that $(\widetilde{D})_{i,i} = \min(C,D_{i,i})$. 
On the contrary, when lower thresholding using the threhold value $c$, one computes a diagonal matrix $\widetilde{D}$ such that $(\widetilde{D})_{i,i} = \max(c,D_{i,i})$. 
The thresholded matrix is $M = U\widetilde{D}U^T$. 

In our case, we do lower thresholding for $\widehat{A}_t$ using threshold value $\delta$ and denote the thresholded matrix as $\widetilde{A}_t$. 
We do upper thresholding for $\widehat{P}_t$ using threshold value $\delta'$ and denote the thresholded matrix as $\widetilde{P}_t$. 
The choice of thresholding parameters $\delta, \delta'$ will be discussed in Theorem \ref{thm:consis}. 

With the thresholded estimator for $A$ and $P$, we now give the thresholded plug-in covariance estimator. 
Similar to equation \eqref{eq:lambda_approx}, we first develop the thresholded estimator for $\Lambda$:
\[
\VEC(\widetilde{\Lambda}_t) = \eta \{\widetilde{A}_t\otimes I + I\otimes \widetilde{A}_t- \eta \widetilde{P}_t\}^{-1} \VEC(\widehat{S}_t).
\]
Then the thresholded plug-in estimator for the asymptotic covariance is
\begin{equation}
    \widetilde{\Sigma}_t = \widetilde{A}_t^{-1} \left(\frac{1}{\eta}\widetilde{\Lambda}_t \widetilde{A}_t + \frac{1}{\eta}\widetilde{A}_t\widetilde{\Lambda}_t-  \widetilde{A}_t \widetilde{\Lambda}_t\widetilde{A}_t\right) \widetilde{A}_t^{-1}.\label{eq:esti_thres_sigma}
\end{equation}
The thresholding is usually not required in practice, as $\widehat{A}_t$ and $\widehat{P}_t$ are close to $A$ and $P$ with high probability. 
The thresholding provides a theoretical guarantee that the matrix inverse in our estimator is doable.

\subsection{Convergence of Our Plug-in Estimator}\label{sec:2.2}
In this section, we show that the thresholded plug-in estimator is consistent. 
In particular, we bound $\mathbbm{E}\|\Sigma - \widetilde{\Sigma}_t\|_F$ and show that it converges to $0$ as $t\to \infty$. 
We also specify the convergence rate as a function of $t$. 
For a specific problem, such as linear regression, one can further analyze the dependence of the convergence rate on the dimension of the problem. 

To establish the consistency of the plug-in estimators, we need some regularity conditions on the objective function. 
These conditions are mild, as we show that some common learning examples, such as linear regression, logistic regression, and exponential family model, all satisfy those conditions; see details in Appendix \ref{app:A}. 
We now define these conditions. 

\begin{assumption}[$\gamma$-strong convexity and $L_1$-smoothness]\label{ass:sc}
Assume that $F(\bm{\theta})$ is twice continuously differentiable, $\gamma$-strongly convex and $L_1$-smooth: 
$$F(\bm{\theta}_1)\geq F(\bm{\theta}_2) + \nabla F(\bm{\theta}_2) ^T (\bm{\theta}_1 - \bm{\theta}_2) + \frac{\gamma}{2}\|\bm{\theta}_1 - \bm{\theta}_2\|_2^2, \forall \bm{\theta}_1,\bm{\theta}_2;$$
$$F(\bm{\theta}_1)\leq F(\bm{\theta}_2) + \nabla F(\bm{\theta}_2) ^T (\bm{\theta}_1 - \bm{\theta}_2) + \frac{L_1}{2}\|\bm{\theta}_1 - \bm{\theta}_2\|_2^2, \forall \bm{\theta}_1,\bm{\theta}_2.$$
\end{assumption}

\begin{assumption}[Finite covariance at optimality]\label{ass:bv}
Assume $\nabla f(\bm{\theta};\bm{x})$ is unbiased and has finite covariance at $\bm{\theta}^*$ (w.r.t. randomness in $\bm{x}$) as follows:
\begin{align*}
    &\mathbbm{E}_{\bm{x}\sim P}[\nabla f(\bm{\theta};\bm{x})] = \nabla F(\bm{\theta}),\\
    &\mathbbm{E}_{\bm{x}\sim P}[\|\nabla f(\bm{\theta}^*;\bm{x}) - \nabla F(\bm{\theta}^*)\|_2^2]= \sigma_1^2 < \infty.
\end{align*}
\end{assumption}

\begin{assumption}[Lipschitz stochastic noise]\label{ass:4}
Denote $\delta(\bm{\theta};\bm{x}) := \nabla f(\bm{\theta};\bm{x}) - \nabla F(\bm{\theta})$. 
Assume there exists a constant $L_2$ that
\begin{equation}
    \mathbbm{E} \|\delta(\bm{\theta}_1;\bm{x}) - \delta(\bm{\theta}_2;\bm{x})\|_2^2 \leq L_2^2 \|\bm{\theta}_1 - \bm{\theta}_2\|_2^2, \forall \bm{\theta}_1,\bm{\theta}_2.
\end{equation}
\end{assumption}

\begin{assumption}\label{ass:5}
Assume the stochastic gradient is mean-smooth around $\bm{\theta}^*$. 
In particular, there exists a constant $L_3$ such that: 
\begin{align*}
    \mathbbm{E}\|\nabla^2 f(\bm{\theta};\bm{x}) - \nabla^2 f(\bm{\theta}^*;\bm{x})\|_2^2 \leq L_3^2\|\bm{\theta} - \bm{\theta}^*\|_2^2, \forall \bm{\theta} \in \mathcal{R}^p.
\end{align*}
Assume the fourth moments of the stochastic gradient and the stochastic Hessian are bounded. 
In particular, there exist constants $l_4$ and $L_4$ such that: 
\begin{align*}
    &\mathbbm{E}\|\nabla f(\bm{\theta}^*;\bm{x})\|_2^4 \leq l_4^4, \\
    &\mathbbm{E}\|\nabla^2 f(\bm{\theta}^*;\bm{x})\|_2^4 \leq L_4^4.
\end{align*}
\end{assumption}

\begin{assumption}[$L_5-$smoothness]\label{ass:smooth}
Assume that $\nabla f(\bm{\theta};\bm{x})$ is mean-squared Lipchitz continuous. 
In particular, there exists a constant $L_5$ that: 
$$\mathbbm{E}\|\nabla f( \bm{\theta}_1;\bm{x}) - \nabla f( \bm{\theta}_2;\bm{x})\|_2^2 \leq L_5^2\| \bm{\theta}_1 -  \bm{\theta}_2\|_2^2, \forall  \bm{\theta}_1, \bm{\theta}_2.$$
\end{assumption}
\begin{remark}
Assumption \ref{ass:smooth} is implied by Assumptions \ref{ass:sc} and \ref{ass:4}. 
When Assumptions \ref{ass:sc} and \ref{ass:4} hold, we must have Assumption \ref{ass:smooth} hold for $L_5^2 = L_1^2 + L_2^2$. 
See the proof in Appendix \ref{app:connect_assumption}. 
\end{remark}

\begin{assumption}\label{ass:6}
Assume the Kronecker product of the Hessian is smooth around $\bm{\theta}^*$ and has bounded covariance at $\bm{\theta}^*$. 
In particular, there exist constants $L_6$, $L_6^\prime$ and $L_7$ such that:
\[ \mathbbm{E}\|\nabla^2 f(\bm{\theta};\bm{x})\otimes \nabla^2 f(\bm{\theta};\bm{x}) - \nabla^2 f(\bm{\theta}^*;\bm{x})\otimes \nabla^2 f(\bm{\theta}^*;\bm{x})\|_2\leq L_6 \|\bm{\theta} - \bm{\theta}^*\|_2 + L_6^\prime \|\bm{\theta} - \bm{\theta}^*\|_2^2,
\]
\[
\mathbbm{E} \|\nabla^2 f(\bm{\theta}^*;\bm{x})\otimes \nabla^2 f(\bm{\theta}^*;\bm{x}) - \mathbbm{E}[\nabla^2 f(\bm{\theta}^*;\bm{x})\otimes \nabla^2 f(\bm{\theta}^*;\bm{x})]\|_2^2\leq L_7.
\]
\end{assumption}
\begin{remark}\label{remark:02}
Assumption \ref{ass:6} is implied by Assumption \ref{ass:5}. 
When Assumption \ref{ass:5} holds, Assumption \ref{ass:6} also holds for $L_6 = 2L_3L_4^2$, $L_6' = L_3^2$ and $L_7 = L_4^4$. See proof in Appendix \ref{app:connect_assumption}. 
\end{remark}

We show that the thresholded plug-in estimator is consistent in the following theorem:
\begin{theorem}[Asymptotically consistent estimator for $\Sigma$]\label{thm:consis}
Under Assumption \ref{ass:sc} to \ref{ass:6}, there exists constants $c_1$, $c_2$ such that when we run ROOT-SGD for $\eta < \min\Biggl(c_1\biggl(\frac{\gamma}{L_2^2} \wedge \frac{1}{L_1} \wedge \frac{\gamma^{1/3}}{L_4^{4/3}} \biggl), 2\delta /L_4^2\Biggr)$, burn-in period $B = \left\lceil{\frac{c_2}{\gamma\eta}}\right\rceil$ and take the thresholding parameter in $\widetilde{\Sigma}_t$ be such that $\delta < \lambda_{\min}(A)$ and $\delta' = L_4^2$, we have
\[\mathbbm{E}\|\Sigma - \widetilde{\Sigma}_t\|_F \lesssim \left[\|S\|_F  C_p^2 + \sqrt{p}C_p^4\right] /\sqrt{t},\]
where $C_p \lesssim \max\{\sigma_1/\gamma,L_3, \sqrt{L_4},l_4,\sqrt{L_5},L_6\}$. 
Thus, for a fixed $C_p$, when sample size $t \to \infty$, we have $\mathbbm{E}\|\Sigma - \widetilde{\Sigma}_t\|_F \to 0$.
\end{theorem}
\begin{remark}
In Theorem \ref{thm:consis}, we can further analyze how the upper bound depends on the parameter dimension for a specific learning problem. 

For example, consider the linear regression model. 
We have data $\bm{x}_i^T = (\bm{a}_i^T,b_i)$, where $\bm{a}_i$'s are the vectors of explanatory variables, and $b_i$'s are the responses. 
Assume $\bm{a}_i \stackrel{i.i.d.}{\sim} N(0,I_p)$, and there is a true parameter $\bm{\theta}^*$ such that 
$b_i = \bm{a}_i^T \bm{\theta}^* + \epsilon_i$ for $\epsilon_i\stackrel{i.i.d.}{\sim} N(0,1)$. 
Here the additive noises $\epsilon_i$'s are independent of $\bm{a}_i$'s. 
Use the squared loss
\[f(\bm{\theta};\bm{x}) = \frac{1}{2}(\bm{a}^T \bm{\theta} - b )^2.\]

As shown in Appendix \ref{app:lr_model}, we have $C_p \lesssim \sqrt{p} $ and $\|S\|_F = \left\|\mathbbm{E}[\epsilon^2 \bm{a}\bm{a}^T]\right\|_F =\|I_p\|_F = \sqrt{p}$. 
Thus the bound in Theorem \ref{thm:consis} becomes 
\[\mathbbm{E}\|\Sigma - \widetilde{\Sigma}_t\|_F \lesssim p^{5/2} /\sqrt{t}.\]
\end{remark}

The proof of Theorem \ref{thm:consis} is in Appendix \ref{app:B}. 

By Theorem \ref{thm:consis}, our thresholded plug-in estimator $\widetilde{\Sigma}_t$ is asymptotically consistent. 
Combined with the asymptotic distribution of the parameter estimator $\widehat{\bm{\theta}}_t$ in \eqref{eq:asymptotic_root} that $\sqrt{t}(\widehat{\bm{\theta}}_t - \bm{\theta}^*) \stackrel{d}{\to} \mathcal{N}(0,\Sigma)$, we can build asymptotically exact confidence interval.  
We summarize this result in the following corollary: 
\begin{corollary}[Asymptotically exact confidence interval]
 Under assumptions of Theorem \ref{thm:consis}, when the parameter dimension $p$ is fixed and sample size $t\to \infty$, we have
 \[P\left((\widehat{\bm{\theta}}_t)_j - z_{q/2}*(\widetilde{\Sigma}_t)_{j,j}^{1/2}/\sqrt{t}\leq \bm{\theta}^*_j \leq (\widehat{\bm{\theta}}_t)_j + z_{q/2}*(\widetilde{\Sigma}_t)_{j,j}^{1/2}/\sqrt{t}\right)\to 1-q.\]
\end{corollary}

With our plug-in covariance estimator proposed and the asymptotic convergence proved in this section, we compare our plug-in covariance estimator with that for the ASGD algorithm. 
The comparison details are in Appendix \ref{app:plug_in_compare}. 
To summarize the comparison, our plug-in estimator has faster convergence than that from the ASGD algorithm. However, the computation of our plug-in estimator is heavier than that of the ASGD, except for the special case that $\nabla^2 f(\bm{\theta}^*;\bm{x}) \equiv \nabla^2 F(\bm{\theta}^*)$. 

\section{A Hessian-free Estimator and Its Asymptotic Consistency}\label{sec:03}
In this section, we propose a covariance estimator for ROOT-SGD that uses only the parameter estimates sequence, $\widehat{\bm{\theta}}_1, \cdots, \widehat{\bm{\theta}}_t$. 
The plug-in estimator in the previous section uses the stochastic Hessian, which is unknown or hard to compute in some practical problems. 
Our covariance estimator in this section does not have this issue: it uses only the parameter estimates and thus is Hessian-free. 

\subsection{A Hessian-free Estimator}
Our Hessian-free estimator uses the random-scaling technique. 
In this section, we review the random-scaling technique in the ASGD algorithm and discuss the covariance estimator derived by the random-scaling technique. 
We further propose a random-scaling estimator for the ROOT-SGD algorithm. 

The random-scaling technique, when applied to the ASGD algorithm, gives the following result. 
For the ASGD estimator $\bar{\bm{\theta}}_t$, \citep{lee2021fast} shows that
\begin{align}
    \frac{\sqrt{t}\bm{w}^T (\bar{\bm{\theta}}_t - \bm{\theta}^*)}{\sqrt{\bm{w}^T \bar{V}_t \bm{w}}} \stackrel{d}{\to} \frac{W_1}{\sqrt{\int_{0}^1 (W_r - rW_1)^2 dr}},\label{eq:robust_testing}
\end{align}
where $\bm{w}$ is any given direction satisfying $\bm{w}\neq \bm{0}$, $W_r$ is the standard Wiener process and $\bar{V}_t = \frac{1}{t^2}\sum_{i=1}^t i^2 (\bar{\bm{\theta}}_i - \bar{\bm{\theta}}_t)(\bar{\bm{\theta}}_i - \bar{\bm{\theta}}_t)^T$. 
The term $\bar{V}_t$ is similar to a (weighted) covariance estimator when we treat $\bar{\bm{\theta}}_i$'s as samples, so we refer to it as the random-scaling estimator. 
And we refer to the L.H.S. of \eqref{eq:robust_testing}, i.e. the $\frac{\sqrt{t}\bm{w}^T (\bar{\bm{\theta}}_t - \bm{\theta}^*)}{\sqrt{\bm{w}^T \bar{V}_t \bm{w}}}$, as the random-scaling quantity. 
Now the random-scaling quantity converges in distribution to $ \frac{W_1}{\sqrt{\int_{0}^1 (W_r - rW_1)^2 dr}}$ for any given direction $\bm{w}$. 
Then $\Var\left(\frac{W_1}{\sqrt{\int_{0}^1 (W_r - rW_1)^2 dr}}\right) * \bar{V}_t$ is an estimator for the asymptotic covariance of $\sqrt{t}(\bar{\bm{\theta}}_t - \bm{\theta}^*)$, and we refer to such an estimator as the random-scaling covariance estimator. 

The random-scaling technique can be applied to other online algorithms. 
For example, a similar result as \eqref{eq:robust_testing} has been proved for the zero-order Kiefer-Wolfowitz algorithm by Theorem 4.4 of \citep{chen2021online}. 
These results indicate the usefulness of the random-scaling technique in different algorithms. 
For the ROOT-SGD estimator, the random-scaling technique may also apply. 

We now follow the random-scaling quantity in \eqref{eq:robust_testing} to propose a random-scaling quantity for the ROOT-SGD algorithm. 
Observing \eqref{eq:robust_testing}, the random-scaling quantity essentially treats the averaging sequence $\Bar{\bm{\theta}}_i$'s as samples and calculate the sample covariance.
The asymptotic distribution in \eqref{eq:robust_testing} relies on a substantial extension of the asymptotic normality of the averaging sequence $\sqrt{t}(\bar{\bm{\theta}}_t - \bm{\theta}^*)$ in \citep{polyak1992acceleration} to the random function of $\sqrt{t}(\bar{\bm{\theta}}_{[rt]} - \bm{\theta}^*)$ for $r\in [0,1]$.
Back to the ROOT-SGD algorithm, \citep{li2022root} shows that the asymptotic normality holds for the estimators $\widehat{\bm{\theta}}_t$, instead of the averaging estimator. 
Then ROOT-SGD counterpart of \eqref{eq:robust_testing} should treat $\widehat{\bm{\theta}}_i$'s as samples instead of the average $\Bar{\bm{\theta}}_i$'s. 
Thus, the random-scaling quantity for the ROOT-SGD is given by: 
\begin{equation}\label{eq:rs_quantity}
\frac{\sqrt{t}\bm{w}^T (\widehat{\bm{\theta}}_t - \bm{\theta}^*)}{\sqrt{\bm{w}^T V_t \bm{w}}} ,   
\end{equation}
where ${V}_t = \frac{1}{t^2}\sum_{i=1}^t i^2 (\widehat{\bm{\theta}}_i - \widehat{\bm{\theta}}_t)(\widehat{\bm{\theta}}_i - \widehat{\bm{\theta}}_t)^T$. 
We refer to $V_t$ as the random-scaling estimator. 
In the next section, we present the asymptotic distribution of \eqref{eq:rs_quantity} and explain how to use it to develop an asymptotically consistent covariance estimator for $\bm{\theta}^*$.

\subsection{Consistency of Our Hessian-free Estimator}
In this section, we present the asymptotic distribution of the random-scaling quantity in \eqref{eq:rs_quantity}, develop an asymptotically consistent covariance estimator based on the random-scaling quantity, and explain how to use the random-scaling estimator to conduct statistical inference for the ROOT-SGD algorithm. 

The asymptotic distribution of \eqref{eq:rs_quantity} is provided in the following theorem: 
\begin{theorem}\label{thm:3}
Suppose Assumption \ref{ass:sc} to \ref{ass:5} hold. 
Then there exists constants $c_1$, $c_2$, when we take the step size $\eta\in \left(0,c_1(\frac{\gamma}{L_2^2} \wedge \frac{1}{L_1} \wedge \frac{\gamma^{1/3}}{L_4^{4/3}} )\right)$ and burn-in period $B = \left\lceil{\frac{c_2}{\gamma\eta}}\right\rceil$ in the ROOT-SGD algorithm, for any given direction $\bm{w} \neq \bm{0}$, we have the following
\begin{align}
    \frac{\sqrt{t}\bm{w}^T (\widehat{\bm{\theta}}_t - \bm{\theta}^*)}{\sqrt{\bm{w}^T V_t \bm{w}}} \stackrel{d}{\to} \frac{W_1}{\sqrt{\int_{0}^1 (W_r - rW_1)^2 dr}},\label{eq:robust_testing_rootsgd}
\end{align}
where ${V}_t = \frac{1}{t^2}\sum_{i=1}^t i^2 (\widehat{\bm{\theta}}_i - \widehat{\bm{\theta}}_t)(\widehat{\bm{\theta}}_i - \widehat{\bm{\theta}}_t)^T$, and $W_r$ is the standard Wiener process (i.e., $W_r\sim N(0,r)$).
\end{theorem}
The proof of Theorem \ref{thm:3} is in Appendix \ref{app:C}. 
\begin{remark}
We add a note on the computation of the random-scaling quantity in \eqref{eq:robust_testing_rootsgd}. 
The random-scaling estimator $V_t= \frac{1}{t^2}\sum_{i=1}^t i^2 (\widehat{\bm{\theta}}_i - \widehat{\bm{\theta}}_t)(\widehat{\bm{\theta}}_i - \widehat{\bm{\theta}}_t)^T$ can be efficiently computed online. To see this, consider the following decomposition.
\begin{align*}
    V_t = \frac{1}{t^2}\sum_{i=1}^t i^2 \widehat{\bm{\theta}}_i \widehat{\bm{\theta}}_i^T -  \frac{1}{t^2}(\sum_{i=1}^t i^2 \widehat{\bm{\theta}}_i) \widehat{\bm{\theta}}_t^T - \frac{1}{t^2} \widehat{\bm{\theta}}_t (\sum_{i=1}^t i^2 \widehat{\bm{\theta}}_i^T) + \frac{t (t+1) (2t + 1)}{6t^2} \widehat{\bm{\theta}}_t \widehat{\bm{\theta}}_t^T,
\end{align*}
in which $\frac{1}{t^2}\sum_{i=1}^t i^2 \widehat{\bm{\theta}}_i \widehat{\bm{\theta}}_i^T$ and $\frac{1}{t^2}\sum_{i=1}^t i^2 \widehat{\bm{\theta}}_i$ can be efficiently computed online without requiring storing all past estimators $\widehat{\bm{\theta}_i}$'s. 
In this way, the random-scaling quantity $\frac{\sqrt{t}\bm{w}^T (\widehat{\bm{\theta}}_t - \bm{\theta}^*)}{\sqrt{\bm{w}^T V_t \bm{w}}}$ can be computed online. 
\end{remark}

\begin{remark}
We explain how to derive a consistent covariance estimator from Theorem \ref{thm:3} and how to use it for statistical inference. 
Denote the asymptotic limit $\frac{W_1}{\sqrt{\int_{0}^1 (W_r - rW_1)^2 dr}}$ as $X_{rs}$. 
Then $X_{rs}$ is a universal random variable that does not depend on the optimization problem we study. 
We refer to $X_{rs}$ as the random-scaling variable.
Using this notation, $\Var\left( X_{rs}\right)*\bm{w}^T V_t \bm{w}$ is an estimator for the variance of $\sqrt{t}\bm{w}^T (\widehat{\bm{\theta}}_t - \bm{\theta}^*)$, and $\widehat{\Sigma}_{t, rs}:= (\Var\left( X_{rs}\right))* V_t$ is a random-scaling covariance estimator for the covariance of $\sqrt{t}(\widehat{\bm{\theta}}_t - \bm{\theta}^*)$. 
By Theorem \ref{thm:3}, $\widehat{\Sigma}_{t, rs}$ is asymptotically consistent. 
When calculating the random-scaling covariance estimator $\widehat{\Sigma}_{t, rs}$, we can use the estimation in \cite{abadir1997two} that $\Var(X_{rs}) \approx 11.177513184$.
Moreover, \citep{abadir1997two,abadir2002simple} compute the quantiles for the random variable $X_{rs} = \left(W_1/\sqrt{\int_{0}^1 (W_r - rW_1)^2 dr}\right)$ based on direct integration. %, which is pretty accurate. 
Using their computed quantiles, one can develop asymptotically consistent statistical inference for the parameter estimation, for example, do statistical testing for $\bm{\theta}^*$ or build a confidence interval. 
\end{remark}

We compare our random-scaling estimator with that from the ASGD algorithm. 
The comparison details are in Appendix \ref{app:hessian_free_compare}. 
The computation costs of the two estimators are the same. 
Both estimators are asymptotically consistent, but the convergence speed is not proven. 
Since the ROOT-SGD converges faster than ASGD, there might be an advantage to use the random-scaling estimator from the ROOT-SGD in practice. 

\section{Numeric Studies}\label{sec:04}
In this section, we perform numerical studies using both synthetic data and a real dataset. 
In particular, in Section \ref{sec:04.01}, we show the performance of the confidence intervals constructed based on our covariance estimators for linear regression and logistic regression; in Section \ref{sec:04.02}, we visualize the confidence intervals based on our covariance estimator using the hand-written digit image example. 

\subsection{Simulation}\label{sec:04.01}
In this section, we simulate some examples under linear regression and logistic regression and compare the confidence intervals in these examples based on the plug-in and Hessian-free covariance estimators. 
The goal is to check if the confidence intervals have the coverage probability converging to the nominal value as the sample size increases. 

\textbf{Models.} 
In linear regression model, we have data $\bm{x}^T = (\bm{a}^T,b)\in R^d \times R$, where $\bm{a}$ is the vector of explanatory variable and $b$ is the response variable. 
The data is generated by the true parameter $\bm{\theta}^*$ as $b = \bm{a}^T \bm{\theta}^* + \epsilon$, where $\epsilon$ is a zero-mean r.v. that is independent of $\bm{a}$. 
We use the squared loss for the linear regression task
\[f(\bm{\theta};\bm{x}) = \frac{1}{2}(\bm{a}^T \bm{\theta} - b )^2.\]

In logistic regression model, we have data $\bm{x}^T = (\bm{a}^T,b) \in  R^d\times\{-1,1\}$. 
The data is generated by the true parameter $\bm{\theta}^*$ as $P(b = 1) = \frac{1}{1 + \exp(-\langle\bm{a},\bm{\theta}^*\rangle)}$. 
Take the negative log-likelihood as the risk function
$$f(\bm{\theta};\bm{x}) = \log(1 + \exp(- b \langle\bm{a},\bm{\theta}\rangle)).$$

\textbf{Data Generation.}
The linear regression and logistic regression data are generated as follows. 
We first set a data dimension $d$, and let the explanatory variables $\bm{a}_i$'s to be i.i.d. $N(\bm{0}_d, I_d)$. 
For both models, we set the true parameter $\bm{\theta}^*$ to be $d$ equally spaced values in $[0,1]$.
We generate the response variable $b_i$'s based on $\bm{a}_i$'s for linear regression and logistic regression cases, respectively. 
In particular, in linear regression, we let the additive noises $\epsilon_i$'s be i.i.d. standard normal random variables.  

\textbf{Algorithm Implementation.} 
We compare the confidence intervals built based on our covariance estimators for ROOT-SGD with those based on the plug-in covariance estimator \citep{chen2020statistical}, the non-overlapping batch-mean covariance estimator \citep{zhu2021online}, and the random-scaling estimator \citep{lee2021fast} for ASGD. 
So we implement the ROOT-SGD and ASGD algorithms; the implementation details are as follows. 
Both algorithms are initialized as $\widehat{\bm{\theta}}_0 = \bm{0}$ in linear regression and logistic regression. 
The total number of samples (i.e., the algorithm updates) is set to $250,000$. 
For ROOT-SGD, we take $\eta = 10^{-3}$ for linear regression and $\eta = 5*10^{-3}$ for logistic regression. 
The burn-in period is set to $B = 1000$ for both cases.  
For ASGD, we take $\eta_0 = 0.5$ and $\alpha = 0.505$, which are the same as the choices in \citep{zhu2021online}. 

\textbf{Confidence Interval Computation.} 
We compare confidence intervals induced by different covariance estimators. 
In particular, we compute the $95\%$ confidence intervals for each dimension of the parameter estimation and check how they cover the true parameters.
For ROOT-SGD algorithm, denote the plug-in covariance estimator as $\widehat{\Sigma}_{t,root-pi}$ and rewrite the random-scaling estimator $V_t$ in \eqref{eq:robust_testing_rootsgd} as $V_{t,root}$ to distinguish it with the random-scaling estimator from ASGD. 
For ASGD algorithm, denote the plug-in covariance estimator from \citep{chen2020statistical} as $\widehat{\Sigma}_{t,asgd-pi}$, denote the non-overlapping batch-mean covariance estimator from \citep{zhu2021online} as $\widehat{\Sigma}_{t,asgd-bm}$, and denote the random-scaling estimator in \citep{lee2021fast} (i.e., $\bar{V}_t$ in \eqref{eq:robust_testing}) as $V_{t,asgd}$.
Let $\widehat{\bm{\theta}}_{t,root}$ and $\widehat{\bm{\theta}}_{t,asgd}$ be the point estimators from ROOT-SGD and ASGD after $t$ updates, respectively. 
The $95\%$ confidence intervals based on each estimator are as follows:

\begin{itemize}[leftmargin=*,noitemsep]
    \item ROOT-SGD algorithm, plug-in estimator: $(\widehat{\bm{\theta}}_{t,root})_{i} \pm z_{0.975}*(\widehat{\Sigma}_{t,root-pi})_{i,i}/\sqrt{t};$
    \item ROOT-SGD algorithm, random-scaling estimator:
    $(\widehat{\bm{\theta}}_{t,root})_{i} \pm q_{rs,0.975}*(V_{t,root})_{i,i}/\sqrt{t};$
    \item ASGD algorithm, plug-in estimator:
    $(\widehat{\bm{\theta}}_{t,asgd})_{i} \pm z_{0.975}*(\widehat{\Sigma}_{t,asgd-pi})_{i,i}/\sqrt{t};$
    \item ASGD algorithm, batch-mean estimator:
    $(\widehat{\bm{\theta}}_{t,asgd})_{i} \pm z_{0.975}*(\widehat{\Sigma}_{t,asgd-bm})_{i,i}/\sqrt{t};$
    \item ASGD algorithm, random-scaling estimator:
        $(\widehat{\bm{\theta}}_{t,asgd})_{i} \pm q_{rs,0.975}*(V_{t,asgd})_{i,i}/\sqrt{t};$
\end{itemize}
where $z_{0.975}$ is the $97.5\%$ percentile for the standard normal random variable, and $q_{rs,0.975}$ is the $97.5\%$ percentile for $X_{rs} = \left(W_1/\sqrt{\int_{0}^1 (W_r - rW_1)^2 dr}\right)$, i.e., the limiting distribution of the random-scaling quantity in Theorem \ref{thm:3}. 
By \citep{abadir1997two,abadir2002simple}, we have $q_{rs,0.975} = 6.747 $. 

\begin{figure}
%\vskip -0.2in
\begin{center}
     \begin{subfigure}[b]{0.49\textwidth}
         \centering
         \includegraphics[width=\textwidth]{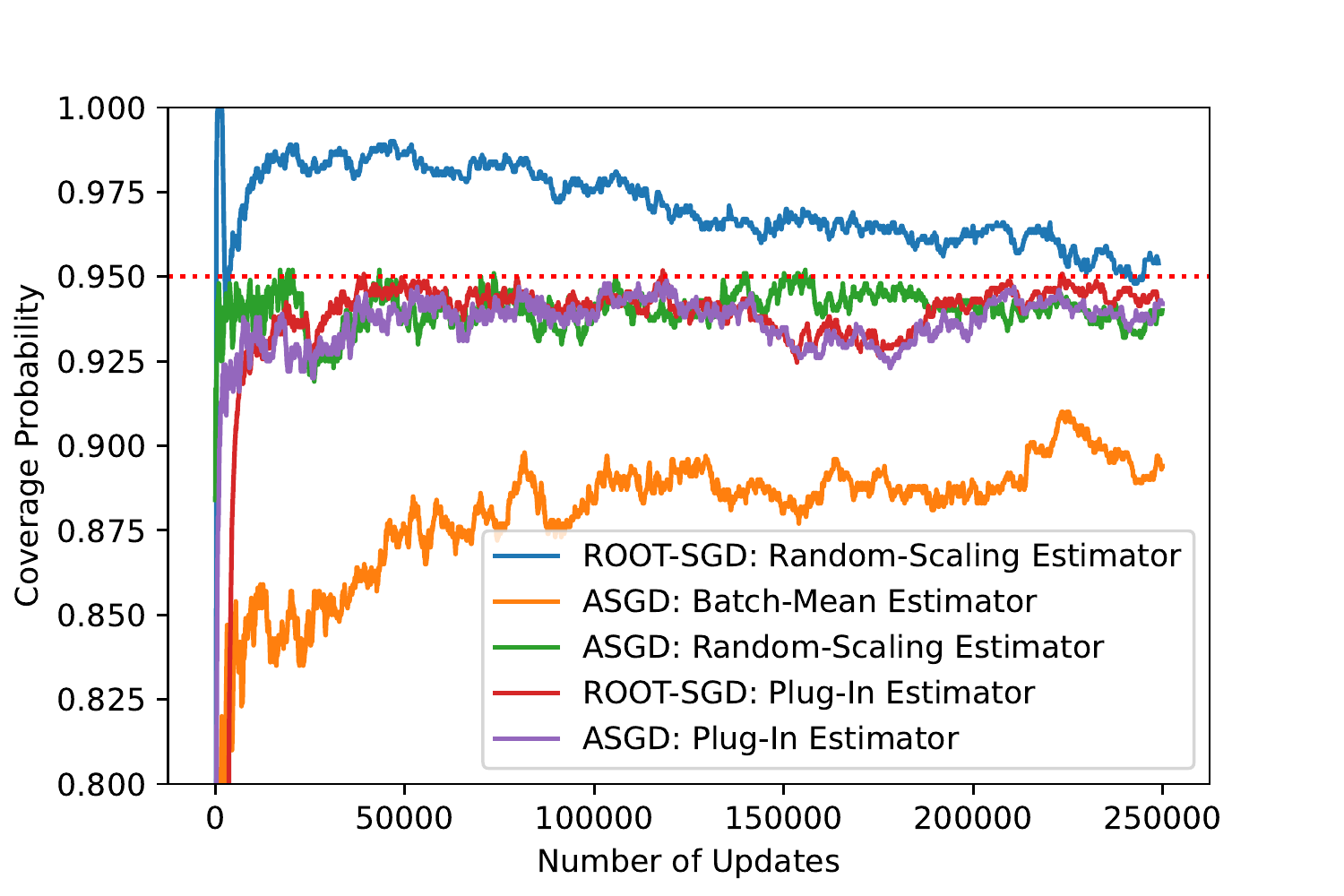}
     \end{subfigure}
     \hfill
    \begin{subfigure}[b]{0.49\textwidth}
         \centering
         \includegraphics[width=\textwidth]{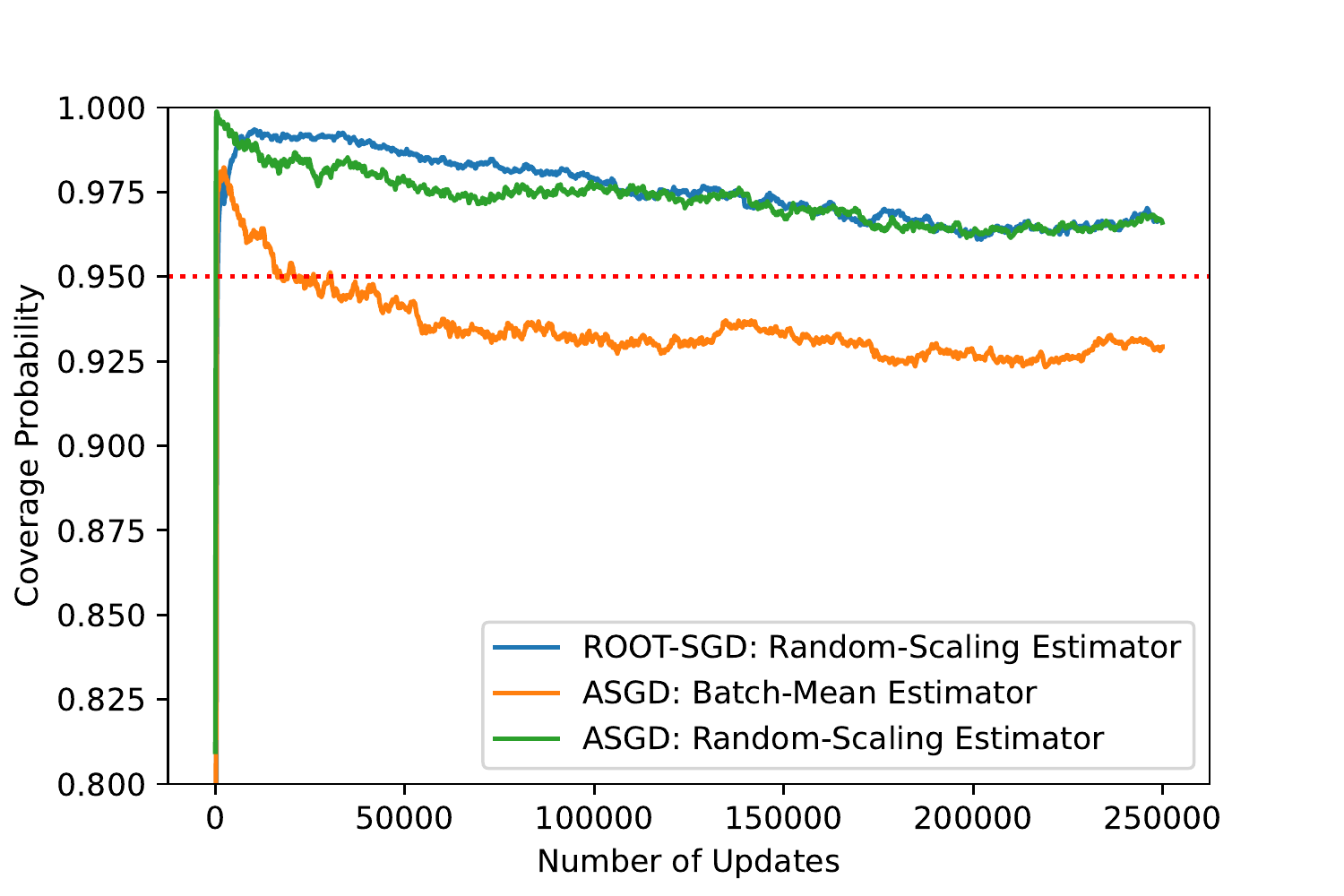}
     \end{subfigure}
          \hfill
    \begin{subfigure}[b]{0.49\textwidth}
         \centering
         \includegraphics[width=\textwidth]{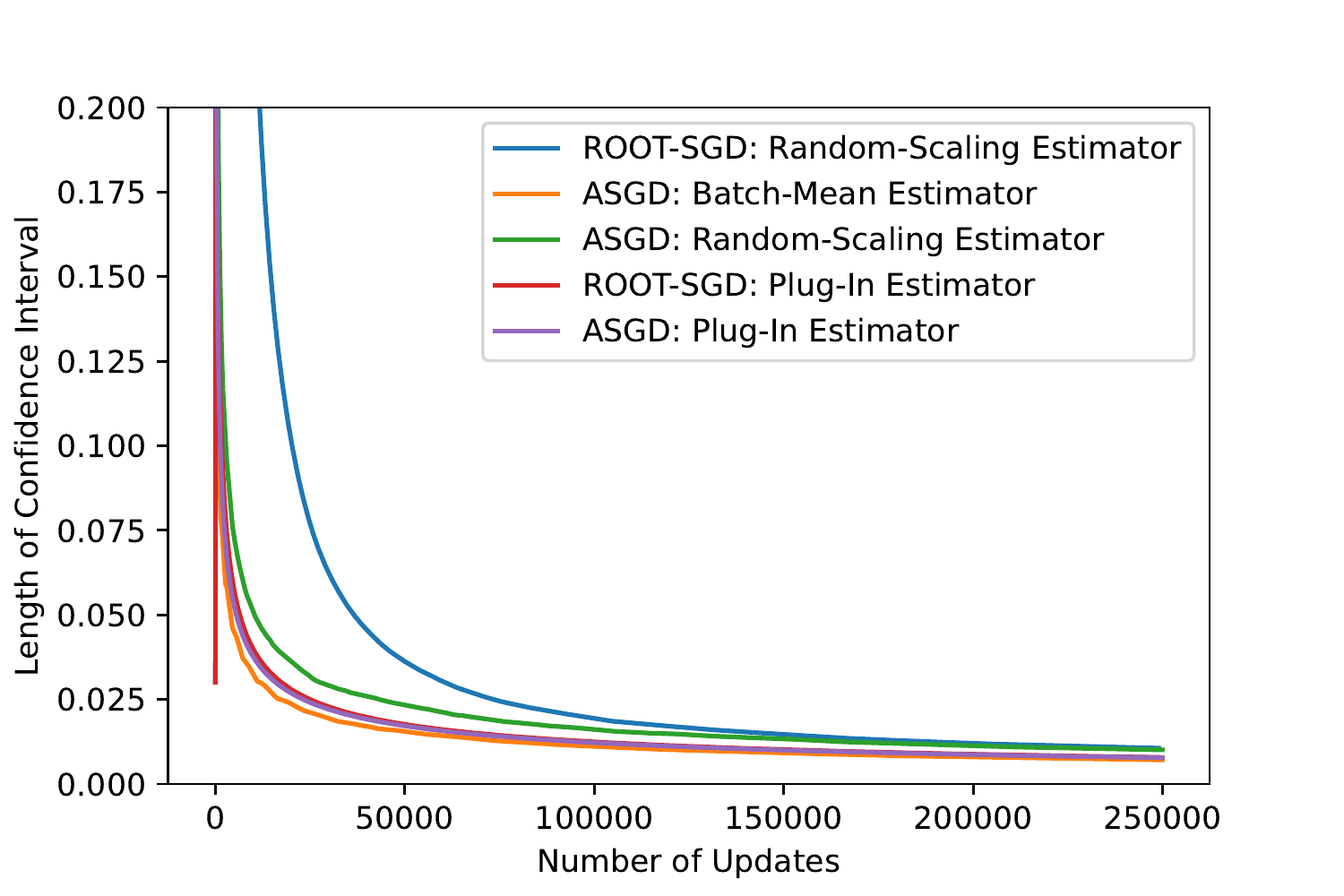}
     \end{subfigure}
          \hfill
    \begin{subfigure}[b]{0.49\textwidth}
         \centering
         \includegraphics[width=\textwidth]{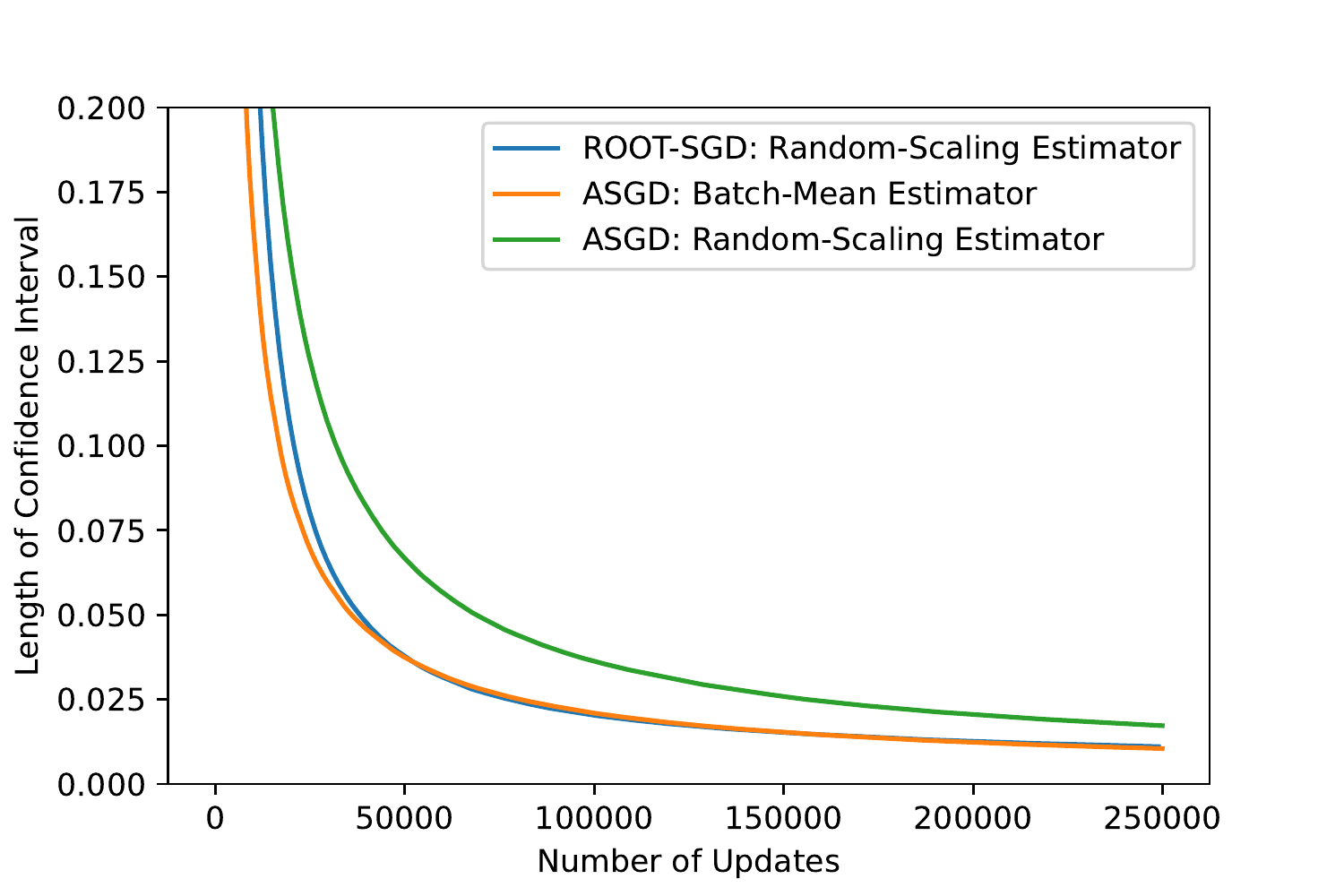}
     \end{subfigure}
               \hfill
    \begin{subfigure}[b]{0.49\textwidth}
         \centering
         \includegraphics[width=\textwidth]{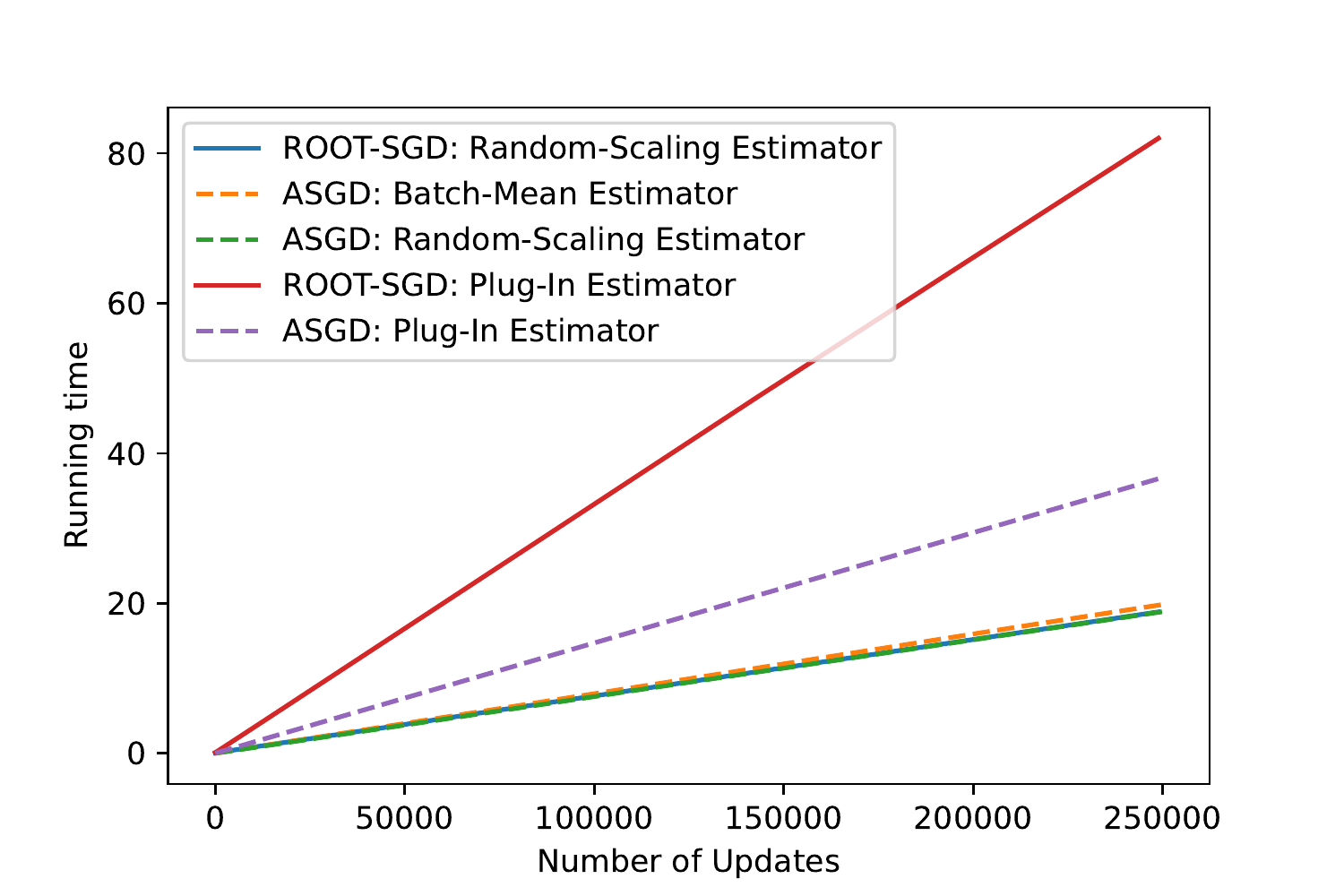}
         \caption{d = 5}
     \end{subfigure}
               \hfill
    \begin{subfigure}[b]{0.49\textwidth}
         \centering
         \includegraphics[width=\textwidth]{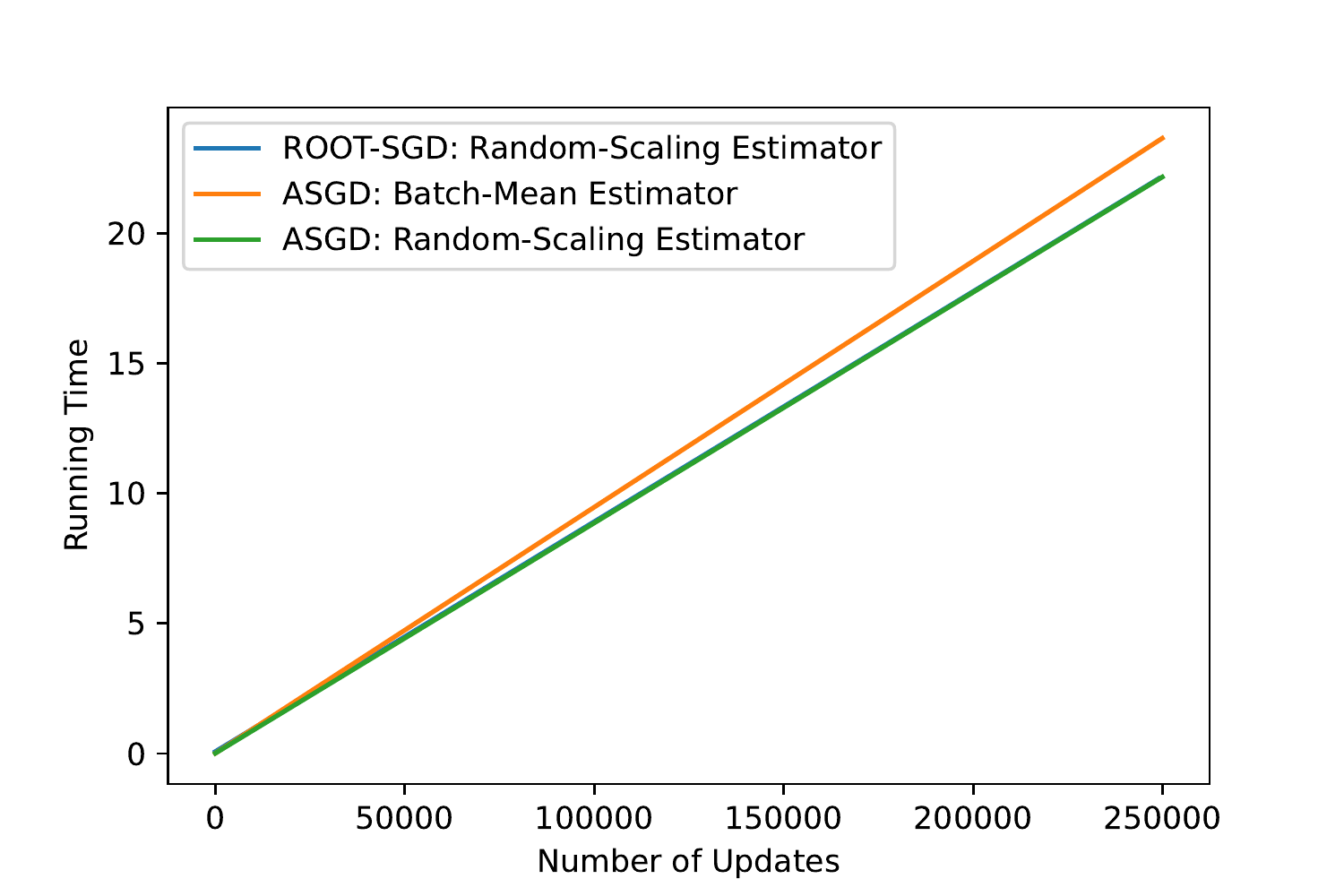}
         \caption{d = 20}
     \end{subfigure}
\caption{Comparison of the confidence intervals in linear regression. The first row is the coverage probability against the number of updates, with a red dashed line denoting the nominal coverage rate of $0.95$. The second row is the average length of the confidence interval. The last row is the total computation time to update the covariance estimator and the confidence interval.}
\label{fig:lr_ci}
\end{center}    
%\vskip -18pt    
\end{figure}

\textbf{Experiment Results.} 
We run $200$ repetitions of experiments for each estimator and compare the average coverage probability and length of the confidence interval from $200$ runs. 
The results of the linear regression experiment are shown in Figure \ref{fig:lr_ci}, where we have the two dimension settings: $d = 5$ and $d = 20$. 
Comparing the confidence intervals, we have the following.
\begin{itemize}
    \item For $d = 5$, all confidence intervals (CI) have coverage probability converging to the nominal value of $0.95$, except the one corresponding to the ASGD batch-mean estimator. 
    The coverage probability of the CI from the ROOT-SGD random-scaling estimator is higher than the counterpart from ASGD and is closer to the nominal coverage rate. 
    For the confidence intervals from plug-in estimators for ROOT-SGD and ASGD, we observe that they have the same confidence interval length. 
    However, the coverage probability of CI from ROOT-SGD is higher than that of ASGD and is closer to the nominal coverage rate. 
    Thus the CIs from ROOT-SGD are better than those from ASGD. 
    These observations show that the statistical inference of ROOT-SGD is more accurate than that of ASGD, possibly thanks to the fast convergence of ROOT-SGD. 
   
   \item For $d = 20$, due to the large computation of the plug-in estimators for both ROOT-SGD and ASGD, we do not compute them. 
   We compare the coverage of CIs from different Hessian-free estimators and conclude that the coverage probability of the CI from the ASGD batch mean estimator does not reach the nominal rate.
   The coverage probabilities of the CIs from both ROOT-SGD and ASGD random-scaling estimators are higher than the nominal rate. 
   However, the CI length of the ROOT-SGD random-scaling estimator is smaller than that of ASGD, which shows the advantage of the ROOT-SGD algorithm in building the confidence interval.
   
   \item In the running time comparison, we can see that the plug-in estimators take more time than the Hessian-free estimators (including the batch mean estimator and the random-scaling estimator), which is not surprising due to the second-order nature of the plug-in estimators. 
   The random-scaling estimators from ROOT-SGD and SGD take almost the same time, as they have lines on top of each other.
   The batch-mean estimator takes slightly more time than the random-scaling estimator.
   
   %\item Based on the above comparisons, the CI derived from the random-scaling covariance estimator using the ROOT-SGD algorithm is most preferred in terms of coverage probability, CI length, and computation time. 
\end{itemize}

For the linear regression example, we further check the empirical distribution of the random-scaling quantity of the ROOT-SGD algorithm (i.e., the l.h.s. of \eqref{eq:robust_testing_rootsgd}) to see if it matches the theoretic limiting distribution $X_{rs}$ (i.e., the r.h.s. of \eqref{eq:robust_testing_rootsgd}). 
In particular, we let $\bm{w}$ be the standard basis vectors in the random-scaling quantity. 
The result is shown in Figure \ref{fig:rs_dist}.
For both $d = 5$ and $d = 20$, we can see that the empirical distribution is close to the theoretical limiting distribution.  

\begin{figure}
%\vskip -0.2in
\begin{center}
     \begin{subfigure}[b]{0.49\textwidth}
         \centering
         \includegraphics[width=\textwidth]{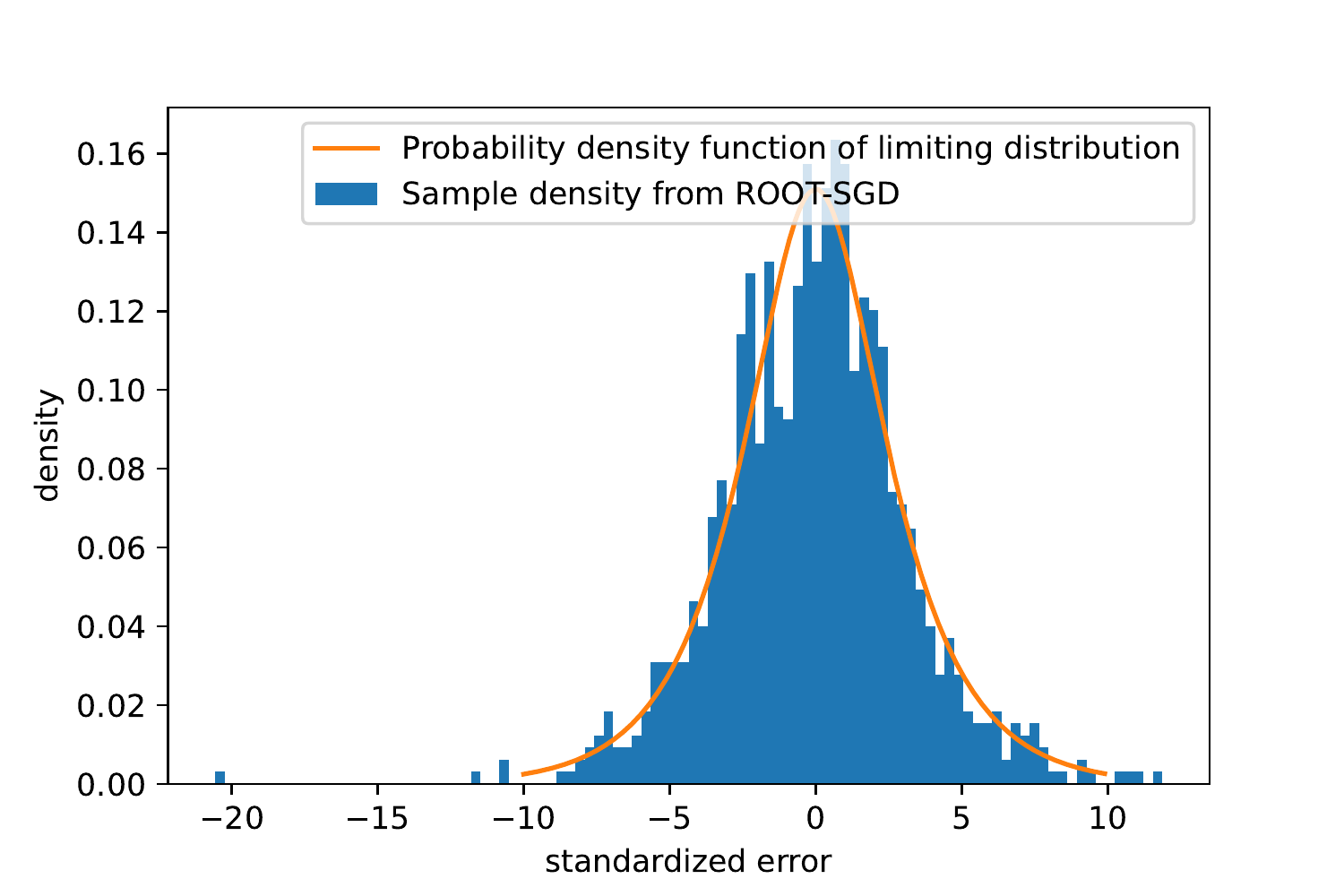}
         \caption{d = 5}
     \end{subfigure}
     \hfill
          \begin{subfigure}[b]{0.49\textwidth}
         \centering
         \includegraphics[width=\textwidth]{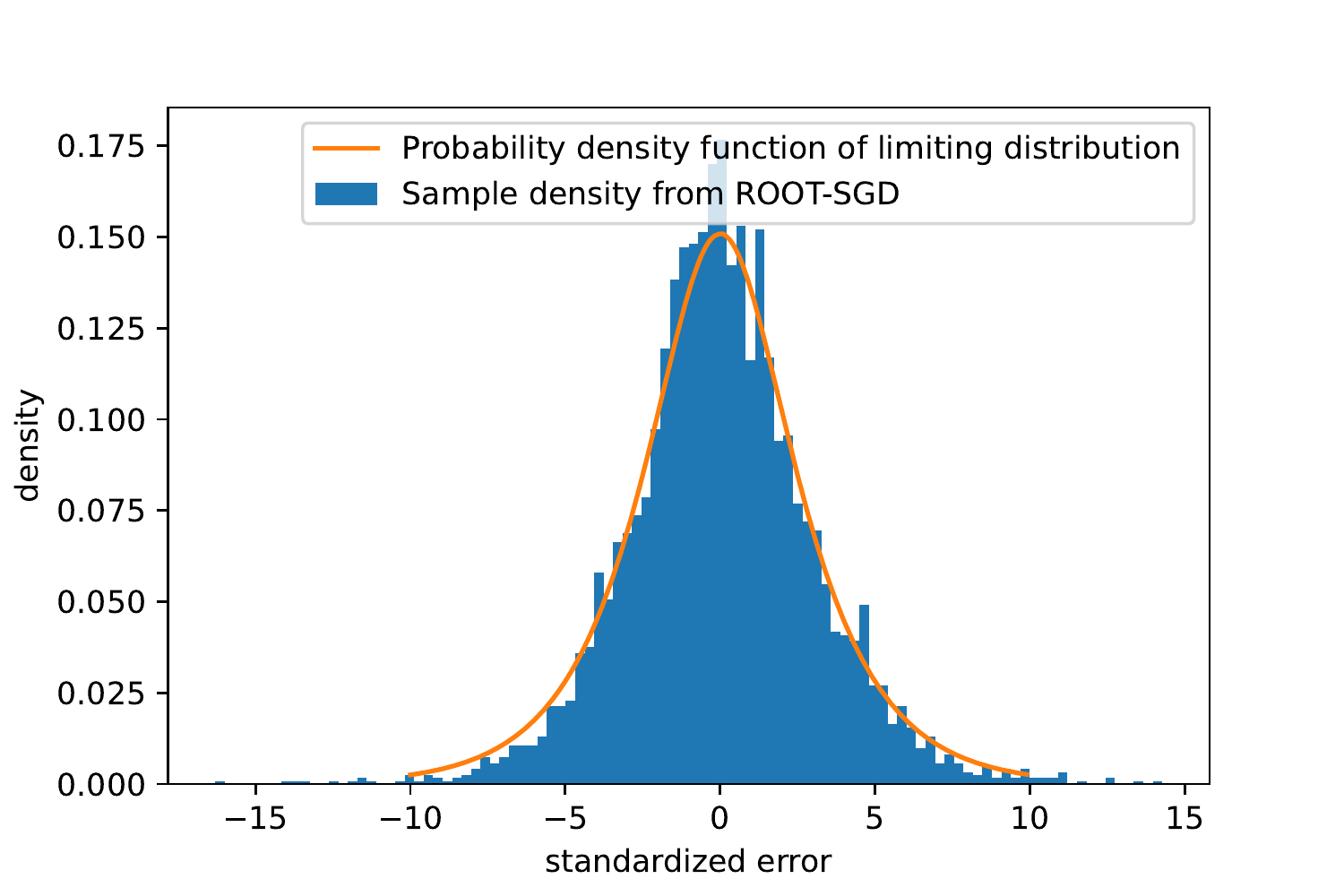}
         \caption{d = 20}
     \end{subfigure}
\caption{Density of the Random Scaling Quantity.}
\label{fig:rs_dist}
\end{center}
\end{figure}

The experiment result for logistic regression are in Figure \ref{fig:logistic_ci}. 
For $d = 5$, all confidence intervals (CI) have the coverage probability converge to the nominal value of $0.95$, except the SGD batch mean estimator -- it has a much lower coverage probability than the nominal probability. 
The plug-in estimator and the random-scaling estimator from ROOT-SGD have better coverage than the SGD counterparts while having the same length of the confidence interval. 
For $d = 20$, the ROOT-SGD random-scaling estimator has a coverage probability that converges faster to the nominal rate compared to the estimators from SGD. 
Furthermore, the ROOT-SGD random-scaling estimator has a higher coverage probability than the SGD random-scaling estimator, although their CI lengths are comparable after $200,000$ updates. 
The observations above all show the advantage of statistical inference by ROOT-SGD. %, thanks to its fast convergence. 

\begin{figure}[ht]
%\vskip -0.2in
\begin{center}
     \begin{subfigure}[b]{0.49\textwidth}
         \centering
         \includegraphics[width=\textwidth]{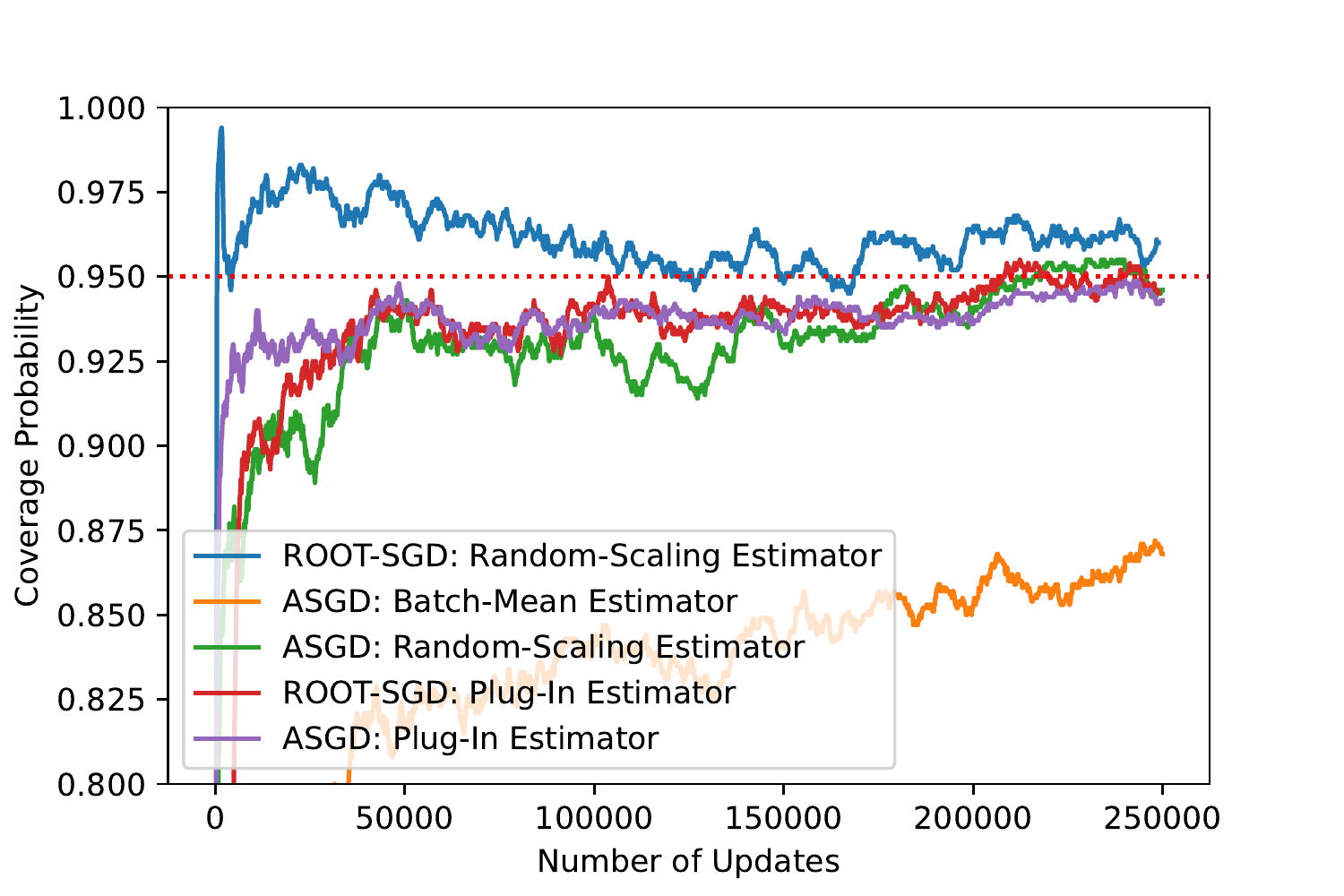}
     \end{subfigure}
     \hfill
    \begin{subfigure}[b]{0.49\textwidth}
         \centering
         \includegraphics[width=\textwidth]{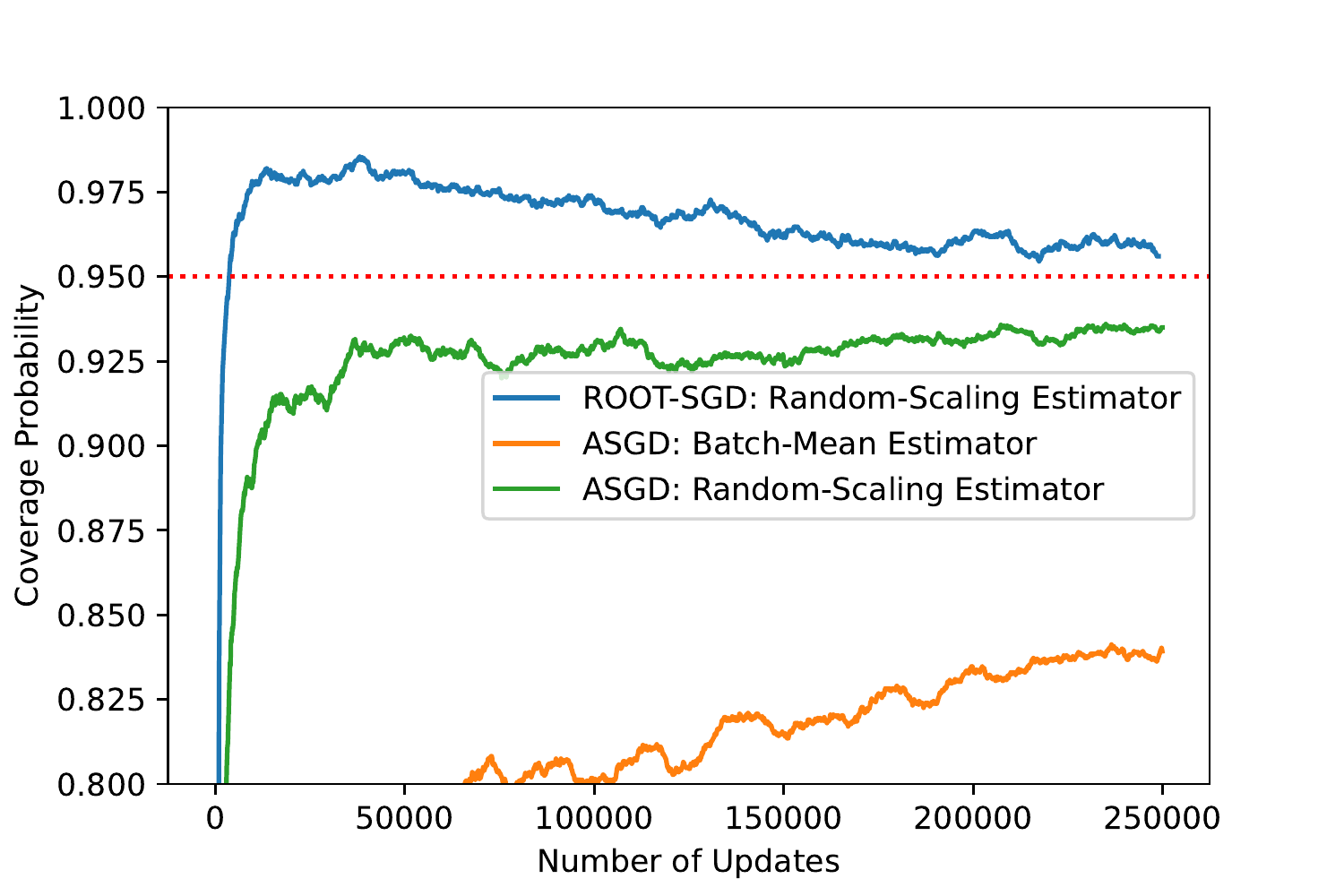}
     \end{subfigure}
          \hfill
    \begin{subfigure}[b]{0.49\textwidth}
         \centering
         \includegraphics[width=\textwidth]{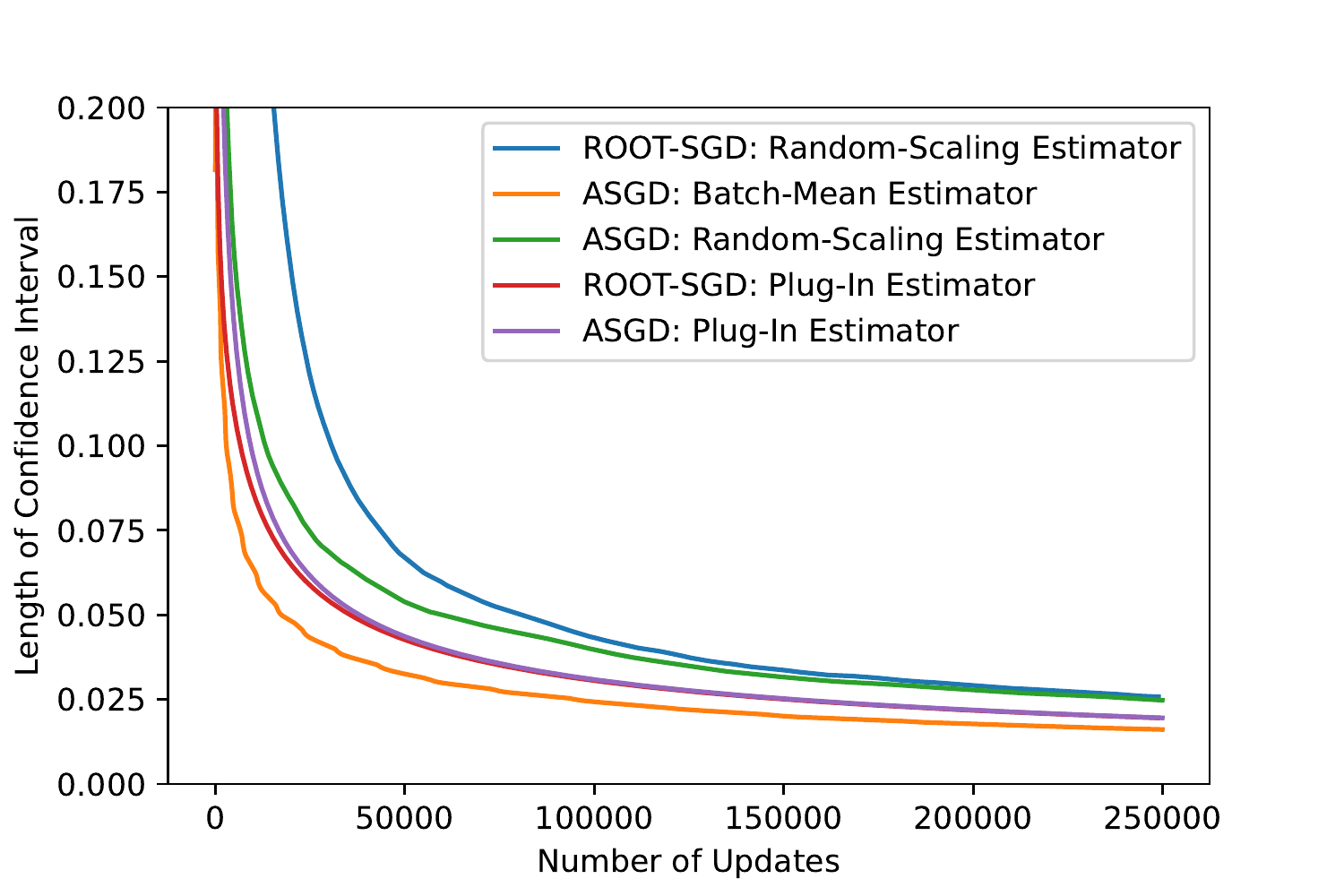}
         \caption{d = 5}
     \end{subfigure}
          \hfill
    \begin{subfigure}[b]{0.49\textwidth}
         \centering
         \includegraphics[width=\textwidth]{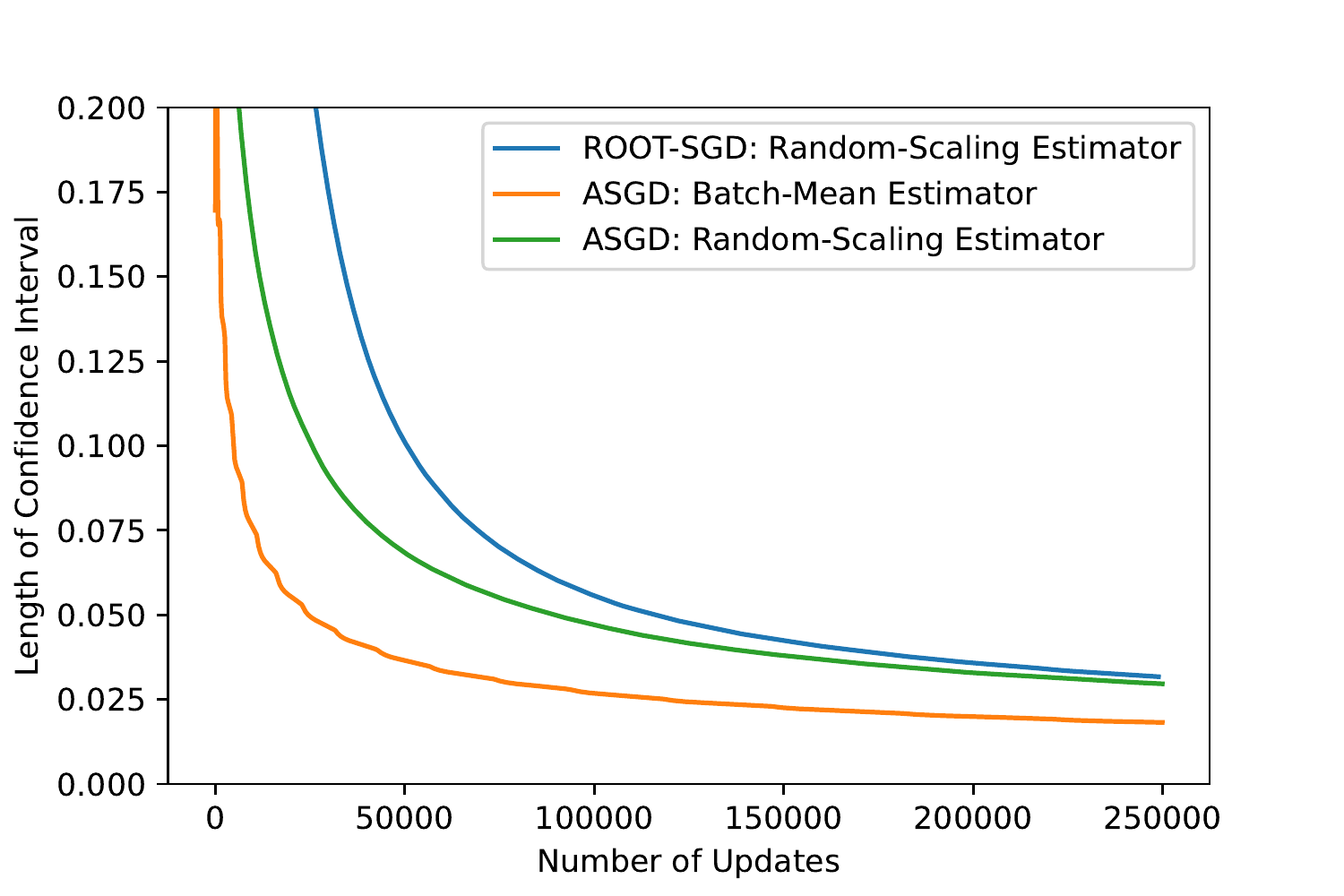}
         \caption{d = 20}
     \end{subfigure}
\caption{Comparison of the confidence intervals estimators in logistics regression. The first row is the coverage probability against the number of updates, and the red dashed line denotes the nominal coverage rate of $0.95$. The second row is the average length of the confidence interval.}
\label{fig:logistic_ci}
\end{center}    
%\vskip -18pt    
\end{figure}

%We further compare our result with those in \citep{chen2020statistical}. For $T = 10^5$ data samples, their estimated $95\%$ confidence intervals have the coverage probability and length as in Table \ref{tab:compare_sgd} 
%\begin{table}[H]
%    \centering
%    \begin{tabular}{|l|l|l|}
%    \hline
%   A&&  \\\hline
%    \multirow{2}{*}{Identity}& Cov Prob & 0.9499 (0.0094)\\
%     & Length & 0.0144 (0.0001)\\\hline
%    \multirow{2}{*}{Toeplitz}& Cov Prob & 0.9484 (0.0097)\\
%     & Length & 0.0181 (0.0005)\\\hline
%    \multirow{2}{*}{Equi Corr}& Cov Prob & 0.9510 (0.0099)\\
%     & Length & 0.0159 (0.0001)\\\hline
%    \end{tabular}
%    \caption{Performance of Statistical Inference Based on SGD}
%    \label{tab:compare_sgd}
%\end{table}

\subsection{Hand-Written Digit Analysis}\label{sec:04.02}
This experiment aims to visualize the confidence interval estimator that is induced by our plug-in covariance estimator. 
In particular, we consider the parameter estimation problem of mean estimation for the MNIST hand-written digit image set. 

The data-set description is as follows. 
There are $60,000$ training images of dimension $28*28$, each labeled as a digit between $0$ and $9$. 
For each label, we assume that it has a mean image and that the image instances are samples from a normal mean model. 

We apply the ROOT-SGD algorithm to estimate the mean image and use our plug-in covariance estimator to estimate the covariance for each digit. 
The implementation details are as follows. 
Recall that the risk function for normal mean estimation is $f(\bm{\theta};\bm{x}) = \frac{1}{2}\| \bm{\theta} - \bm{x}\|_2^2$. 
Under this risk function, a stochastic gradient in ROOT-SGD is calculated by randomly sampling data from the training set of that digit. 
When implementing the ROOT-SGD algorithm, we set the step size $\eta = 0.05$, the burn-in period $B = 10,000$, and the number of samples $t = 100,000$. 
We calculate the plug-in estimator for the covariance of the mean. 
Note that for this risk function, we have $\nabla^2 f(\bm{\theta}^*;\bm{x}) \equiv \nabla^2 F(\bm{\theta}^*) = I$. 
Thus, the plug-in estimator for the asymptotic covariance of ROOT-SGD reduces to $\widehat{A}_t^{-1}\widehat{S}_t \widehat{A}_t^{-1}$. 
This allows a fast computation of the covariance estimator.

Using the mean image estimate and plug-in covariance estimate, we further compute the confidence interval for the mean image. 
We show examples of mean image estimation and the 95$\%$ confidence bound of mean estimation of digits 0,1,2 in Figure \ref{fig:mnist_mean}. 
In Figure \ref{fig:mnist_mean}, the confidence intervals all look reasonable.  

\begin{figure}[ht]
    \centering
    \includegraphics[scale = 0.4]{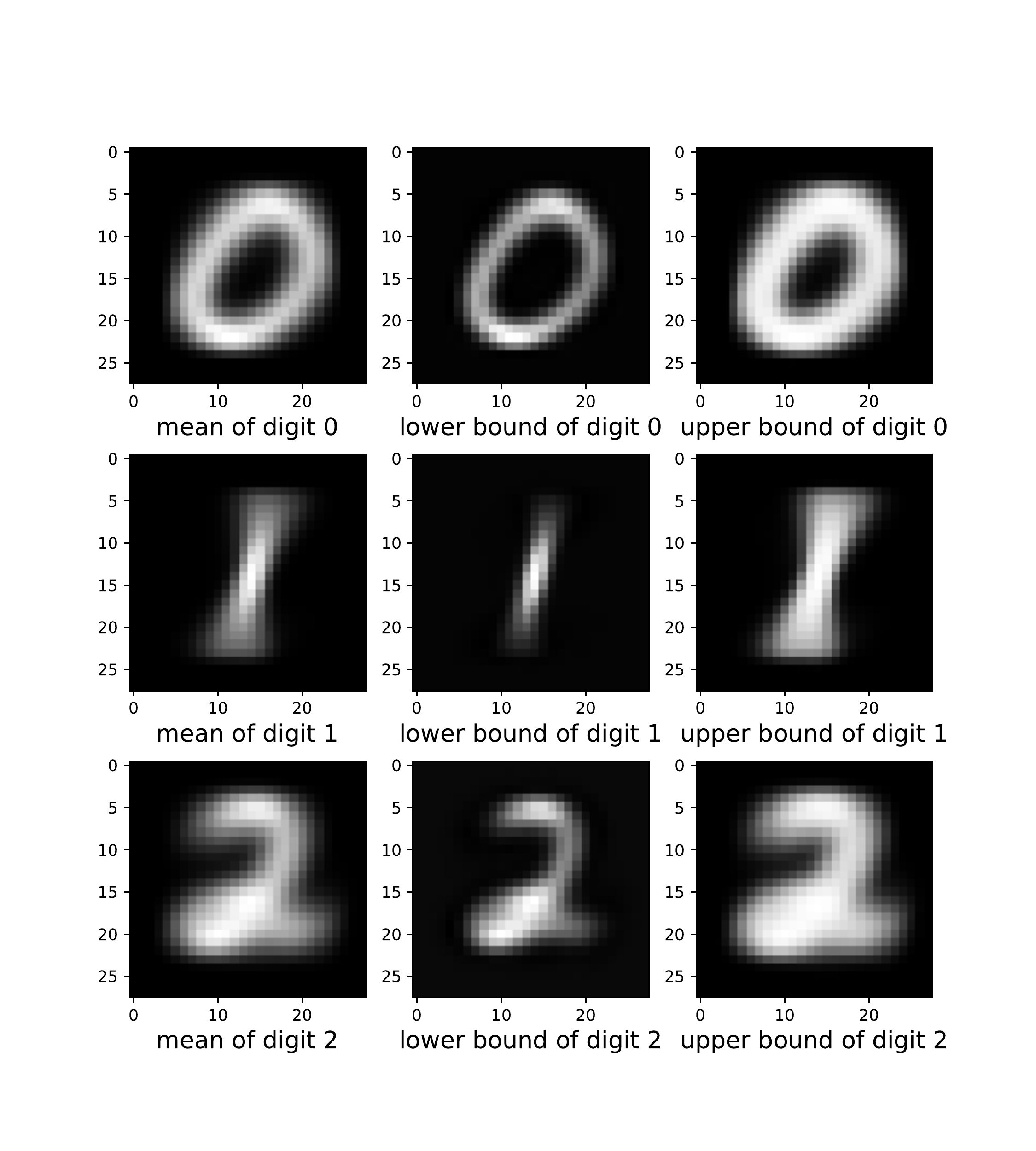}
    \caption{MNIST example: the mean and confidence intervals are reasonably estimated in an online fashion.}
    \label{fig:mnist_mean}
\end{figure}

\section{Discussion and Future Direction}\label{sec:05}
In this paper, we provide two online covariance estimators for the ROOT-SGD algorithm. 
Our plug-in covariance estimator converges to the true asymptotic covariance at the optimal rate $\mathcal{O}(1/\sqrt{t})$, where $t$ is the sample size. 
Our random-scaling covariance estimator is asymptotically consistent. 
Comparing the two estimators, the random-scaling covariance estimator requires less computation but does not have a guaranteed convergence rate. 
Numerical studies are performed to verify the asymptotic consistency of our covariance estimators.

We discuss some future directions in the following. 
We have developed asymptotically consistent statistical inference using our covariance estimators in this paper. 
One can consider further improving the statistical inference in a finite learning case. 
For example, one may debias the point estimator, then use the debiased point estimator combined with our covariance estimators for statistical inference.  
We have shown statistical inference examples in linear regression, logistic regression, and normal mean estimation in this paper. 
One may further apply the statistical inference to other problems since our formulation can be adapted to a broad range of stochastic optimization problems. 
For example, one can consider the statistical inference for the entropic optimal transport \citep{bercu2021asymptotic}, which is an important problem in modern machine learning. 
In summary, our work can have broad applications that will be of interest not only to the statistics community but also to the machine learning community.

\bigskip
\begin{center}
{\large\bf SUPPLEMENTARY MATERIAL}
\end{center}

\begin{description}
\item[proof:] All omitted proof in the paper. (.pdf file)
\item[algos:] Code for implementing the statistical inference methods in Section \ref{sec:04}. (.py file)
\item[inference:] Jupyter notebook to reproduce the experimental result. (.ipynb file)
\end{description}

\bibliographystyle{apalike}

\bibliography{ref}

\begin{thebibliography}{}

\bibitem[Abadir and Paruolo, 1997]{abadir1997two}
Abadir, K.~M. and Paruolo, P. (1997).
\newblock Two mixed normal densities from cointegration analysis.
\newblock {\em Econometrica: Journal of the Econometric Society}, pages
  671--680.

\bibitem[Abadir and Paruolo, 2002]{abadir2002simple}
Abadir, K.~M. and Paruolo, P. (2002).
\newblock Simple robust testing of regression hypotheses: A comment.
\newblock {\em Econometrica}, 70(5):2097--2099.

\bibitem[Bercu and Bigot, 2021]{bercu2021asymptotic}
Bercu, B. and Bigot, J. (2021).
\newblock Asymptotic distribution and convergence rates of stochastic
  algorithms for entropic optimal transportation between probability measures.
\newblock {\em The Annals of Statistics}, 49(2):968--987.

\bibitem[Chen et~al., 2021]{chen2021online}
Chen, X., Lai, Z., Li, H., and Zhang, Y. (2021).
\newblock Online statistical inference for gradient-free stochastic
  optimization.
\newblock {\em arXiv preprint arXiv:2102.03389}.

\bibitem[Chen et~al., 2020]{chen2020statistical}
Chen, X., Lee, J.~D., Tong, X.~T., and Zhang, Y. (2020).
\newblock {Statistical inference for model parameters in stochastic gradient
  descent}.
\newblock {\em The Annals of Statistics}, 48(1):251 -- 273.

\bibitem[Defazio et~al., 2014]{NIPS2014_ede7e2b6}
Defazio, A., Bach, F., and Lacoste-Julien, S. (2014).
\newblock {SAGA}: A fast incremental gradient method with support for
  non-strongly convex composite objectives.
\newblock In {\em Advances in Neural Information Processing Systems},
  volume~27. Curran Associates, Inc.

\bibitem[Dozat, 2016]{dozat2016incorporating}
Dozat, T. (2016).
\newblock Incorporating {N}esterov momentum into {A}dam.
\newblock {\em ICLR 2016 Workshop}.

\bibitem[Fang et~al., 2018]{NEURIPS2018_1543843a}
Fang, C., Li, C.~J., Lin, Z., and Zhang, T. (2018).
\newblock {SPIDER}: Near-optimal non-convex optimization via stochastic
  path-integrated differential estimator.
\newblock In {\em Advances in Neural Information Processing Systems},
  volume~31. Curran Associates, Inc.

\bibitem[Hall and Heyde, 2014]{hall2014martingale}
Hall, P. and Heyde, C.~C. (2014).
\newblock {\em Martingale limit theory and its application}.
\newblock Academic press.

\bibitem[Johnson and Zhang, 2013]{johnson2013accelerating}
Johnson, R. and Zhang, T. (2013).
\newblock Accelerating stochastic gradient descent using predictive variance
  reduction.
\newblock In {\em Advances in neural information processing systems},
  volume~26, pages 315--323.

\bibitem[Kiefer et~al., 2000]{kiefer2000simple}
Kiefer, N.~M., Vogelsang, T.~J., and Bunzel, H. (2000).
\newblock Simple robust testing of regression hypotheses.
\newblock {\em Econometrica}, 68(3):695--714.

\bibitem[Kingma and Ba, 2014]{kingma2014adam}
Kingma, D.~P. and Ba, J. (2014).
\newblock Adam: A method for stochastic optimization.
\newblock {\em arXiv preprint arXiv:1412.6980}.

\bibitem[Lahiri and Lahiri, 2003]{lahiri2003resampling}
Lahiri, S. and Lahiri, S. (2003).
\newblock {\em Resampling methods for dependent data}.
\newblock Springer Science \& Business Media.

\bibitem[Lee et~al., 2021]{lee2021fast}
Lee, S., Liao, Y., Seo, M.~H., and Shin, Y. (2021).
\newblock Fast and robust online inference with stochastic gradient descent via
  random scaling.
\newblock {\em arXiv preprint arXiv:2106.03156}.

\bibitem[Li et~al., 2022]{li2022root}
Li, C.~J., Mou, W., Wainwright, M., and Jordan, M. (2022).
\newblock {ROOT-SGD}: Sharp nonasymptotics and asymptotic efficiency in a
  single algorithm.
\newblock In {\em Conference on Learning Theory}, pages 909--981. PMLR.

\bibitem[Nguyen et~al., 2017]{pmlr-v70-nguyen17b}
Nguyen, L.~M., Liu, J., Scheinberg, K., and Tak{\'a}{\v{c}}, M. (2017).
\newblock {SARAH}: A novel method for machine learning problems using
  stochastic recursive gradient.
\newblock In {\em Proceedings of the 34th International Conference on Machine
  Learning}, volume~70, pages 2613--2621. PMLR.

\bibitem[Politis et~al., 1999]{politis1999subsampling}
Politis, D.~N., Romano, J.~P., and Wolf, M. (1999).
\newblock {\em Subsampling}.
\newblock Springer Science \& Business Media.

\bibitem[Polyak and Juditsky, 1992]{polyak1992acceleration}
Polyak, B.~T. and Juditsky, A.~B. (1992).
\newblock Acceleration of stochastic approximation by averaging.
\newblock {\em SIAM journal on control and optimization}, 30(4):838--855.

\bibitem[Robbins and Monro, 1951]{robbins1951stochastic}
Robbins, H. and Monro, S. (1951).
\newblock A stochastic approximation method.
\newblock {\em The annals of mathematical statistics}, pages 400--407.

\bibitem[Su and Zhu, 2018]{su2018uncertainty}
Su, W.~J. and Zhu, Y. (2018).
\newblock Uncertainty quantification for online learning and stochastic
  approximation via hierarchical incremental gradient descent.
\newblock {\em arXiv preprint arXiv:1802.04876}.

\bibitem[Van~der Vaart, 2000]{van2000asymptotic}
Van~der Vaart, A.~W. (2000).
\newblock {\em Asymptotic statistics}, volume~3.
\newblock Cambridge university press.

\bibitem[Zhou, 2018]{zhou2018duality}
Zhou, X. (2018).
\newblock On the {F}enchel duality between strong convexity and {L}ipschitz
  continuous gradient.
\newblock {\em arXiv preprint arXiv:1803.06573}.

\bibitem[Zhu et~al., 2021]{zhu2021online}
Zhu, W., Chen, X., and Wu, W.~B. (2021).
\newblock Online covariance matrix estimation in stochastic gradient descent.
\newblock {\em Journal of the American Statistical Association}, pages 1--12.

\end{thebibliography}

\newpage
\appendix
\section{Proof for Proposition \ref{prop:root_sgd_convergence_exponential}}\label{app:prof_root_sgd_convergence_exponential}
In this section, we prove Proposition \ref{prop:root_sgd_convergence_exponential}. 
\begin{proof}
We start the proof by writing the equivalent condition for inequality \eqref{eq:linear_conv_exponential_family}. 
\begin{align*}
    \eqref{eq:linear_conv_exponential_family} &\Longleftrightarrow \|\widehat{\bm{\theta}}_t - \bm{\theta}^*_t\|_2^2 \leq (1 - \eta\alpha)^2 \|\widehat{\bm{\theta}}_{t-1} - \bm{\theta}^*_t\|_2^2\\
    &\Longleftrightarrow \left\|\widehat{\bm{\theta}}_{t-1} - \eta \left(\nabla B(\widehat{\bm{\theta}}_{t-1}) - \nabla B(\bm{\theta}^*_t)\right) - \bm{\theta}^*_t\right\|_2^2 \leq (1 - \eta\alpha)^2 \|\widehat{\bm{\theta}}_{t-1} - \bm{\theta}^*_t\|_2^2\\
    &\Longleftrightarrow \left\|\nabla B(\widehat{\bm{\theta}}_{t-1}) - \nabla B(\bm{\theta}^*_t)\right\|_2^2 - \frac{2}{\eta}\left\langle \nabla B(\widehat{\bm{\theta}}_{t-1}) - \nabla B(\bm{\theta}^*_t), \widehat{\bm{\theta}}_{t-1} - \bm{\theta}^*_t\right\rangle \\
    &\quad\quad + \alpha(\frac{2}{\eta} - \alpha) \|\widehat{\bm{\theta}}_{t-1} - \bm{\theta}^*_t\|_2^2 \leq 0
\end{align*}
Since $B(\cdot)$ is smooth and strongly convex, by some standard equivalent conditions of smoothness and strongly convexity (for example, Lemma 2 and Lemma 4 in \citep{zhou2018duality}), we have
\begin{align*}
&\left\langle \nabla B(\widehat{\bm{\theta}}_{t-1}) - \nabla B(\bm{\theta}^*_t), \widehat{\bm{\theta}}_{t-1} - \bm{\theta}^*_t\right\rangle\\
\geq& \mu \|\widehat{\bm{\theta}}_{t-1} - \bm{\theta}^*_t\|_2^2\geq \frac{\mu}{l} \|\widehat{\bm{\theta}}_{t-1} - \bm{\theta}^*_t\|_2 \left\|\nabla B(\widehat{\bm{\theta}}_{t-1}) - \nabla B(\bm{\theta}^*_t)\right\|_2.
\end{align*}
For the convenience of proof, we define a new quantity $\kappa = \frac{\eta(\mu+l)}{2}$.
Then by our range of $\eta$, we immediately have $\kappa\in (0, \frac{\mu}{l}]$. 
Using this defined quantity, a sufficient condition for \eqref{eq:linear_conv_exponential_family} to hold is:
\begin{align}
    &\left\|\nabla B(\widehat{\bm{\theta}}_{t-1}) - \nabla B(\bm{\theta}^*_t)\right\|_2^2 - \frac{2\kappa}{\eta} \|\widehat{\bm{\theta}}_{t-1} - \bm{\theta}^*_t\|_2 \left\|\nabla B(\widehat{\bm{\theta}}_{t-1}) - \nabla B(\bm{\theta}^*_t)\right\|_2 \nonumber \\
    &\quad\quad + \alpha(\frac{2}{\eta} - \alpha) \|\widehat{\bm{\theta}}_{t-1} - \bm{\theta}^*_t\|_2^2 \leq 0. \label{eq:linear_conv_step1}
\end{align}
Take $\alpha = \frac{\mu+l - \sqrt{(\mu+l)^2 - 4\mu l\kappa^2}}{2\kappa}\in(0,\frac{1}{\eta})$, then we have
\begin{align}
    &\left\|\nabla B(\widehat{\bm{\theta}}_{t-1}) - \nabla B(\bm{\theta}^*_t)\right\|_2^2 - \frac{2\kappa}{\eta} \|\widehat{\bm{\theta}}_{t-1} - \bm{\theta}^*_t\|_2 \left\|\nabla B(\widehat{\bm{\theta}}_{t-1}) - \nabla B(\bm{\theta}^*_t)\right\|_2 \nonumber \\
    &\quad\quad + \alpha(\frac{2}{\eta} - \alpha) \|\widehat{\bm{\theta}}_{t-1} - \bm{\theta}^*_t\|_2^2 \nonumber\\
    =& \left(\left\|\nabla B(\widehat{\bm{\theta}}_{t-1}) - \nabla B(\bm{\theta}^*_t)\right\|_2 - \mu \|\widehat{\bm{\theta}}_{t-1} - \bm{\theta}^*_t\|_2\right)\nonumber\\
    & \quad* \left(\left\|\nabla B(\widehat{\bm{\theta}}_{t-1}) - \nabla B(\bm{\theta}^*_t)\right\|_2 - l \|\widehat{\bm{\theta}}_{t-1} - \bm{\theta}^*_t\|_2\right)\nonumber\\
    \leq& 0, \label{eq:linear_conv_step2}
\end{align}
where \eqref{eq:linear_conv_step2} by $\mu \|\widehat{\bm{\theta}}_{t-1} - \bm{\theta}^*_t\|_2\leq \left\|\nabla B(\widehat{\bm{\theta}}_{t-1}) - \nabla B(\bm{\theta}^*_t)\right\|_2\leq l \|\widehat{\bm{\theta}}_{t-1} - \bm{\theta}^*_t\|_2$. 
This conclude the proof for Proposition \ref{prop:root_sgd_convergence_exponential}. 
\end{proof}

\section{Connections Between Assumptions}\label{app:connect_assumption}
In Section \ref{sec:2.2}, we list several assumptions. 
Among those assumptions, Assumptions \ref{ass:smooth} and \ref{ass:6} are consequences of other assumptions. 
In this section, we prove such statement. 

\begin{proof}
We check that Assumption \ref{ass:smooth} is implied by Assumptions \ref{ass:sc} and \ref{ass:4} as follows. 
Assume Assumptions \ref{ass:sc} and \ref{ass:4} hold, then we can bound
\begin{align*}
    &\mathbbm{E}\|\nabla f( \bm{\theta}_1;\bm{x}) - \nabla f( \bm{\theta}_2;\bm{x})\|_2^2\\
    =& \|\nabla F( \bm{\theta}_1) - \nabla F( \bm{\theta}_2)\|_2^2 + \mathbbm{E}\|\delta ( \bm{\theta}_1;\bm{x}) - \delta( \bm{\theta}_2;\bm{x})\|_2^2\\
    \leq & (L_1^2 + L_2^2) \|\bm{\theta}_1 - \bm{\theta}_2\|_2^2. 
\end{align*} 
Thus, Assumption \ref{ass:smooth} must hold with $L_5^2 = L_1^2 + L_2^2$.

We check that Assumption \ref{ass:6} is implied by Assumption \ref{ass:5} as follows. 
Assume Assumption \ref{ass:5} holds, then the first inequality in Assumption \ref{ass:6} can be derived by: 
\begin{align}
    &\mathbbm{E}\|\nabla^2 f(\bm{\theta};\bm{x})\otimes \nabla^2 f(\bm{\theta};\bm{x}) - \nabla^2 f(\bm{\theta}^*;\bm{x})\otimes \nabla^2 f(\bm{\theta}^*;\bm{x})\|_2 \nonumber\\
    \leq & \mathbbm{E}\| (\nabla^2 f(\bm{\theta};\bm{x}) - \nabla^2 f(\bm{\theta}^*;\bm{x}))  \otimes (\nabla^2 f(\bm{\theta};\bm{x}) -\nabla^2 f(\bm{\theta}^*;\bm{x}))\|_2 \nonumber\\
    & + \mathbbm{E}\|\nabla^2 f(\bm{\theta}^*;\bm{x}) \otimes (\nabla^2 f(\bm{\theta};\bm{x}) -\nabla^2 f(\bm{\theta}^*;\bm{x}))\|_2 \nonumber\\
    & + \mathbbm{E}\|(\nabla^2 f(\bm{\theta};\bm{x}) -\nabla^2 f(\bm{\theta}^*;\bm{x}))\otimes \nabla^2 f(\bm{\theta}^*;\bm{x})\|_2\nonumber\\
    = & \mathbbm{E}\|\nabla^2 f(\bm{\theta};\bm{x}) - \nabla^2 f(\bm{\theta}^*;\bm{x})\|_2^2 + 2\mathbbm{E}\|\nabla^2 f(\bm{\theta};\bm{x}) -\nabla^2 f(\bm{\theta}^*;\bm{x})\|_2 \|\nabla^2 f(\bm{\theta}^*;\bm{x})\|_2\nonumber\\
    \leq & L_3^2\|\bm{\theta} - \bm{\theta}^*\|_2^2 + 2 \sqrt{\mathbbm{E}\|\nabla^2 f(\bm{\theta};\bm{x}) -\nabla^2 f(\bm{\theta}^*;\bm{x})\|_2^2 \mathbbm{E}\|\nabla^2 f(\bm{\theta}^*;\bm{x})\|_2^2}\nonumber\\
    \leq & L_3^2\|\bm{\theta} - \bm{\theta}^*\|_2^2 + 2 L_3 L_4^2\|\bm{\theta} - \bm{\theta}^*\|_2. \label{eq:assum5_assum6a}
\end{align}

And the second inequality of Assumption \ref{ass:6} can be obtained as:
\begin{align}
    & \mathbbm{E} \|\nabla^2 f(\bm{\theta}^*;\bm{x})\otimes \nabla^2 f(\bm{\theta}^*;\bm{x}) - \mathbbm{E}[\nabla^2 f(\bm{\theta}^*;\bm{x})\otimes \nabla^2 f(\bm{\theta}^*;\bm{x})]\|_2^2 \nonumber\\
    \leq & \mathbbm{E} \|\nabla^2 f(\bm{\theta}^*;\bm{x})\otimes \nabla^2 f(\bm{\theta}^*;\bm{x})\|_2^2 = \mathbbm{E} \|\nabla^2 f(\bm{\theta}^*;\bm{x})\|_2^4 \leq L_4^4.\label{eq:assum5_assum6b}
\end{align}
In this way, Assumption \ref{ass:6} must hold with $L_6 = 2L_3L_4^2$, $L_6' = L_3^2$, and $L_7 = L_4^4$.
\end{proof}

\section{Check Assumptions for Common Examples}\label{app:A}
In this section, we check that all assumptions in Section \ref{sec:2.2} hold for the common parameter estimation problems of exponential family model, linear regression and logistic regression. 
\subsection{Exponential Family Model}
We consider the exponential family model with natural parameters, where the likelihood function is
\[L(\bm{\theta};\bm{x}) = h(\bm{x})\exp[\langle\bm{\theta}, \bm{T}(\bm{x})\rangle - B(\bm{\theta})].\]
Let the risk function be the negative log-likelihood
\[f(\bm{\theta};\bm{x}) = -\langle\bm{\theta}, \bm{T}(\bm{x})\rangle + B(\bm{\theta}).\]
Then 
\[\bm{\theta}^* = \arg\min F(\bm{\theta}) = -\langle\bm{\theta}, \mathbbm{E}\bm{T}(\bm{x})\rangle + B(\bm{\theta}),\]
thus $\mathbbm{E}T(\bm{x}) = \nabla B(\bm{\theta}^*)$. 
Compute that 
\[\nabla f(\bm{\theta};\bm{x}) = - \bm{T}(\bm{x}) + \nabla B(\bm{\theta});\]
\[\nabla^2 f(\bm{\theta};\bm{x}) = \nabla^2 B(\bm{\theta}).\]
\[ \nabla F(\bm{\theta}) = - \mathbbm{E}\bm{T}(\bm{x}) + \nabla B(\bm{\theta});\]
\[ \nabla^2 F(\bm{\theta}) = \nabla^2 B(\bm{\theta}).\]

Suppose that $B(\cdot)$ is $\mu$-strongly convex and $l$-smooth, $\nabla^2 B(\cdot)$ is $l'$-Lipchitz continuous, and the true parameter $\bm{\theta}^*$ is taken such that the second and fourth moment of $\bm{T}(\bm{x})$ are bounded. 
Then we can check that all assumptions in Section \ref{sec:2.2} hold as follows. 

\begin{itemize}

\item \textbf{For Assumption \ref{ass:sc}:} $\nabla^2 F(\bm{\theta}) = \nabla^2 B(\bm{\theta})$, so $F(\cdot)$ is $\gamma$-strongly convex and $L_1$-smooth for $\gamma = \mu$, $L_1 = l$.

\item \textbf{For Assumption \ref{ass:bv}:} It is clear that $\nabla f(\bm{\theta};\bm{x})$ is unbiased, and we check its covariance at $\bm{\theta}^*$
\[\mathbbm{E} [\|\nabla f(\bm{\theta}^*;\bm{x}) - \nabla F(\bm{\theta}^*)\|_2^2] = \mathbbm{E} \|\bm{T}(\bm{x}) - \mathbbm{E}\bm{T}(\bm{x})\|_2^2\leq \sigma_1^2\]
for a $\sigma_1^2 <\infty$ since the second moment of $\bm{T}(\bm{x})$ is bounded. 

\item \textbf{For Assumption \ref{ass:4}:} We have $\delta(\bm{\theta};\bm{x}) = \mathbbm{E}\bm{T}(\bm{x}) - \bm{T}(\bm{x})$, then 
\[ \mathbbm{E} \|\delta(\bm{\theta}_1;\bm{x}) - \delta(\bm{\theta}_2;\bm{x})\|_2^2 = 0,\]
so Assumption \ref{ass:4} holds with $L_2^2 = 0$.

\item \textbf{For Assumption \ref{ass:5}:}  
\[\mathbbm{E}\|\nabla^2 f(\bm{\theta};\bm{x}) - \nabla^2 f(\bm{\theta}^*;\bm{x})\|_2^2= \|\nabla^2 B(\bm{\theta}) - \nabla^2 B(\bm{\theta}^*)\|_2^2  \leq  (l')^2 \|\bm{\theta} - \bm{\theta}^*\|_2^2,\]
\[\mathbbm{E}\|\nabla f(\bm{\theta}^*;\bm{x})\|_2^4 =  \mathbbm{E}\|\nabla B(\bm{\theta}^*) - \bm{T}(\bm{x})\|_2^4 =  \mathbbm{E}\|\mathbbm{E}[\bm{T}(\bm{x})] - \bm{T}(\bm{x}) \|_2^4 \leq l_4^4, \]
for a $l_4^4<\infty$ since the fourth moment of $\bm{T}(\bm{x})$ is bounded.
\[\mathbbm{E}\|\nabla^2 f(\bm{\theta}^*;\bm{x})\|_2^4  = \|\nabla^2 B(\bm{\theta}^*)\|_2^4 \leq l^4.\]
So Assumption \ref{ass:5} holds with $L_3^2 = (l')^2$, $L_4^4 = l^4$.

\item \textbf{For Assumption \ref{ass:smooth}:} 
$$\mathbbm{E}\|\nabla f( \bm{\theta}_1;\bm{x}) - \nabla f( \bm{\theta}_2;\bm{x})\|_2^2 = \|\nabla B(\bm{\theta}_1) - \nabla B(\bm{\theta}_2)\|_2^2
\leq l^2 \| \bm{\theta}_1 -  \bm{\theta}_2\|_2^2.$$
So Assumption \ref{ass:smooth} holds with $L_5^2 = l^2$.

\item \textbf{For Assumption \ref{ass:6}:}
By Remark \ref{remark:02}, we have
\begin{align*}
&\mathbbm{E}\|\nabla^2 f(\bm{\theta};\bm{x})\otimes \nabla^2 f(\bm{\theta};\bm{x}) - \nabla^2 f(\bm{\theta}^*;\bm{x})\otimes \nabla^2 f(\bm{\theta}^*;\bm{x})\|_2\\
\leq & L_3^2\|\bm{\theta} - \bm{\theta}^*\|_2^2 + 2 L_3 L_4^2\|\bm{\theta} - \bm{\theta}^*\|_2.
\end{align*}
And we check that
\begin{align*}
    &\mathbbm{E} \|\nabla^2 f(\bm{\theta}^*;\bm{x})\otimes \nabla^2 f(\bm{\theta}^*;\bm{x}) - \mathbbm{E}[\nabla^2 f(\bm{\theta}^*;\bm{x})\otimes \nabla^2 f(\bm{\theta}^*;\bm{x})]\|_2^2\\
    \leq& \mathbbm{E} \|\nabla^2 B(\bm{\theta}^*) \otimes \nabla^2 B(\bm{\theta}^*) - \mathbbm{E}\nabla^2 B(\bm{\theta}^*) \otimes \nabla^2 B(\bm{\theta}^*) \|_2^2 = 0.
\end{align*}
So Assumption \ref{ass:6} holds with $L_6^\prime = L_3^2 = (l')^2, L_6 =2L_3 L_4^2 = 2l' l^2$ and $L_7 = 0$.
\end{itemize}

\subsection{Linear Regression Model}\label{app:lr_model}
Consider the standard linear regression framework, where data $\bm{x}_i^T = (\bm{a}_i^T,b_i)$. 
The input $\bm{a}_i\in R^p$ are random samples from the same multivariate distribution, and the response $b_i = \bm{a}_i^T \bm{\theta}^* + \epsilon_i$, where $\bm{\theta}^*$ is the true parameter and $\epsilon_i$'s are i.i.d. centered random variables independent of $\bm{a}_i$'s. 
In particular, to track the dependence of Theorem \ref{thm:consis} on the problem dimension $p$, we further assume $\bm{a}\sim N(0,I_p), \epsilon\sim N(0,1)$. 

Use the squared loss
\[f(\bm{\theta};\bm{x}) = \frac{1}{2}(\bm{a}^T \bm{\theta} - b )^2.\]
Then 
\[\bm{\theta^*} = \arg\min F(\bm{\theta}) = \mathbbm{E}_{\bm{x}} \left[ \frac{1}{2}(\bm{a}^T \bm{\theta} - b )^2\right].\]
Compute that
\[\nabla f(\bm{\theta};\bm{x}) = \bm{a}(\bm{a}^T \bm{\theta} - b ) = \bm{a}\bm{a}^T (\bm{\theta} - \bm{\theta}^*) -  \bm{a}\epsilon ;\]
\[\nabla^2 f(\bm{\theta};\bm{x}) = \bm{a}\bm{a}^T;\]
\[\nabla F(\bm{\theta}) =  \mathbbm{E}[\nabla f(\bm{\theta}; \bm{x})] = \mathbbm{E}[\bm{a}\bm{a}^T] (\bm{\theta} - \bm{\theta}^*);\]
\[\nabla^2 F(\bm{\theta}) = \mathbbm{E}[\nabla^2 f(\bm{\theta};\bm{x})] = \mathbbm{E}[\bm{a}\bm{a}^T] = I_p.\]

We check assumptions in Section \ref{sec:2.2}:

\begin{itemize}

\item \textbf{For Assumption \ref{ass:sc}:} Since $\nabla^2 F(\bm{\theta}) = I_p$, $F(\bm{\theta})$ is $\gamma$-strongly convex and $L_1$-smooth with $\gamma = L_1 = 1$.

\item \textbf{For Assumption \ref{ass:bv}:} It is clear that $\nabla f(\bm{\theta};\bm{x})$ is unbiased, and we check its covariance at $\bm{\theta}^*$
\[\mathbbm{E} [\|\nabla f(\bm{\theta}^*;\bm{x}) - \nabla F(\bm{\theta}^*)\|_2^2] = \mathbbm{E} \|\bm{a}\epsilon  \|_2^2 = p,\]
so Assumption \ref{ass:bv} holds with $\sigma_1^2 = p$.

\item \textbf{For Assumption \ref{ass:4}:} We have $\delta(\bm{\theta};\bm{x}) = (\bm{a}\bm{a}^T - I_p) (\bm{\theta} - \bm{\theta}^*) -  \bm{a}\epsilon$, then 
\[ \mathbbm{E} \|\delta(\bm{\theta}_1;\bm{x}) - \delta(\bm{\theta}_2;\bm{x})\|_2^2 =  \mathbbm{E} \|(\bm{a}\bm{a}^T - I_p) (\bm{\theta}_1 - \bm{\theta}_2)\|_2^2
\leq \mathbbm{E} \|\bm{a}\bm{a}^T - I_p\|_2^2 \|\bm{\theta}_1 - \bm{\theta}_2\|_2^2,\]
where $\mathbbm{E} \|\bm{a}\bm{a}^T - I_p\|_2^2 \leq 2\mathbbm{E} \|\bm{a}\bm{a}^T\|_2^2 + 2\|I_p\|_2^2 = 2p+2$, so Assumption \ref{ass:4} holds with $L_2^2 =2p+2$. 

\item \textbf{For Assumption \ref{ass:5}:}  
\[\mathbbm{E}\|\nabla^2 f(\bm{\theta};\bm{x}) - \nabla^2 f(\bm{\theta}^*;\bm{x})\|_2^2= \mathbbm{E}\|\bm{a}\bm{a}^T - \bm{a}\bm{a}^T\|_2^2  = 0,\]
\[\mathbbm{E}\|\nabla f(\bm{\theta}^*;\bm{x})\|_2^4 =  \mathbbm{E}\|\bm{a}\epsilon\|_2^4  \leq  \mathbbm{E}\|\bm{a}\|_2^4 \|\epsilon\|_2^4 =\mathbbm{E}\|\bm{a}\|_2^4 \mathbbm{E} \|\epsilon\|_2^4 = 3(2p + p^2), \]
\[\mathbbm{E}\|\nabla^2 f(\bm{\theta}^*;\bm{x})\|_2^4  = \mathbbm{E}\| \bm{a}\bm{a}^T\|_2^4 = \mathbbm{E}\| \bm{a}\|_2^8 = p(p+2)(p+4)(p+6).\]
Assumption \ref{ass:5} holds with $L_3 = 0, l_4^4 = 3p(p+2), L_4^4 = p(p+2)(p+4)(p+6)$.

\item \textbf{For Assumption \ref{ass:smooth}:} 
$$\mathbbm{E}\|\nabla f( \bm{\theta}_1;\bm{x}) - \nabla f( \bm{\theta}_2;\bm{x})\|_2^2 = \mathbbm{E}\|\bm{a}\bm{a}^T(\bm{\theta}_1 - \bm{\theta}_2)\|_2^2
\leq \mathbbm{E}\|\bm{a}\bm{a}^T\|_2^2 \| \bm{\theta}_1 -  \bm{\theta}_2\|_2^2.$$
So Assumption \ref{ass:smooth} holds with $L_5^2 = \mathbbm{E}\|\bm{a}\bm{a}^T\|_2^2 = \mathbbm{E}\|\bm{a}\|_2^4 = 2p + p^2$.

\item \textbf{For Assumption \ref{ass:6}:}
\[ \mathbbm{E}\|\nabla^2 f(\bm{\theta};\bm{x})\otimes \nabla^2 f(\bm{\theta};\bm{x}) - \nabla^2 f(\bm{\theta}^*;\bm{x})\otimes \nabla^2 f(\bm{\theta}^*;\bm{x})\|_2 = 0,
\]
\begin{align*}
    &\mathbbm{E} \|\nabla^2 f(\bm{\theta}^*;\bm{x})\otimes \nabla^2 f(\bm{\theta}^*;\bm{x}) - \mathbbm{E}[\nabla^2 f(\bm{\theta}^*;\bm{x})\otimes \nabla^2 f(\bm{\theta}^*;\bm{x})]\|_2^2\\
    \leq& \mathbbm{E} \|\nabla^2 f(\bm{\theta}^*;\bm{x})\otimes \nabla^2 f(\bm{\theta}^*;\bm{x})\|_2^2\leq p(p+2)(p+4)(p+6).
\end{align*}
Assumption \ref{ass:6} holds with $L_6 = L_6^\prime = 0, L_7 = p(p+2)(p+4)(p+6)$.
\end{itemize}
For linear regression, we have the constant in Theorem \ref{thm:consis} is
\[C_p \lesssim \max\{\sigma_1/\gamma,L_3, \sqrt{L_4},l_4,\sqrt{L_5},L_6\} \asymp \sqrt{p}. \] %, $\|S\|_F = \|I_p\|_F = \sqrt{p}$. 

\subsection{Logistic Regression Model}
Consider the logistic regression model as follows. 
The data sample $\bm{x}_i^T = (\bm{a}_i^T,b_i) \in  R^p\times\{-1,1\}$. 
Suppose that $\bm{a}\sim N(0,I_p)$, and the data pair $(\bm{a}_i^T,b_i)$ is related by a  $\bm{\theta}^* \in  R^p$ such that $P(b_i = 1) = \frac{1}{1 + \exp(-\langle\bm{a}_i,\bm{\theta}^*\rangle)}$. 
Then the negative log-likelihood as objective function
$$f(\bm{\theta};\bm{x}) = \log(1 + \exp(- b \langle\bm{a},\bm{\theta}\rangle)),$$
then 
$$\nabla f(\bm{\theta};\bm{x}) = - \frac{b\bm{a}}{1 + \exp(b \langle\bm{a},\bm{\theta}\rangle)},$$
 $$\nabla^2 f(\bm{\theta};\bm{x}) = \frac{\exp(b \langle\bm{a},\bm{\theta}\rangle) \bm{a}\bm{a}^T}{(1 + \exp(b \langle\bm{a},\bm{\theta}\rangle))^2} = \frac{\bm{a}\bm{a}^T}{(1 + \exp(\langle\bm{a},\bm{\theta}\rangle))(1 + \exp( - \langle\bm{a},\bm{\theta}\rangle))}. $$
 
We check that all assumptions hold for logistic regression example. 
\begin{itemize}
\item \textbf{For Assumption \ref{ass:sc}:} 
Calculate that 
\[\nabla^2 F(\bm{\theta}) = \mathbbm{E}\left[\frac{ \bm{a}\bm{a}^T}{(1 + \exp( \langle\bm{a},\bm{\theta}\rangle))(1 + \exp(- \langle\bm{a},\bm{\theta}\rangle))}\right].\]
Then by Lemma A.3. in \citep{chen2020statistical}, $\nabla^2 F(\bm{\theta})$ is positive definite; thus $F(\bm{\theta})$ is strongly convex. 

Furthermore, we have $\exp(b \langle\bm{a},\bm{\theta}\rangle)/(1 + \exp(b \langle\bm{a},\bm{\theta}\rangle))^2\leq 1/4$. 
Thus
\[\nabla^2 F(\bm{\theta}) = \mathbbm{E}\left[\frac{\exp(b \langle\bm{a},\bm{\theta}\rangle) \bm{a}\bm{a}^T}{(1 + \exp(b \langle\bm{a},\bm{\theta}\rangle))^2}\right]\preceq \frac{1}{4} \mathbbm{E}[\bm{a}\bm{a}^T] = \frac{1}{4}I_p.\]
$F(\bm{\theta})$ is $1/4$-smooth.

\item  \textbf{For Assumption \ref{ass:bv}:} It is clear that $\nabla f(\bm{\theta};\bm{x})$ is unbiased, and we check its covariance at $\bm{\theta}^*$
\begin{align*}
&\mathbbm{E} [\|\nabla f(\bm{\theta}^*;\bm{x}) - \nabla F(\bm{\theta}^*)\|_2^2] \\
= &\mathbbm{E} [\|\nabla f(\bm{\theta}^*;\bm{x}) - \mathbbm{E}\nabla f(\bm{\theta}^*;\bm{x})\|_2^2]\\
\leq &\mathbbm{E} [\|\nabla f(\bm{\theta}^*;\bm{x})\|_2^2]\\
= & \mathbbm{E}\left[\left\vert-\frac{b}{1 + \exp(b \langle\bm{a},\bm{\theta}^*\rangle)}\right\vert*\|\bm{a}\|_2^2\right]\\
\leq & \mathbbm{E}\|\bm{a}\|_2^2 = p
\end{align*}
so Assumption \ref{ass:bv} holds with $\sigma_1^2 = p$.

\item \textbf{For Assumption \ref{ass:4}:} 
We have $\delta(\bm{\theta};\bm{x}) = \nabla f(\bm{\theta};\bm{x}) - \nabla F(\bm{\theta})$, then 
\[ \mathbbm{E} \|\delta(\bm{\theta}_1;\bm{x}) - \delta(\bm{\theta}_2;\bm{x})\|_2^2 =  2\mathbbm{E} \|\nabla f(\bm{\theta}_1;\bm{x}) - \nabla f(\bm{\theta}_2;\bm{x})\|_2^2 + 2\|\nabla F(\bm{\theta}_1) - \nabla F(\bm{\theta}_2)\|_2^2.\]
As we have checked in Assumption \ref{ass:sc},
\[\|\nabla F(\bm{\theta}_1) - \nabla F(\bm{\theta}_2)\|_2^2\leq \frac{1}{16}\|\bm{\theta}_1 - \bm{\theta}_2\|_2^2.\]

Since $\exp(b \langle\bm{a},\bm{\theta}\rangle)/(1 + \exp(b \langle\bm{a},\bm{\theta}\rangle))^2\leq 1/4$, we have
$$\nabla^2 f(\bm{\theta};\bm{x}) = \frac{\exp(b \langle\bm{a},\bm{\theta}\rangle) \bm{a}\bm{a}^T}{(1 + \exp(b \langle\bm{a},\bm{\theta}\rangle))^2}\preceq \frac{1}{4} \bm{a}\bm{a}^T.$$
Thus
$$\mathbbm{E}\|\nabla f( \bm{\theta}_1;\bm{x}) - \nabla f( \bm{\theta}_2;\bm{x})\|_2^2 \leq  \mathbbm{E}\|\bm{a}\bm{a}^T/4\|_2^2 \|\bm{\theta}_1 - \bm{\theta}_2\|_2^2=(2p + p^2)/16\|\bm{\theta}_1 - \bm{\theta}_2\|_2^2.$$
Assumption \ref{ass:4} holds with $L_2^2 = (1 + 2p + p^2)/8$, thus $L_2 \sim \mathcal{O}(p)$. 

\item  \textbf{For Assumption \ref{ass:5}:}  
\begin{align*}
&\mathbbm{E}\|\nabla^2 f(\bm{\theta};\bm{x}) - \nabla^2 f(\bm{\theta}^*;\bm{x})\|_2^2\\
=&\mathbbm{E}\left\|\frac{\bm{a}\bm{a}^T}{(1 + \exp(\langle\bm{a},\bm{\theta}\rangle))(1 + \exp( - \langle\bm{a},\bm{\theta}\rangle))} - \frac{\bm{a}\bm{a}^T}{(1 + \exp(\langle\bm{a},\bm{\theta}^*\rangle))(1 + \exp( - \langle\bm{a},\bm{\theta}^*\rangle))}\right\|_2^2\\
=&\mathbbm{E}\Bigg\{\left(\frac{1}{(1 + \exp(\langle\bm{a},\bm{\theta}\rangle))(1 + \exp( - \langle\bm{a},\bm{\theta}\rangle))} - \frac{1}{(1 + \exp(\langle\bm{a},\bm{\theta}^*\rangle))(1 + \exp( - \langle\bm{a},\bm{\theta}^*\rangle))}\right)^2\\
&*\|\bm{a}\bm{a}^T\|_2^2\Bigg\}\\
\leq & \frac{1}{16}\mathbbm{E}(\langle\bm{a},\bm{\theta}^*\rangle - \langle\bm{a},\bm{\theta}\rangle)^2*\|\bm{a}\|_2^4\\
\leq& \frac{1}{16}\mathbbm{E}\|\bm{a}\|_2^6 \|\bm{\theta}^* - \bm{\theta}\|_2^2 = [p(p+2)(p+4)/16] \|\bm{\theta}^* - \bm{\theta}\|_2^2
\end{align*}
\[\mathbbm{E}\|\nabla f(\bm{\theta}^*;\bm{x})\|_2^4 \leq \mathbbm{E}\|\bm{a}\|_2^4  = 2p + p^2, \]
\[\mathbbm{E}\|\nabla^2 f(\bm{\theta}^*;\bm{x})\|_2^4 \leq \mathbbm{E}\| \bm{a}\bm{a}^T/4\|_2^4 = \mathbbm{E}\| \bm{a}\|_2^8/256 = p(p+2)(p+4)(p+6)/256.\]
Assumption \ref{ass:5} holds with $L_3^2 = p(p+2)(p+4)/16, l_4^4 = p(p+2), L_4^4 = p(p+2)(p+4)(p+6)/256$. So $L_3 \sim \mathcal{O}(p^{1.5})$, $l_4 \sim \mathcal{O}(p^{.5})$, $L_4 \sim \mathcal{O}(p)$.

\item \textbf{For Assumption \ref{ass:smooth}:} 
Since $\exp(b \langle\bm{a},\bm{\theta}\rangle)/(1 + \exp(b \langle\bm{a},\bm{\theta}\rangle))^2\leq 1/4$, we have
$$\nabla^2 f(\bm{\theta};\bm{x}) = \frac{\exp(b \langle\bm{a},\bm{\theta}\rangle) \bm{a}\bm{a}^T}{(1 + \exp(b \langle\bm{a},\bm{\theta}\rangle))^2}\preceq \frac{1}{4} \bm{a}\bm{a}^T.$$
Thus
$$\mathbbm{E}\|\nabla f( \bm{\theta}_1;\bm{x}) - \nabla f( \bm{\theta}_2;\bm{x})\|_2^2 \leq  \mathbbm{E}\|\bm{a}\bm{a}^T/4\|_2^2 \|\bm{\theta}_1 - \bm{\theta}_2\|_2^2.$$
So Assumption \ref{ass:smooth} holds with $L_5^2 = \mathbbm{E}\|\bm{a}\bm{a}^T\|_2^2/16 = \mathbbm{E}\|\bm{a}\|_2^4/16 = (2p + p^2)/16$, thus $L_5\sim \mathcal{O}(p)$.

\item \textbf{For Assumption \ref{ass:6}:}
By Remark \ref{remark:02}, Assumption \ref{ass:6} holds with $L_6^\prime = L_3^2$, $L_6 = 2L_3 L_4^2$ and $L_7 = L_4^4 $. So $L_6^\prime \sim \mathcal{O}(p^3)$, $L_6 \sim \mathcal{O}(p^{3.5})$, $L_7 \sim \mathcal{O}(p^4)$.
\end{itemize}

\section{Ancillary Lemmas and Proof for Theorem \ref{thm:consis}}\label{app:B}
In this section, we prove Theorem \ref{thm:consis}. 
We first provide some ancillary lemmas that will be used to prove Theorem \ref{thm:consis} in Section \ref{app:thm1_lemma}. 
Then we give the proof of Theorem \ref{thm:consis} in Section \ref{app:thm1_proof}. 

\subsection{Lemmas}\label{app:thm1_lemma}
We first review some ancillary lemmas that will be useful in proving Theorem \ref{thm:consis}. 

\begin{lemma}[Implication of strong convexity]
The following condition is implied by $\gamma-$strong convexity for a differentiable function $F$:
\[\|\nabla F(\bm{\theta}_1) - \nabla F(\bm{\theta}_2)\|_2\geq \gamma\|\bm{\theta}_1-\bm{\theta}_2\|_2, \forall \bm{\theta}_1,\bm{\theta}_2.\]
\end{lemma}

\begin{lemma}[Kronecker product rule] We have
\begin{align*}
    \VEC(A B C) = (C^T\otimes A) \VEC(B).
\end{align*}
Suppose that $A$ and $B$ are square matrices of size $m$ and $n$, respectively. Let $\lambda_1,\cdots,\lambda_m$ be the eigenvalues of $A$ and $\mu_1,\cdots,\mu_n$ be the eigenvalues of $B$, then the eigenvalues of $A\otimes B$ are
\[\lambda_i\mu_j, i=1,\cdots,m, j = 1,\cdots,n.\]
Thus
\[\|A\otimes B\|_2 = \|A\|_2 \|B\|_2,\]
and  
\[\|AB\|_F = \|\VEC(AB)\|_2 = \|(I\otimes A)\VEC(B)\|_2\leq \|I\otimes A\|_2\|\VEC(B)\|_2 = \|A\|_2 \|B\|_F ,\]
similarly, 
\[\|AB\|_F\leq \|B\|_2 \|A\|_F.\]
\end{lemma}

\begin{lemma}[Matrix perturbation for inverse, Lemma C.1 in \citep{chen2020statistical}]
Let $B = A + E$ and assume that $A,B$ are invertible. If $\|A^{-1}E\|_2 < 1/2$, we have
\begin{equation}\label{eq:inv_pertub}
    \|B^{-1} - A^{-1}\|_2 \leq 2 \|E\|_2 \|A^{-1}\|_2^2.
\end{equation}
\end{lemma}

\begin{lemma}[Asymptotic convergence of ROOT-SGD, Proposition 2 in \citep{li2022root}]\label{lem:asymptotic_root}
Suppose that Assumptions \ref{ass:sc} to \ref{ass:5} hold. 
Then there exists constants $c_1$, $c_2$, given the step size $\eta\in \left(0,c_1(\frac{\gamma}{L_2^2} \wedge \frac{1}{L_1} \wedge \frac{\gamma^{1/3}}{L_4^{4/3}} )\right)$ and burn-in period $B = \left\lceil{\frac{c_2}{\gamma\eta}}\right\rceil$ we have
\begin{equation*}
   \sqrt{t}(\widehat{\bm{\theta}}_t - \bm{\theta}^*) \stackrel{d}{\to} \mathcal{N}(0,\Sigma)
\end{equation*}
for $\Sigma = A^{-1} (S + \mathbbm{E}[\Xi_{\bm{x}}(\bm{\theta}^*) \Lambda \Xi_{\bm{x}}(\bm{\theta}^*)]) A^{-1}$, where $\Xi_{\bm{x}}(\bm{\theta}) = \nabla^2 f(\bm{\theta};\bm{x}) - \nabla^2 F(\bm{\theta})$ (hide the random variable $x$) and $\Lambda$ is by solving the following equation in $\Lambda$:
\begin{equation}
    \Lambda A + A\Lambda - \eta \mathbbm{E}[\Xi_{\bm{x}}(\bm{\theta}^*)\Lambda \Xi_{\bm{x}}(\bm{\theta}^*)] - \eta A \Lambda A = \eta S. %\label{eq:lambda}
\end{equation}
\end{lemma}

\begin{lemma}[Convergence of ROOT-SGD, Theorem 5 in \citep{li2022root}]\label{prop:convergence}
Under Assumptions 1,2,3, take step size $\eta < \eta_{\max}:= \frac{\gamma}{8 L_2^2} \wedge\frac{1}{4 L_1} $ and choose the burn-in time
\[B:= \left\lceil{\frac{24}{\gamma\eta}}\right\rceil,\]
then for any iteration $t\geq B$, the iterate $\widehat{\bm{\theta}}_t$ from ROOT-SGD satisfies the bound
\[\mathbbm{E}\|\nabla F(\widehat{\bm{\theta}}_t)\|_2^2\leq \frac{2700\|\nabla F(\widehat{\bm{\theta}}_0)\|_2^2}{\eta^2\gamma^2(t+1)^2}+\frac{28\sigma_1^2}{t+1}.\]
Thus
\[\mathbbm{E} \|\widehat{\bm{\theta}}_t - \bm{\theta}^*\|_2^2 \leq \frac{1}{\gamma^2}\mathbbm{E}\|\nabla F(\widehat{\bm{\theta}}_t)\|_2^2\leq \frac{2700\|\nabla F(\widehat{\bm{\theta}}_0)\|_2^2}{\eta^2\gamma^4(t+1)^2}+\frac{28\sigma_1^2}{\gamma^2(t+1)}.\]
\end{lemma}

\subsection{Proof for Theorem \ref{thm:consis}}\label{app:thm1_proof}
In this section, we prove Theorem \ref{thm:consis}, the convergence of the plug-in estimator, given the lemmas in the previous section. 
Recall that the plug-in estimator has each term in the asymptotic covariance replaced by their empirical counterparts. 
We first bound the error between each term and its empirical estimator in Lemma \ref{lem:thm1_terms_err}. 
With those bounds, we finally prove Theorem \ref{thm:consis}.  
For notation simplicity, we denote $\widetilde{A}_t$, $\widetilde{P}_t$, $\widetilde{\Lambda}_t$, $\widetilde{\Sigma}_t$, $\widehat{A}_t$, $\widehat{P}_t$, $\widehat{\Lambda}_t$, $\widehat{\Sigma}_t$, $\widehat{S}_t$ as $\widetilde{A}$, $\widetilde{P}$, $\widetilde{\Lambda}$, $\widetilde{\Sigma}$, $\widehat{A}$, $\widehat{P}$, $\widehat{\Lambda}$, $\widehat{\Sigma}$, $\widehat{S}$, respectively, (i.e., omit all $t$s in the notation) throughout the proof. 

\begin{lemma}\label{lem:thm1_terms_err}
Under Assumptions \ref{ass:sc}-\ref{ass:6}, denote $C_p \lesssim \max\{\sigma_1/\gamma,L_3, \sqrt{L_4},l_4,\sqrt{L_5},L_6\}$, for step size $\eta < \min(\eta_{\max}, 2\delta /L_4^2)$ we have
\[\mathbbm{E} \|\widetilde{A} - A\|_2\lesssim C_p^2/\sqrt{t},\]
\[\mathbbm{E} \|\widetilde{A}^{-1} - A^{-1}\|_2\lesssim C_p^2/\sqrt{t},\]
\[\mathbbm{E} \|\widehat{S} - S\|_2\lesssim C_p^4 /\sqrt{t},\]
\[\mathbbm{E}\|\widetilde{\Lambda} - \Lambda\|_F \lesssim  [\eta \|\Lambda\|_F(L_4^2 + C_p^2) + \eta\sqrt{p}C_p^4 + \|\Lambda\|_2 C_p^2]/\sqrt{t}.\]
\end{lemma}
\begin{proof}
First, by Proposition \ref{prop:convergence}, we have the following. 
\begin{align*}\sum_{i=B+1}^{t} \mathbbm{E} \|\widehat{\bm{\theta}}_{i-1} - \bm{\theta}^*\|_2^2 &\leq \sum_{i=B+1}^{t}\frac{2700\|\nabla F(\widehat{\bm{\theta}}_0)\|_2^2}{\eta^2\gamma^4i^2}+\sum_{i=B+1}^{t}\frac{28\sigma_1^2}{\gamma^2 i}\\
& = \frac{2700\|\nabla F(\widehat{\bm{\theta}}_0)\|_2^2}{\eta^2\gamma^4}\sum_{i=B+1}^{t}\frac{1}{i^2}+\frac{28\sigma_1^2}{\gamma^2 }\sum_{i=B+1}^{t}\frac{1}{i}\\
&\leq  \frac{2700\|\nabla F(\widehat{\bm{\theta}}_0)\|_2^2}{\eta^2\gamma^4}\int_{i = B}^{t}\frac{1}{i^2} d i + \frac{28\sigma_1^2}{\gamma^2 }\int_{i=B}^{t}\frac{1}{i} d i\\
&\leq \frac{2700\|\nabla F(\widehat{\bm{\theta}}_0)\|_2^2}{\eta^2\gamma^4 B}+\frac{28\sigma_1^2 \log t}{\gamma^2 },
\end{align*}
and 
\begin{align*}\sum_{i=B+1}^{t} \mathbbm{E} \|\widehat{\bm{\theta}}_{i-1} - \bm{\theta}^*\|_2 &\leq \sum_{i=B+1}^{t} \sqrt{\mathbbm{E} \|\widehat{\bm{\theta}}_{i-1} - \bm{\theta}^*\|_2^2}\\
&\leq \sum_{i=B+1}^{t}\sqrt{\frac{2700\|\nabla F(\widehat{\bm{\theta}}_0)\|_2^2}{\eta^2\gamma^4i^2}+ \frac{28\sigma_1^2}{\gamma^2 i}}\\
& \leq \frac{\sqrt{2700}\|\nabla F(\widehat{\bm{\theta}}_0)\|_2}{\eta\gamma^2} \sum_{i=B+1}^{t}\frac{1}{i}+\frac{\sqrt{28}\sigma_1}{\gamma }\sum_{i=B+1}^{t}\frac{1}{\sqrt{i}}\\
&\leq  \frac{\sqrt{2700}\|\nabla F(\widehat{\bm{\theta}}_0)\|_2}{\eta\gamma^2}\int_{i = B}^{t}\frac{1}{i}+\frac{\sqrt{28}\sigma_1}{\gamma }\int_{i=B}^{t}\frac{1}{\sqrt{i}}\\
&\leq \frac{\sqrt{2700}\|\nabla F(\widehat{\bm{\theta}}_0)\|_2\log t}{\eta\gamma^2}+\frac{2\sqrt{28}\sigma_1 (\sqrt{t} -\sqrt{B})}{\gamma}.
\end{align*}
Thus 
\begin{align}
    &\frac{1}{t-B}\sum_{i=B+1}^{t} \mathbbm{E} \|\widehat{\bm{\theta}}_{i-1} - \bm{\theta}^*\|_2^2\lesssim \frac{C_p^2 \log t}{t}\label{eq:bound_dist1},\\
    &\frac{1}{t-B}\sum_{i=B+1}^{t} \mathbbm{E} \|\widehat{\bm{\theta}}_{i-1} - \bm{\theta}^*\|_2\lesssim \frac{C_p}{\sqrt{t}} \label{eq:bound_dist2}.
\end{align}
\textbf{For $\mathbbm{E}\|\widetilde{A} - A\|_2$: }
\begin{align*}
    \mathbbm{E}\|\widehat{A} - A\|_2 &=  \mathbbm{E} \left\|\frac{1}{t - B}\sum_{i=B+1}^{t}\nabla^2 f(\widehat{\bm{\theta}}_{i-1};\bm{x}_i) - \nabla^2 F(\bm{\theta}^*)\right\|_2 \\
    &\leq  \mathbbm{E} \left\|\frac{1}{t - B}\sum_{i=B+1}^{t}\nabla^2 f(\widehat{\bm{\theta}}_{i-1};\bm{x}_i) - \frac{1}{t - B}\sum_{i=B+1}^{t}\nabla^2 f(\bm{\theta}^*;\bm{x}_i)\right\|_2 \\
    &\quad +  \mathbbm{E} \left\|\frac{1}{t - B}\sum_{i=B+1}^{t}\nabla^2 f(\bm{\theta}^*;\bm{x}_i) - \nabla^2 F(\bm{\theta}^*)\right\|_2\\
    &\leq \frac{1}{t - B}\sum_{i=B+1}^{t} \mathbbm{E} \|\nabla^2 f(\widehat{\bm{\theta}}_{i-1};\bm{x}) - \nabla^2 f(\bm{\theta}^*;\bm{x})\|_2 \\
    &\quad +  \mathbbm{E} \left\|\frac{1}{t - B}\sum_{i=B+1}^{t}\nabla^2 f(\bm{\theta}^*;\bm{x}_i) - \nabla^2 F(\bm{\theta}^*)\right\|_2\\
    &\leq\frac{1}{t - B}\sum_{i=B+1}^{t} L_3 \mathbbm{E} \|\widehat{\bm{\theta}}_{i-1} - \bm{\theta}^*\|_2+ \mathbbm{E} \left\|\frac{1}{t - B}\sum_{i=B+1}^{t}\nabla^2 f(\bm{\theta}^*;\bm{x}_i) - \nabla^2 F(\bm{\theta}^*)\right\|_2\\
     &\leq\frac{L_3}{t - B}\sum_{i=B+1}^{t} \mathbbm{E} \|\widehat{\bm{\theta}}_{i-1} - \bm{\theta}^*\|_2+ \sqrt{\mathbbm{E} \left\|\frac{1}{t - B}\sum_{i=B+1}^{t}\nabla^2 f(\bm{\theta}^*;\bm{x}_i) - \nabla^2 F(\bm{\theta}^*)\right\|_2^2}\\
     &=
     \frac{L_3}{t - B}\sum_{i=B+1}^{t} \mathbbm{E} \|\widehat{\bm{\theta}}_{i-1} - \bm{\theta}^*\|_2+ \frac{1}{t - B}\sqrt{\sum_{i=B+1}^{t} \mathbbm{E} \|\nabla^2 f(\bm{\theta}^*;\bm{x}) - \nabla^2 F(\bm{\theta}^*)\|_2^2}\\
     &\leq \frac{L_3}{t - B}\sum_{i=B+1}^{t} \mathbbm{E} \|\widehat{\bm{\theta}}_{i-1} - \bm{\theta}^*\|_2+ \frac{1}{\sqrt{t - B}}\sqrt{\mathbbm{E} \|\nabla^2 f(\bm{\theta}^*;\bm{x})\|_2^2}\\
      &\leq \frac{L_3}{t - B}\sum_{i=B+1}^{t} \mathbbm{E} \|\widehat{\bm{\theta}}_{i-1} - \bm{\theta}^*\|_2+ \frac{1}{\sqrt{t - B}}(\mathbbm{E} \|\nabla^2 f(\bm{\theta}^*;\bm{x})\|_2^4)^{1/4}\\
      &\leq \frac{L_3}{t - B}\sum_{i=B+1}^{t} \mathbbm{E} \|\widehat{\bm{\theta}}_{i-1} - \bm{\theta}^*\|_2+ \frac{L_4}{\sqrt{t - B}}\\
      &\stackrel{\eqref{eq:bound_dist2}}{\lesssim} C_p^2/\sqrt{t},
\end{align*}
and 
\begin{align*}
    \mathbbm{E}\|\widehat{A} - \widetilde{A}\|_2 &= \mathbbm{E}[\mathbbm{1}_{\widehat{A}\succeq \delta I }*0 + (1-\mathbbm{1}_{\widehat{A}\succeq \delta I })*(\delta - \lambda_{\min}(\widehat{A}))]\\
    &\lesssim \mathbbm{E}[1-\mathbbm{1}_{\widehat{A}\succeq \delta I }]\\
    &= 1- P(\widehat{A}\succeq \delta I) \\
    &= 1- P(\lambda_{\min}(A + (\widehat{A} - A))\geq \delta ) \\
    &\stackrel{(i)}{\leq} 1 - P(\|\widehat{A} - A\|_2 \leq \lambda_{\min}(A) - \delta)\\
    &\stackrel{(ii)}{\leq } \frac{1}{\lambda_{\min}(A) - \delta} \mathbbm{E}\|\widehat{A} - A\|_2\asymp \mathbbm{E}\|\widehat{A} - A\|_2,
\end{align*}
where (i) by Weyl's inequality that $\lambda_{\min}(A + B)\geq \lambda_{\min}(A) - \|B\|_2$, and (ii) by Markov inequality. 
Thus, 
\begin{align*}
     \mathbbm{E}\|\widetilde{A} - {A}\|_2 &\leq  \mathbbm{E}\|\widehat{A} - A\|_2  +  \mathbbm{E}\|\widehat{A} - \widetilde{A}\|_2\lesssim C_p^2/\sqrt{t}.
\end{align*}
\textbf{For $\mathbbm{E}\|\widetilde{A}^{-1} - A^{-1}\|_2$:}\\
Denote $dA = \widetilde{A} - A$, then 
\begin{align*}
    \|\widetilde{A}^{-1} - A^{-1}\|_2 &\stackrel{\eqref{eq:inv_pertub}}{\leq} \mathbbm{1}_{\|A^{-1}dA\|_2\leq 1/2} 2 \|dA\|_2\|A^{-1}\|_2^2 + \mathbbm{1}_{\|A^{-1}dA\|_2> 1/2}(\|A^{-1}\|_2 + \|\widetilde{A}^{-1}\|_2)\\
    &\leq 2 \|dA\|_2\|A^{-1}\|_2^2 + \mathbbm{1}_{\|A^{-1}dA\|_2> 1/2}(\lambda_{\min}(A)^{-1} + \delta^{-1}),
\end{align*}
thus
\begin{align*}
    \mathbbm{E}\|\widetilde{A}^{-1} - A^{-1}\|_2 &\leq 2 \|A^{-1}\|_2^2 \mathbbm{E} \|dA\|_2+P(\|A^{-1}dA\|_2> 1/2)(\lambda_{\min}(A)^{-1} + \delta^{-1})\\
    &\leq 2 \|A^{-1}\|_2^2 \mathbbm{E} \|dA\|_2+ 2 (\lambda_{\min}(A)^{-1} + \delta^{-1})\mathbbm{E}\|A^{-1}dA\|_2 \\
    &\lesssim \mathbbm{E} \|dA\|_2\lesssim C_p^2 /\sqrt{t}.
\end{align*}
\textbf{For $\mathbbm{E}\|\widehat{S} - S\|_2$: }
\begin{align*}
    &\mathbbm{E}\|\widehat{S} - S\|_2 = \mathbbm{E}\left\|\frac{1}{t - B}\sum_{i=B+1}^{t}\nabla f(\widehat{\bm{\theta}}_{i-1};\bm{x}_i) \nabla f(\widehat{\bm{\theta}}_{i-1};\bm{x}_i)^T - \mathbbm{E}[\nabla f (\bm{\theta}^*;\bm{x}) \nabla f (\bm{\theta}^*;\bm{x})^T ]\right\|_2\\
    \leq & \mathbbm{E}\left\|\frac{1}{t - B}\sum_{i=B+1}^{t}\nabla f(\widehat{\bm{\theta}}_{i-1};\bm{x}_i) \nabla f(\widehat{\bm{\theta}}_{i-1};\bm{x}_i)^T - \frac{1}{t - B}\sum_{i=B+1}^{t}\nabla f(\bm{\theta}^*;\bm{x}_i) \nabla f(\bm{\theta}^*;\bm{x}_i)^T \right\|_2\\
    &+ \mathbbm{E}\left\|\frac{1}{t - B}\sum_{i=B+1}^{t}\nabla f(\bm{\theta}^*;\bm{x}_i) \nabla f(\bm{\theta}^*;\bm{x}_i)^T - \mathbbm{E}[\nabla f (\bm{\theta}^*;\bm{x}) \nabla f (\bm{\theta}^*;\bm{x})^T ]\right\|_2\\
    \leq &\frac{1}{t - B}\sum_{i=B+1}^{t} \mathbbm{E}\|\nabla f(\widehat{\bm{\theta}}_{i-1};\bm{x}_i) \nabla f(\widehat{\bm{\theta}}_{i-1};\bm{x}_i)^T - \nabla f(\bm{\theta}^*;\bm{x}_i) \nabla f(\bm{\theta}^*;\bm{x}_i)^T \|_2\\
    &+ \mathbbm{E}\left\|\frac{1}{t - B}\sum_{i=B+1}^{t}\nabla f(\bm{\theta}^*;\bm{x}_i) \nabla f(\bm{\theta}^*;\bm{x}_i)^T - \mathbbm{E}[\nabla f (\bm{\theta}^*;\bm{x}) \nabla f (\bm{\theta}^*;\bm{x})^T ]\right\|_2\\
    \leq& \mathbbm{E}\left\|\frac{1}{t - B}\sum_{i=B+1}^{t}\nabla f(\bm{\theta}^*;\bm{x}_i) \nabla f(\bm{\theta}^*;\bm{x}_i)^T - \mathbbm{E}[\nabla f (\bm{\theta}^*;\bm{x}) \nabla f (\bm{\theta}^*;\bm{x})^T ]\right\|_2 \\
    &+ \frac{1}{t - B}\sum_{i=B+1}^{t} \bigg[ \mathbbm{E}\|(\nabla f(\widehat{\bm{\theta}}_{i-1};\bm{x}_i) - \nabla f(\bm{\theta}^*;\bm{x}_i)) (\nabla f(\widehat{\bm{\theta}}_{i-1};\bm{x}_i) -  \nabla f(\bm{\theta}^*;\bm{x}_i))^T\|_2 \\
    &\quad + \mathbbm{E}\|\nabla f(\bm{\theta}^*;\bm{x}_i) (\nabla f(\widehat{\bm{\theta}}_{i-1};\bm{x}_i) -  \nabla f(\bm{\theta}^*;\bm{x}_i))^T\|_2 \\
    &\quad+ \mathbbm{E}\|(\nabla f(\widehat{\bm{\theta}}_{i-1};\bm{x}_i) -  \nabla f(\bm{\theta}^*;\bm{x}_i))\nabla f(\bm{\theta}^*;\bm{x}_i)^T\|_2\bigg]
\end{align*}
\begin{align*}
    \leq &\mathbbm{E}\left\|\frac{1}{t - B}\sum_{i=B+1}^{t}\nabla f(\bm{\theta}^*;\bm{x}_i) \nabla f(\bm{\theta}^*;\bm{x}_i)^T - \mathbbm{E}[\nabla f (\bm{\theta}^*;\bm{x}) \nabla f (\bm{\theta}^*;\bm{x})^T ]\right\|_2 \\
    &+ \frac{1}{t - B}\sum_{i=B+1}^{t} \Bigg[\mathbbm{E}\|\nabla f(\widehat{\bm{\theta}}_{i-1};\bm{x}_i) - \nabla f(\bm{\theta}^*;\bm{x}_i)\|_2^2 \\
    & \quad+ 2\sqrt{\mathbbm{E}\|\nabla f(\bm{\theta}^*;\bm{x}_i)\|_2^2 \mathbbm{E}\| \nabla f(\widehat{\bm{\theta}}_{i-1};\bm{x}_i) -  \nabla f(\bm{\theta}^*;\bm{x}_i)\|_2^2 }\Bigg]\\
    \leq &\underbrace{\mathbbm{E}\left\|\frac{1}{t - B}\sum_{i=B+1}^{t}\nabla f(\bm{\theta}^*;\bm{x}_i) \nabla f(\bm{\theta}^*;\bm{x}_i)^T - \mathbbm{E}[\nabla f (\bm{\theta}^*;\bm{x}) \nabla f (\bm{\theta}^*;\bm{x})^T ]\right\|_2}_{(1)} \\
    &+ \underbrace{\frac{1}{t - B}\sum_{i=B+1}^{t} [L_5^2 \mathbbm{E}\|\widehat{\bm{\theta}}_{i-1} - \bm{\theta}^*\|_2^2 + 2L_5 \sqrt{\mathbbm{E}\|\nabla f(\bm{\theta}^*;\bm{x}_i)\|_2^2 \mathbbm{E}\| \widehat{\bm{\theta}}_{i-1}-\bm{\theta}^*\|_2^2} ]}_{(2)}.
\end{align*}
Where term (1):
\begin{align*}
    &\mathbbm{E}\left\|\frac{1}{t - B}\sum_{i=B+1}^{t}\nabla f(\bm{\theta}^*;\bm{x}_i) \nabla f(\bm{\theta}^*;\bm{x}_i)^T - \mathbbm{E}[\nabla f (\bm{\theta}^*;\bm{x}) \nabla f (\bm{\theta}^*;\bm{x})^T ]\right\|_2\\
    \leq &\sqrt{\mathbbm{E}\left\|\frac{1}{t - B}\sum_{i=B+1}^{t}\nabla f(\bm{\theta}^*;\bm{x}_i) \nabla f(\bm{\theta}^*;\bm{x}_i)^T - \mathbbm{E}[\nabla f (\bm{\theta}^*;\bm{x}) \nabla f (\bm{\theta}^*;\bm{x})^T ]\right\|_2^2}\\
    = &\sqrt{\frac{1}{t - B}\mathbbm{E}\|\nabla f(\bm{\theta}^*;\bm{x}) \nabla f(\bm{\theta}^*;\bm{x})^T - \mathbbm{E}[\nabla f (\bm{\theta}^*;\bm{x}) \nabla f (\bm{\theta}^*;\bm{x})^T ]\|_2^2} \\
    \leq &\sqrt{\frac{4}{t - B}\mathbbm{E}\|\nabla f(\bm{\theta}^*;\bm{x}) \nabla f(\bm{\theta}^*;\bm{x})^T\|_2^2} \\
     = &\sqrt{\frac{4}{t - B}\mathbbm{E}\|\nabla f(\bm{\theta}^*;\bm{x})\|_2^4}\\
     \leq &\sqrt{\frac{4}{t - B} l_4^4}\lesssim C_p^2/\sqrt{t},
\end{align*}
and term (2):
\begin{align*}
    &\frac{1}{t - B}\sum_{i=B+1}^{t} \left[L_5^2 \mathbbm{E}\|\widehat{\bm{\theta}}_{i-1} - \bm{\theta}^*\|_2^2 + 2L_5 \sqrt{\mathbbm{E}\|\nabla f(\bm{\theta}^*;\bm{x}_i)\|_2^2 \mathbbm{E}\| \widehat{\bm{\theta}}_{i-1}-\bm{\theta}^*\|_2^2} \right]\\
    \leq &\frac{L_5^2}{t - B}\sum_{i=B+1}^{t} \mathbbm{E}\|\widehat{\bm{\theta}}_{i-1} - \bm{\theta}^*\|_2^2 + \frac{2L_5l_4}{t - B}\sum_{i=B+1}^{t} \sqrt{ \mathbbm{E}\| \widehat{\bm{\theta}}_{i-1}-\bm{\theta}^*\|_2^2} \\
    \lesssim & L_5^2 C_p^2 \log T/T + L_5 C_p^2 /\sqrt{t} \lesssim C_p^4/\sqrt{t}.
\end{align*}
Thus
\begin{align*}
    \mathbbm{E}\|\widehat{S}-S\|_2 \lesssim C_p^4 /\sqrt{t}.
\end{align*}
\textbf{For $\mathbbm{E}\|\widetilde{\Lambda} - \Lambda\|_F$: } \\
To bound $\mathbbm{E}\|\widetilde{\Lambda} - \Lambda\|_F$, we first claim the following hold, which we will check at the end of this proof:
\begin{align}
& \|P - \widetilde{P}\|_2 \lesssim L_4^2\label{eq:p1},\\
&\mathbbm{E}\|\widehat{P}- P\|_2 \lesssim (\sqrt{L_7} + C_p^2)/\sqrt{t},\label{eq:p2}\\
&\mathbbm{E}\|\widetilde{P} - P\|_2 \lesssim (\sqrt{L_7} + C_p^2)/\sqrt{t}.\label{eq:p3}
\end{align}
Let us denote 
\[dS = \widehat{S} - S, d\Lambda = \widetilde{\Lambda} - \Lambda.\]
Define a linear operator $L_{\widetilde{P}}$ such that \[\VEC(L_{\widetilde{P}}(\Lambda)) = \widetilde{P} \VEC(\Lambda),\]
then $\widetilde{\Lambda}$ satisfies the following matrix equality:
    \begin{align}\label{eq:plug_thres_lambda}
        \widetilde{\Lambda} \widetilde{A} + \widetilde{A}\widetilde{\Lambda} - \eta L_{\widetilde{P}}(\widetilde{\Lambda})= \eta \widehat{S}.
    \end{align}
That is,
\[ (\Lambda+ d\Lambda) (A + dA) + (A + dA)(\Lambda + d\Lambda) - \eta L_{\widetilde{P}}(\Lambda + d\Lambda) = \eta (S + dS).\]
Subtracting \eqref{eq:lambda} we have
\begin{align*}
    &d\Lambda A + \Lambda dA + d\Lambda dA + dA\Lambda + Ad\Lambda + dAd\Lambda\\
    =& \eta L_{\widetilde{P}}(\Lambda + d\Lambda) -\eta \mathbbm{E}[\nabla^2 f(\bm{\theta}^*;\bm{x})\Lambda \nabla^2 f(\bm{\theta}^*;\bm{x})] + \eta dS.
\end{align*}
That is, 
\begin{align*}
    &d\Lambda \widetilde{A} + \widetilde{A} d\Lambda - \eta L_{\widetilde{P}}( d\Lambda)\\
    =  &\eta L_{\widetilde{P}}(\Lambda) -\eta \mathbbm{E}[\nabla^2 f(\bm{\theta}^*;\bm{x})\Lambda \nabla^2 f(\bm{\theta}^*;\bm{x})] + \eta dS - \Lambda dA - dA \Lambda\\
    :=& \epsilon.
\end{align*}
Consider the eigen-decomposition of $\widetilde{A} = U D U^T$, we then have 
\begin{align*}
    &\epsilon = d\Lambda U D U^T + U D U^T d\Lambda - \eta \mathbbm{E}[\nabla^2 f(\bm{\theta}^*;\bm{x})  d\Lambda \nabla^2 f(\bm{\theta}^*;\bm{x})]\\
    & \quad\quad + \eta (\mathbbm{E}[\nabla^2 f(\bm{\theta}^*;\bm{x})  d\Lambda \nabla^2 f(\bm{\theta}^*;\bm{x})] - L_{\widetilde{P}}(d\Lambda))\\
    \Longrightarrow \quad   & U^T\epsilon U = U^T d\Lambda U D  + D U^T d\Lambda U - \eta \mathbbm{E}[U^T \nabla^2 f(\bm{\theta}^*;\bm{x})  d\Lambda \nabla^2 f(\bm{\theta}^*;\bm{x}) U] \\ &\quad\quad+ \eta U^T (\mathbbm{E}[\nabla^2 f(\bm{\theta}^*;\bm{x})  d\Lambda \nabla^2 f(\bm{\theta}^*;\bm{x})] - L_{\widetilde{P}}(d\Lambda))U.
\end{align*}
Thus 
\begin{align}
   &\|\epsilon\|_F = \|U^T\epsilon U\|_F \nonumber\\
   \geq& \|U^T d\Lambda U D  + D U^T d\Lambda U\|_F - \eta \|\mathbbm{E}[U^T \nabla^2 f(\bm{\theta}^*;\bm{x})  d\Lambda \nabla^2 f(\bm{\theta}^*;\bm{x}) U]\|_F\nonumber \\ &- \eta \|U^T (\mathbbm{E}[\nabla^2 f(\bm{\theta}^*;\bm{x})  d\Lambda \nabla^2 f(\bm{\theta}^*;\bm{x})] - L_{\widetilde{P}}(d\Lambda))U\|_F\nonumber\\
    \geq & 2d_{\min} \|U^T d\Lambda U\|_F - \eta \|\VEC(\mathbbm{E}[\nabla^2 f(\bm{\theta}^*;\bm{x})  d\Lambda \nabla^2 f(\bm{\theta}^*;\bm{x})])\|_2 \nonumber\\
    & - \eta \|\VEC (\mathbbm{E}[\nabla^2 f(\bm{\theta}^*;\bm{x})  d\Lambda \nabla^2 f(\bm{\theta}^*;\bm{x})] - L_{\widetilde{P}}(d\Lambda))\|_2\nonumber\\
    \geq & 2\delta \| d\Lambda\|_F - \eta \|\mathbbm{E}[\nabla^2 f(\bm{\theta}^*;\bm{x})\otimes \nabla^2 f(\bm{\theta}^*;\bm{x})] \VEC(d\Lambda)\|_2 \nonumber\\
    & - \eta \|(\mathbbm{E}[\nabla^2 f(\bm{\theta}^*;\bm{x})  \otimes \nabla^2 f(\bm{\theta}^*;\bm{x})] - \widetilde{P})\VEC(d\Lambda)\|_2\label{eq:eps_eq1},
\end{align}
where $d_{\min}$ is the smallest diagonal element of the diagonal matrix $D$, i.e., the smallest eigenvalue of $\widetilde{A}$, so we have $d_{\min}\geq \delta$.

We now analyze the terms in \eqref{eq:eps_eq1}. We have that
\begin{align}
    &\|\mathbbm{E}[\nabla^2 f(\bm{\theta}^*;\bm{x})\otimes \nabla^2 f(\bm{\theta}^*;\bm{x})] \VEC(d\Lambda)\|_2 \leq     \|\mathbbm{E}[\nabla^2 f(\bm{\theta}^*;\bm{x})\otimes \nabla^2 f(\bm{\theta}^*;\bm{x})]\|_2 \|\VEC(d\Lambda)\|_2\nonumber\\
    \leq & \mathbbm{E}\|\nabla^2 f(\bm{\theta}^*;\bm{x})\otimes \nabla^2 f(\bm{\theta}^*;\bm{x})\|_2 \|d\Lambda\|_F = \mathbbm{E}\| \nabla^2 f(\bm{\theta}^*;\bm{x})\|_2^2 \|d\Lambda\|_F\leq L_4^2 \|d\Lambda\|_F,\label{eq:eps_eq2}
\end{align}
where the last inequality is by Assumption \ref{ass:5}. And we have
\begin{align}
    \|(\mathbbm{E}[\nabla^2 f(\bm{\theta}^*;\bm{x})  \otimes \nabla^2 f(\bm{\theta}^*;\bm{x})] - \widetilde{P})\VEC(d\Lambda)\|_2 \leq \|P - \widetilde{P}\|_2\|\VEC(d\Lambda)\|_2 \lesssim L_4^2 \|d\Lambda\|_F.
\end{align}
Thus 
\begin{align}
    \|\epsilon\|_F \gtrsim 2\delta \| d\Lambda\|_F - \eta L_4^2  \|d\Lambda\|_F\label{eq:eps_lower}.
\end{align}
On the other hand, 
\[\epsilon = \eta L_{\widetilde{P}}(\Lambda) -\eta \mathbbm{E}[\nabla^2 f(\bm{\theta}^*;\bm{x})\Lambda \nabla^2 f(\bm{\theta}^*;\bm{x})] + \eta dS - \Lambda dA - dA \Lambda.\]
So
\begin{align*}
    \|\epsilon\|_F\leq \eta \|L_{\widetilde{P}}(\Lambda) - \mathbbm{E}[\nabla^2 f(\bm{\theta}^*;\bm{x})\Lambda \nabla^2 f(\bm{\theta}^*;\bm{x})]\|_F + \eta \|dS\|_F + \|\Lambda dA + dA \Lambda\|_F,
\end{align*}
where
\begin{align*}
    &\mathbbm{E}\|L_{\widetilde{P}}(\Lambda)  - \mathbbm{E}[\nabla^2 f(\bm{\theta}^*;\bm{x})\Lambda \nabla^2 f(\bm{\theta}^*;\bm{x})]\|_F\\
    = & \mathbbm{E}\|\VEC(L_{\widetilde{P}}(\Lambda)  - \mathbbm{E}[\nabla^2 f(\bm{\theta}^*;\bm{x})\Lambda \nabla^2 f(\bm{\theta}^*;\bm{x})])\|_2\\
    =& \mathbbm{E}\|(\widetilde{P}- P) \VEC(\Lambda)\|_2\\
    \leq& (\mathbbm{E}\|\widetilde{P}- P\|_2) \| \VEC(\Lambda)\|_2 \stackrel{\eqref{eq:p3}}{\lesssim} \|\Lambda\|_F (\sqrt{L_7} + C_p^2)/\sqrt{t},
\end{align*}
\begin{align*}
    \mathbbm{E}\|dS\|_F \leq \sqrt{p} \mathbbm{E}\|dS\|_2 \lesssim \sqrt{p} C_p^4/\sqrt{t},
\end{align*}
and the last term:
\begin{align*}
     \mathbbm{E}\|\Lambda dA + dA \Lambda\|_F &=  \mathbbm{E}\|\VEC(\Lambda dA) + \VEC(dA \Lambda)\|_2 = \mathbbm{E}\| (I\otimes \Lambda + \Lambda^T \otimes I) \VEC(dA)\|_2\\
     &\leq (\| I\otimes \Lambda\|_2 +\| \Lambda^T \otimes I\|_2) \mathbbm{E}\|\VEC(dA)\|_2\\
     &= 2\|\Lambda\|_2 \mathbbm{E}\|\VEC(dA)\|_2\\
     &= 2\|\Lambda\|_2 \mathbbm{E}\|dA\|_F \leq 2\|\Lambda\|_2 \mathbbm{E}\|dA\|_2 \lesssim \|\Lambda\|_2 C_p^2 /\sqrt{t}.
\end{align*}
Thus 
\begin{align}
    \mathbbm{E}\|\epsilon\|_F&\lesssim [\eta \|\Lambda\|_F(\sqrt{L_7} + C_p^2) + \eta\sqrt{p}C_p^4 + \|\Lambda\|_2 C_p^2]/\sqrt{t}\nonumber\\
    &\lesssim [\eta \|\Lambda\|_F(L_4^2 + C_p^2) + \eta\sqrt{p}C_p^4 + \|\Lambda\|_2 C_p^2]/\sqrt{t}.
    \label{eq:eps_upper}
\end{align}
Combining \eqref{eq:eps_lower} and \eqref{eq:eps_upper}, we have:
\begin{equation}
   (2\delta - \eta L_4^2) \mathbbm{E}\|d\Lambda\|_F\lesssim [\eta \|\Lambda\|_F(L_4^2 + C_p^2) + \eta\sqrt{p}C_p^4 + \|\Lambda\|_2 C_p^2]/\sqrt{t}.
\end{equation}
That is,
\[\mathbbm{E}\|\widetilde{\Lambda} - \Lambda\|_F \lesssim [\eta \|\Lambda\|_F(L_4^2 + C_p^2) + \eta\sqrt{p}C_p^4 + \|\Lambda\|_2 C_p^2]/\sqrt{t}.\]

\textbf{Check \eqref{eq:p1}, \eqref{eq:p2}, \eqref{eq:p3}:}\\
$\bullet$ For $\|\widetilde{P} - P\|_2\lesssim L_4^2$:
\begin{align*}
    \|\widetilde{P} - P\|_2  &\leq  \|\widetilde{P}\|_2 + \|P\|_2 \\
    &\leq L_4^2 + \|\mathbbm{E}[\nabla^2 f(\bm{\theta}^*;\bm{x})  \otimes \nabla^2 f(\bm{\theta}^*;\bm{x})]\|_2\\
    &\leq L_4^2 + \mathbbm{E}\|\nabla^2 f(\bm{\theta}^*;\bm{x})  \otimes \nabla^2 f(\bm{\theta}^*;\bm{x})\|_2\\
    & = L_4^2 + \mathbbm{E}\|\nabla^2 f(\bm{\theta}^*;\bm{x})\|_2^2\\
    &\leq L_4^2 + \sqrt{\mathbbm{E}\|\nabla^2 f(\bm{\theta}^*;\bm{x})\|_2^4}\leq 2L_4^2.
\end{align*}
$\bullet$ For $\mathbbm{E}\|\widehat{P} - P\|_2\lesssim (\sqrt{L_7} + C_p^2)/\sqrt{t}$:
\begin{align}
    &\mathbbm{E}\| \widehat{P} - P\|_2\nonumber\\
    =& \mathbbm{E}\left\|\mathbbm{E}[\nabla^2 f(\bm{\theta}^*;\bm{x})  \otimes \nabla^2 f(\bm{\theta}^*;\bm{x})] - \frac{1}{t - B}\sum_{i=B+1}^{t}\nabla^2 f(\widehat{\bm{\theta}}_{i-1};\bm{x}_i)  \otimes \nabla^2 f(\widehat{\bm{\theta}}_{i-1};\bm{x}_i)\right\|_2  \nonumber\\
    \leq& \mathbbm{E}\left\|\mathbbm{E}[\nabla^2 f(\bm{\theta}^*;\bm{x})  \otimes \nabla^2 f(\bm{\theta}^*;\bm{x})] - \frac{1}{t - B}\sum_{i=B+1}^{t}\nabla^2 f(\bm{\theta}^*;\bm{x}_i)  \otimes \nabla^2 f(\bm{\theta}^*;\bm{x}_i)\right\|_2\nonumber\\
    &+ \mathbbm{E}\Bigg\|\frac{1}{t - B}\sum_{i=B+1}^{t}\nabla^2 f(\bm{\theta}^*;\bm{x}_i)  \otimes \nabla^2 f(\bm{\theta}^*;\bm{x}_i)\nonumber\\
    &\quad\quad\quad- \frac{1}{t - B}\sum_{i=B+1}^{t}\nabla^2 f(\widehat{\bm{\theta}}_{i-1};\bm{x}_i)  \otimes \nabla^2 f(\widehat{\bm{\theta}}_{i-1};\bm{x}_i)\Bigg\|_2, \label{eq:eps_eq3}
\end{align}
for which
\begin{align}
&\mathbbm{E}\left\|\mathbbm{E}[\nabla^2 f(\bm{\theta}^*;\bm{x})  \otimes \nabla^2 f(\bm{\theta}^*;\bm{x})] - \frac{1}{t - B}\sum_{i=B+1}^{t}\nabla^2 f(\bm{\theta}^*;\bm{x}_i)  \otimes \nabla^2 f(\bm{\theta}^*;\bm{x}_i)\right\|_2\nonumber\\
\leq &\sqrt{\mathbbm{E}\left\|\mathbbm{E}[\nabla^2 f(\bm{\theta}^*;\bm{x})  \otimes \nabla^2 f(\bm{\theta}^*;\bm{x})] - \frac{1}{t - B}\sum_{i=B+1}^{t}\nabla^2 f(\bm{\theta}^*;\bm{x}_i)  \otimes \nabla^2 f(\bm{\theta}^*;\bm{x}_i)\right\|_2^2}\nonumber\\
\leq& \frac{\sqrt{L_7}}{\sqrt{t-B}} \label{eq:eps_eq4},
\end{align}
where the bound \eqref{eq:eps_eq4} comes from Assumption \ref{ass:6}. And 
\begin{align}
    &\mathbbm{E}\Bigg\|\frac{1}{t - B}\sum_{i=B+1}^{t}\nabla^2 f(\bm{\theta}^*;\bm{x}_i)  \otimes \nabla^2 f(\bm{\theta}^*;\bm{x}_i)\nonumber\\
    &\quad\quad\quad- \frac{1}{t - B}\sum_{i=B+1}^{t}\nabla^2 f(\widehat{\bm{\theta}}_{i-1};\bm{x}_i)  \otimes \nabla^2 f(\widehat{\bm{\theta}}_{i-1};\bm{x}_i)\Bigg\|_2\nonumber\\
    \leq&\frac{1}{t - B}\sum_{i=B+1}^{t}\mathbbm{E}[\|\nabla^2 f(\bm{\theta}^*;\bm{x}_i)  \otimes \nabla^2 f(\bm{\theta}^*;\bm{x}_i) - \nabla^2 f(\widehat{\bm{\theta}}_{i-1};\bm{x}_i)  \otimes \nabla^2 f(\widehat{\bm{\theta}}_{i-1};\bm{x}_i)\|_2]\nonumber\\
    =&\frac{L_6}{t - B}\sum_{i=B+1}^{t} \mathbbm{E}[\|\bm{\theta}^* - \widehat{\bm{\theta}}_{i-1}\|_2] + \frac{L_6^\prime}{t - B}\sum_{i=B+1}^{t} \mathbbm{E}[\|\bm{\theta}^* - \widehat{\bm{\theta}}_{i-1}\|_2^2]\lesssim C_p^2/\sqrt{t}. \label{eq:eps_eq5}
\end{align}
Thus, $\mathbbm{E}\|\widehat{P} - P\|_2\lesssim (\sqrt{L_7} + C_p^2)/\sqrt{t}$.
\newline
\newline
$\bullet$ For $\mathbbm{E}\|\widetilde{P} - P\|_2 \lesssim (\sqrt{L_7} + C_p^2)/\sqrt{t}$:\\
First, we have
\begin{align*}
    \|\widehat{P} - \widetilde{P}\|_2 &= \mathbbm{1}_{\widehat{P} \preceq L_4^2 I} * 0 + (1 - \mathbbm{1}_{\widehat{P} \preceq L_4^2 I}) * (\|\widehat{P}\|_2 - L_4^2)\\
    &\leq \mathbbm{1}_{\widehat{P} \preceq L_4^2 I} * 0 + (1 - \mathbbm{1}_{\widehat{P} \preceq L_4^2 I}) * (\|\widehat{P} - {P}\|_2 + \|{P}\|_2- L_4^2)\\
    &\leq (1 - \mathbbm{1}_{\widehat{P} \preceq L_4^2 I}) * (\|\widehat{P} - {P}\|_2 + L_4^2- L_4^2) \leq \|\widehat{P} - {P}\|_2.
\end{align*}
Thus
\begin{align*}
    \mathbbm{E}\|\widetilde{P} - P\|_2 &\leq  \mathbbm{E}\|\widetilde{P} - \widehat{P}\|_2 +  \mathbbm{E}\|\widehat{P} - P\|_2\\
    &\leq 2 \mathbbm{E}\|\widehat{P} - P\|_2 \lesssim (\sqrt{L_7} + C_p^2)/\sqrt{t}.
\end{align*}

\end{proof}

\paragraph{With Lemma \ref{lem:thm1_terms_err}, we now prove Theorem \ref{thm:consis}.}
\begin{proof}
We have
 \[ \widetilde{\Sigma} = \widetilde{A}^{-1} (\frac{1}{\eta}\widetilde{\Lambda} \widetilde{A} + \frac{1}{\eta}\widetilde{A}\widetilde{\Lambda}-  \widetilde{A} \widetilde{\Lambda}\widetilde{A}) \widetilde{A}^{-1} = \frac{1}{\eta}\widetilde{A}^{-1} \widetilde{\Lambda} + \frac{1}{\eta}\widetilde{\Lambda}\widetilde{A}^{-1} - \widetilde{\Lambda} ,\]
 and  
 \[ {\Sigma} = \frac{1}{\eta}{A}^{-1} {\Lambda} + \frac{1}{\eta}{\Lambda}{A}^{-1} - {\Lambda} .\]
 Since 
 \begin{align*}
     &\quad \mathbbm{E}\|\widetilde{A}^{-1} \widetilde{\Lambda} - A^{-1}\Lambda\|_F \\
     &= \mathbbm{E}\|(\widetilde{A}^{-1} -A^{-1} + A^{-1})(\widetilde{\Lambda} - {\Lambda} + {\Lambda}) - A^{-1}\Lambda\|_F\\
     &\leq \mathbbm{E}\|(\widetilde{A}^{-1} -A^{-1})(\widetilde{\Lambda} - {\Lambda})\|_F + \mathbbm{E}\|(\widetilde{A}^{-1} -A^{-1})\Lambda\|_F + \mathbbm{E}\|A^{-1}(\widetilde{\Lambda} - {\Lambda})\|_F\\
     &\leq \|\widetilde{A}^{-1} -A^{-1}\|_2 \mathbbm{E}\|\widetilde{\Lambda} - {\Lambda}\|_F + \|\Lambda\|_F \mathbbm{E}\|\widetilde{A}^{-1} -A^{-1}\|_2 + \|A^{-1}\|_2 \mathbbm{E}\|\widetilde{\Lambda} - {\Lambda}\|_F\\
     &\leq (\delta^{-1} + \lambda_{\min}(A))\mathbbm{E}\|\widetilde{\Lambda} - {\Lambda}\|_F + \|\Lambda\|_F \mathbbm{E}\|\widetilde{A}^{-1} -A^{-1}\|_2 + \delta^{-1} \mathbbm{E}\|\widetilde{\Lambda} - {\Lambda}\|_F\\
     &\lesssim [\eta \|\Lambda\|_F(L_4^2 + C_p^2) + \eta\sqrt{p}C_p^4 + \|\Lambda\|_2 C_p^2]/\sqrt{t} + \|\Lambda\|_F C_p^2 /\sqrt{t},
\end{align*}
and similarly
\[\mathbbm{E}\|\widetilde{\Lambda}\widetilde{A}^{-1} - \Lambda A^{-1}\|_F \lesssim [\eta \|\Lambda\|_F(L_4^2 + C_p^2) + \eta\sqrt{p}C_p^4 + \|\Lambda\|_2 C_p^2]/\sqrt{t} + \|\Lambda\|_F C_p^2 /\sqrt{t},\]
 we have that
 \begin{align*}
   \mathbbm{E}\|\Sigma - \widetilde{\Sigma}\|_F &\leq \frac{1}{\eta}\|\widetilde{A}^{-1} \widetilde{\Lambda} - A^{-1}\Lambda\|_F + \frac{1}{\eta}\|\widetilde{\Lambda}\widetilde{A}^{-1} - \Lambda{A}^{-1}\|_F + \mathbbm{E}\|\widetilde{\Lambda} - {\Lambda}\|_F\\
    &\lesssim (1+ \frac{1}{\eta})
    [\eta \|\Lambda\|_F(L_4^2 + C_p^2) + \eta\sqrt{p}C_p^4 + \|\Lambda\|_2 C_p^2]/\sqrt{t} + \frac{1}{\eta}\|\Lambda\|_F C_p^2 /\sqrt{t}\\
    &\lesssim [\|\Lambda\|_F(L_4^2 + C_p^2) + \sqrt{p}C_p^4 ]/\sqrt{t} + \frac{1}{\eta}\|\Lambda\|_F C_p^2 /\sqrt{t}.
 \end{align*}
Recall that $\Lambda$ is by solving following equation:
\[
    \Lambda A + A\Lambda - \eta \mathbbm{E}[\nabla^2 f(\bm{\theta}^*;\bm{x})\Lambda \nabla^2 f(\bm{\theta}^*;\bm{x})] = \eta S .
\]
Thus, 
\begin{align*}
    \eta\|S\|_F &=  \|\Lambda A + A\Lambda - \eta \mathbbm{E}[\nabla^2 f(\bm{\theta}^*;\bm{x})\Lambda \nabla^2 f(\bm{\theta}^*;\bm{x})]\|_F\\
    &\geq \|\Lambda A + A\Lambda\|_F - \eta \|\mathbbm{E}[\nabla^2 f(\bm{\theta}^*;\bm{x})\Lambda \nabla^2 f(\bm{\theta}^*;\bm{x})]\|_F\\
    &\geq 2\lambda_{\min}(A) \|\Lambda\|_F - \eta \|\mathbbm{E}[\nabla^2 f(\bm{\theta}^*;\bm{x})\otimes \nabla^2 f(\bm{\theta}^*;\bm{x})]\|_2 \|\VEC(\Lambda)\|_2\\
    &\geq 2\lambda_{\min}(A) \|\Lambda\|_F - \eta \mathbbm{E}\|\nabla^2 f(\bm{\theta}^*;\bm{x}) \nabla^2 f(\bm{\theta}^*;\bm{x})\|_2^2 \|\VEC(\Lambda)\|_2\\
    &\geq (2\lambda_{\min}(A)  - \eta L_{4}^2) \|\Lambda\|_F.
\end{align*}
That is, 
\[\|\Lambda\|_F \lesssim \eta \|S\|_F.\]
To sum up, we have
 \begin{align*}
   \mathbbm{E}\|\Sigma - \widetilde{\Sigma}\|_F
    &\lesssim [\eta\|S\|_F(L_4^2 + C_p^2) + \sqrt{p}C_p^4 + \|S\|_F C_p^2] /\sqrt{t}\\
    &\lesssim [\sqrt{p}C_p^4 + \|S\|_F C_p^2] /\sqrt{t}.
 \end{align*}
\end{proof}

\section{A Preliminary Result for Proving Theorem \ref{thm:3}}
In this section, we prove a preliminary result for proving Theorem \ref{thm:3}. 
The result is given in Proposition \ref{lem:asymptotic_root_func}. 
It generalizes Lemma \ref{lem:asymptotic_root} to a functional form for the random function $C_t(r) := r\sqrt{t}\bm{w}^T (\widehat{\bm{\theta}}_{[rt]} - \bm{\theta}^*)$, $r\in [0,1]$. 

The preliminary result uses the invariance principle of martingale CLT. 
So we first review the invariance principle of martingale CLT in the following lemma:
\begin{lemma}[Theorem 4.4. in \citep{hall2014martingale}, rewritten in conditional Lindeberg condition]\label{lem:inv_clt}
For $\{S_t,\mathcal{F}_t\}$ a zero mean martingale, define $X_t := S_t - S_{t-1}$, $V_t^2 := \sum_{i=1}^t \mathbbm{E}(X_i^2 | \mathcal{F}_{i-1})$, $s_t^2 := \mathbbm{E}(S_t^2) = \sum_{i=1}^t \mathbbm{E}(X_i^2) $. 
Assume the conditional Lindeberg condition holds, that is
\[\forall \epsilon > 0, s_t^{-2}\sum_{i=1}^t \mathbbm{E}\left[X_i^2 I(|X_i|> \epsilon s_t)| \mathcal{F}_{i-1}\right] \stackrel{t\to\infty}{\to} 0 .\]
And assume that 
\[s_t^{-2} V_t^2 \stackrel{p}{\to} \eta^2 > 0\quad a.s..\]
Then for the random function defined on $r\in[0,1]$: $\zeta_t(r) := V_t^{-1} \{S_i + (V_{i+1}^2 - V_{i}^2)^{-1} (r V_t^2 - V_{i}^2) X_{i+1}\}$ where $i$ is such that $V_i^2 \leq rV_t^2 < V_{i+1}^2$, one have
\[\zeta_t(r) \stackrel{d}{\to} W_r.\]
\end{lemma}

With Lemma \ref{lem:inv_clt}, we can develop the functional CLT extension of Lemma \ref{lem:asymptotic_root} in the following proposition: 
\begin{proposition}[Functional CLT extension of Lemma \ref{lem:asymptotic_root}]\label{lem:asymptotic_root_func}
Under assumptions of Lemma \ref{lem:asymptotic_root}, we have that  
 \begin{equation}
r\sqrt{t}\bm{w}^T (\widehat{\bm{\theta}}_{[rt]} - \bm{\theta}^*)\stackrel{d}{\to} (\bm{w}^T \Sigma \bm{w})^{1/2} W_r,
 \end{equation}

\end{proposition}

\begin{proof}
We make use of the martingale decomposition of ROOT-SGD updates. 
By equation (31) in \citep{li2022root}, the ROOT-SGD has the difference $\bm{z}_i := \bm{v}_i - \nabla F(\widehat{\bm{\theta}}_{i-1})$ decomposed as follows:
\begin{align*}
    \bm{z}_i = \frac{1}{i}\underbrace{\sum_{s = B}^i {\delta}_s (\widehat{\bm{\theta}}_{s-1})}_{:= M_i} + \frac{B}{i} \bm{z}_B + \frac{1}{i} \underbrace{\sum_{s = B}^i(s - 1)( \delta_s (\widehat{\bm{\theta}}_{s-1}) -  \delta_s (\widehat{\bm{\theta}}_{s-2}))}_{:= \Psi_i}, 
\end{align*}
where ${\delta}_s (\bm{\theta}) := \delta(\bm{\theta};\bm{x}_s) = \nabla f(\bm{\theta};\bm{x}_s) - \nabla F(\bm{\theta})$. 
Then $M_i$ and $\Psi_i$ are two martingale sequences. 
However, it is hard to directly analyze these two sequences.
To understand the behavior of $M_i$ and $\Psi_i$, \citep{li2022root} defines two auxiliary processes in their Appendix D.1.:
\[N_i := \sum_{s = B}^i {\delta}_s (\bm{\theta}^*), \quad \Upsilon_i := \eta\sum_{s= B}^i\Xi_s(\bm{\theta}^*)\bm{y}_{s-1},\]
where $\Xi_s(\bm{\theta}) := \nabla^2 f(\bm{\theta}; \bm{x}_s) - \nabla^2 F(\bm{\theta}) $, and $\bm{y}_{s}$ is a zero-mean Markov process defined as: $\bm{y}_{B-1} = \bm{0}$, $\bm{y}_{s} = \bm{y}_{s-1} -  \eta \nabla^2 f(\bm{\theta}^*; \bm{x}_s) \bm{y}_{s-1}+ \delta_s(\bm{\theta}^*)$ for $s \geq B$. 

For the auxiliary processes $N_i$ and $\Upsilon_i$, by bound (81) and (82) in \citep{li2022root}, we have $\|N_{t}  - M_{t}\|_2/\sqrt{t}\stackrel{p}{\to} 0$ and $\|\Upsilon_{t} - \Psi_{t}\|_2/\sqrt{t} \stackrel{p}{\to} 0$. Then for a $r\in (0,1]$, we have $\|N_{[rt]}  - M_{[rt]}\|_2/\sqrt{rt}\stackrel{p}{\to} 0$ and $\|\Upsilon_{[rt]} - \Psi_{[rt]}\|_2/\sqrt{rt} \stackrel{p}{\to} 0$. Thus, we must have
\begin{align}
  &sup_{r\in [0,1]}\|N_{[rt]}  - M_{[rt]}\|_2/\sqrt{t}\stackrel{p}{\to} 0, \label{eq:appc_eq1}\\
  &sup_{r\in [0,1]}\|\Upsilon_{[rt]} - \Psi_{[rt]}\|_2/\sqrt{t} \stackrel{p}{\to} 0\label{eq:appc_eq2}.
\end{align}
We claim the weak convergence of $N_{[rt]} + \Upsilon_{[rt]}$ is as follows, which we check later:
\begin{equation}
     (N_{[rt]} + \Upsilon_{[rt]})/\sqrt{t} \stackrel{d}{\to} (S + \mathbbm{E}(\Xi_{\bm{x}}(\bm{\theta}^*) \Lambda\Xi_{\bm{x}}(\bm{\theta}^*)))^{1/2} \bm{W}_r,\label{eq:appc_eq3}
\end{equation}
where $\bm{W}_r$ is a $p$ - dimensional random variable with each coordinate independently distributed as $W_r$ the standard wiener process. 

By \eqref{eq:appc_eq1} to \eqref{eq:appc_eq3}, we have 
\begin{align*}
    &\frac{[rt]}{\sqrt{t}}\bm{z}_{[rt]}\\
    =& \frac{1}{\sqrt{t}}M_{[rt]} + \frac{B}{\sqrt{t}} \bm{z}_B + \frac{1}{\sqrt{t}} \Psi_{[rt]}\\
    \stackrel{d}{\to} & \frac{1}{\sqrt{t}} (N_{[rt]} + \Upsilon_{[rt]})\\
    \stackrel{d}{\to} & (S + \mathbbm{E}(\Xi_{\bm{x}}(\bm{\theta}^*) \Lambda\Xi_{\bm{x}}(\bm{\theta}^*)))^{1/2} \bm{W}_r. 
\end{align*}
By \citep{li2022root}, $\sqrt{t}\bm{v}_t\stackrel{p}{\to} \bm{0}$ and $\sqrt{t}\|\nabla F(\widehat{\bm{\theta}}_{t}) - A(\widehat{\bm{\theta}}_{t} - \bm{\theta}^*)\|_2 \stackrel{p}{\to} 0$ (on pages 56 and 57), we have $sup_{r\in[0,1]}r\sqrt{t}\bm{v}_t\stackrel{p}{\to} \bm{0}$ and 
\\
$\sup_{r\in[0,1]}r\sqrt{t}\|\nabla F(\widehat{\bm{\theta}}_{t}) - A(\widehat{\bm{\theta}}_{t} - \bm{\theta}^*)\|_2 \stackrel{p}{\to} 0$.
So 
\[
\frac{[rt]}{\sqrt{t}}\nabla F(\widehat{\bm{\theta}}_{[rt] - 1}) = \frac{[rt]}{\sqrt{t}} \bm{v}_{[rt]} - \frac{[rt]}{\sqrt{t}}\bm{z}_{[rt]} \stackrel{d}{\to} (S + \mathbbm{E}(\Xi_{\bm{x}}(\bm{\theta}^*) \Lambda\Xi_{\bm{x}}(\bm{\theta}^*)))^{1/2} \bm{W}_r;
\]
\[
\frac{[rt]}{\sqrt{t}} A (\widehat{\bm{\theta}}_{[rt] - 1} - \bm{\theta}^*) \stackrel{d}{\to} (S + \mathbbm{E}(\Xi_{\bm{x}}(\bm{\theta}^*) \Lambda\Xi_{\bm{x}}(\bm{\theta}^*)))^{1/2} \bm{W}_r;
\]
\[
\frac{[rt]}{\sqrt{t}} (\widehat{\bm{\theta}}_{[rt]} - \bm{\theta}^*) \stackrel{d}{\to} (A^{-1}(S + \mathbbm{E}(\Xi_{\bm{x}}(\bm{\theta}^*) \Lambda\Xi_{\bm{x}}(\bm{\theta}^*)))A^{-1})^{1/2} \bm{W}_r.
\]
That is,
\begin{align*}
r\sqrt{t}\bm{w}^T (\widehat{\bm{\theta}}_{[rt]} - \bm{\theta}^*) &\stackrel{d}{\to} (\bm{w}^T A^{-1}(S + \mathbbm{E}(\Xi_{\bm{x}}(\bm{\theta}^*) \Lambda\Xi_{\bm{x}}(\bm{\theta}^*)))A^{-1}\bm{w})^{1/2} {W}_r\\
&= (\bm{w}^T \Sigma \bm{w})^{1/2} W_r. 
\end{align*}
As the proposition claim. 

It remains to prove that $(N_{[rt]} + \Upsilon_{[rt]})/\sqrt{t} \stackrel{d}{\to} (S + \mathbbm{E}(\Xi_{\bm{x}}(\bm{\theta}^*) \Lambda\Xi_{\bm{x}}(\bm{\theta}^*)))^{1/2} \bm{W}_r$. 
We use Lemma \ref{lem:inv_clt} to prove such statement. 
To apply Lemma \ref{lem:inv_clt}, we define the quantity $X_i$ as follows: 
for a fixed $\bm{w} \neq \bm{0}$, let $X_i = \bm{w}^T \bm{\nu}_i$ where $\bm{\nu}_i = \delta_i (\bm{\theta}^*) + \eta \Xi_i(\bm{\theta}^*)\bm{y}_{i-1}$. 
Then $\bm{w}^T(N_{[rt]} + \Upsilon_{[rt]}) = \sum_{i = B}^{[rt]} X_i$. 

We now prove the limiting distribution of $(N_{[rt]} + \Upsilon_{[rt]})/\sqrt{t}$ in following steps. 

We first check $s_t^{-2} V_t^2 \stackrel{p}{\to} \eta^2 > 0$ $a.s.$. 
By \citep{li2022root} (page 55), $\frac{1}{t}\sum_{i=B}^t \mathbbm{E}(\bm{\nu}_i \bm{\nu}_i^T | \mathcal{F}_{i-1}) \stackrel{p}{\to} S + \mathbbm{E}(\Xi_{\bm{x}}(\bm{\theta}^*) \Lambda\Xi_{\bm{x}}(\bm{\theta}^*)) $, then 
\begin{align*}
&\frac{1}{t} V_t^2 = \frac{1}{t} \sum_{i=1}^t \mathbbm{E}(X_i^2 | \mathcal{F}_{i-1}) = \bm{w}^T \left[\frac{1}{t}\sum_{i=B}^t \mathbbm{E}(\bm{\nu}_i \bm{\nu}_i^T | \mathcal{F}_{i-1})\right] \bm{w}\\
& \qquad \stackrel{d}{\to} \bm{w}^T(S + \mathbbm{E}(\Xi_{\bm{x}}(\bm{\theta}^*) \Lambda\Xi_{\bm{x}}(\bm{\theta}^*))) \bm{w},\\
& \frac{1}{t} s_t^2 = \mathbbm{E}\left[\frac{1}{t} V_t^2 \right] \to \bm{w}^T(S + \mathbbm{E}(\Xi_{\bm{x}}(\bm{\theta}^*) \Lambda\Xi_{\bm{x}}(\bm{\theta}^*))) \bm{w}.
\end{align*}
Since $C := \bm{w}^T(S + \mathbbm{E}(\Xi_{\bm{x}}(\bm{\theta}^*) \Lambda\Xi_{\bm{x}}(\bm{\theta}^*))) \bm{w} > 0$ is a constant, we also have $\frac{1}{t} V_t^2$ converges in probability implied by comvergence in distribution
\[
\frac{1}{t} V_t^2 \stackrel{p}{\to} \bm{w}^T(S + \mathbbm{E}(\Xi_{\bm{x}}(\bm{\theta}^*) \Lambda\Xi_{\bm{x}}(\bm{\theta}^*))) \bm{w}.
\]
Thus 
\[s_t^{-2} V_t^2 \stackrel{p}{\to} 1 .\]

We then check the conditional Lindeberg condition.
By \citep{li2022root} (page 56), one have that as $t\to\infty$, 
\begin{align}\label{eq:lindeberg_rootsgd}
    \forall \epsilon' > 0,  R_t(\epsilon') := t^{-1}\sum_{i=1}^t \mathbbm{E}\left[\|\nu_i\|_2^2 I(\|\nu_i\|_2> \epsilon' \sqrt{t})| \mathcal{F}_{i-1}\right] \stackrel{p}{\to} 0.
\end{align}
To check the Lindeberg condition for $X_i$ sequence for any $\epsilon > 0$, we note that 
\begin{align*}
 s_t^{-2}\sum_{i=1}^t \mathbbm{E}\left[X_i^2 I(|X_i|> \epsilon s_t)| \mathcal{F}_{i-1}\right] \to \frac{1}{tC} \sum_{i=1}^t \mathbbm{E}\left[X_i^2 I(|X_i|> \epsilon \sqrt{Ct})| \mathcal{F}_{i-1}\right],
\end{align*}
where 
\begin{align*}
    &\frac{1}{t} \sum_{i=1}^t \mathbbm{E}\left[X_i^2 I(|X_i|> \epsilon \sqrt{Ct})| \mathcal{F}_{i-1}\right]\\
\leq &\frac{1}{t} \sum_{i=1}^t \mathbbm{E}\left[\|\bm{\nu}_i\|_2^2 \|\bm{w}\|_2^2 I(\|\bm{\nu}_i\|_2> \epsilon \sqrt{Ct}/\|\bm{w}\|_2)| \mathcal{F}_{i-1}\right]\\
\to & 0 \left(\text{By taking }\epsilon' = \frac{\epsilon C}{\|\bm{w}\|_2} \text{ in \eqref{eq:lindeberg_rootsgd}}\right).
\end{align*}
Thus the Lindeberg condition holds. 

We now have all the conditions in Lemma \ref{lem:inv_clt} hold, thus we have $\zeta_t(r)$ satisfies that
\[\zeta_t(r) \stackrel{d}{\to} W_r,\]
where $ \zeta_t(r)= V_t^{-1} \{S_i + (V_{i+1}^2 - V_{i}^2)^{-1} (r V_t^2 - V_{i}^2) X_{i+1}\}$ and $i$ is such that $V_i^2 \leq rV_t^2 < V_{i+1}^2$. 

On the other hand, ${V_{[rt]}^2 }/{V_{t}^2} \to r$. Thus $\zeta_t\left({V_{[rt]}^2 }/{V_{t}^2}\right) \to \zeta_t(r)$. That is, $V_t^{-1} S_{[rt]} = \zeta_t\left({V_{[rt]}^2 }/{V_{t}^2}\right) \stackrel{d}{\to} W_r$, then by Slutsky's theorem $S_{[rt]}/\sqrt{t} \stackrel{d}{\to} \sqrt{C} W_r$.
In this way,
\[\bm{w}^T(N_{[rt]} + \Upsilon_{[rt]})/\sqrt{t} \to S_{[rt]}/\sqrt{t} \stackrel{d}{\to} \sqrt{\bm{w}^T(S + \mathbbm{E}(\Xi_{\bm{x}}(\bm{\theta}^*) \Lambda\Xi_{\bm{x}}(\bm{\theta}^*))) \bm{w} } W_r .\]
Recall that the above holds for any $\bm{w}\neq \bm{0}$, so we have
\[(N_{[rt]} + \Upsilon_{[rt]})/\sqrt{t} \stackrel{d}{\to} (S + \mathbbm{E}(\Xi_{\bm{x}}(\bm{\theta}^*) \Lambda\Xi_{\bm{x}}(\bm{\theta}^*)))^{1/2} \bm{W}_r,\]
which is the limiting distribution as claimed. 
\end{proof}

\section{Proof for Theorem \ref{thm:3}}\label{app:C}
In this section, we prove Theorem \ref{thm:3} using Proposition \ref{lem:asymptotic_root_func}. 
\begin{proof}
% The proof 
% is done by analyzing the weak convergence of the terms $\bm{w}^T (\widehat{\bm{\theta}}_t - \bm{\theta}^*)$ and $\bm{w}^T V_t \bm{w}$, then the weak convergence of their division holds by continuous mapping theorem. The asymptotic distribution of $\bm{w}^T (\widehat{\bm{\theta}}_t - \bm{\theta}^*)$ directly follows from Lemma \ref{lem:asymptotic_root}. By \eqref{eq:asymptotic_root}, we immediately have\begin{equation}\label{eq:thm3eq1}\sqrt{t}\bm{w}^T (\widehat{\bm{\theta}}_t - \bm{\theta}^*)\stackrel{d}{\to}N(0, \bm{w}^T \Sigma \bm{w}) \quad( = (\bm{w}^T \Sigma \bm{w})^{1/2} W_1).     \end{equation}The  asymptotic distribution of $\bm{w}^T V_t \bm{w}$ 
% requires a generalization of Lemma \ref{lem:asymptotic_root} to a functional form for the random function $C_t(r) := r\sqrt{t}\bm{w}^T (\widehat{\bm{\theta}}_{[rt]} - \bm{\theta}^*)$, $r\in [0,1]$. 
 By Proposition \ref{lem:asymptotic_root_func}, we have that the random function $C_t(r) = r\sqrt{t}\bm{w}^T (\widehat{\bm{\theta}}_{[rt]} - \bm{\theta}^*)$ satisfies
 \begin{equation}\label{eq:thm3eq2}
C_t(r) \stackrel{d}{\to} (\bm{w}^T \Sigma \bm{w})^{1/2} W_r.
 \end{equation}
 Our statistic is 
 \begin{align*}
     &\frac{\sqrt{t}\bm{w}^T (\widehat{\bm{\theta}}_t - \bm{\theta}^*)}{\sqrt{\bm{w}^T V_t \bm{w}}}\\
     =&\frac{C_t(1)}{\sqrt{\frac{1}{t}\sum_{i=1}^t \frac{i^2}{t} (\bm{w}^T (\widehat{\bm{\theta}}_{i} - \widehat{\bm{\theta}}_t))^2}}\\
      \stackrel{t\to\infty}{\to}&\frac{C_t(1)}{\sqrt{\int_{ 0}^1 r^2 t (\bm{w}^T(\widehat{\bm{\theta}}_{[rt]} - \widehat{\bm{\theta}}_t))^2 dr}}\\
      =& \frac{C_t(1)}{\sqrt{\int_{ 0}^1 (C_t(r) - rC_t(1))^2 dr}}.
 \end{align*}
 Now 
 $\frac{C_t(1)}{\sqrt{\int_{0}^1 (C_t(r) - r C_t(1))^2 dr}}$ is a continuous function of $C_t(\cdot)$, so by continuous mapping theorem (Theorem 18.11 of \citep{van2000asymptotic}) we have that
 %Combining \eqref{eq:thm3eq1} and \eqref{eq:thm3eq2}, we then have
 %\begin{equation}
 %r\sqrt{t}\bm{w}^T (\widehat{\bm{\theta}}_{[rt]} - \widehat{\bm{\theta}}_t)\stackrel{d}{\to} (\bm{w}^T \Sigma \bm{w})^{1/2} (W_r - r W_1),\label{eq:thm3eq3}
 %\end{equation}
 %thus 
 \begin{equation*}
        \frac{C_t(1)}{\sqrt{\int_{ 0}^1 (C_t(r) - rC_t(1))^2 dr}} \stackrel{d}{\to} \frac{(\bm{w}^T \Sigma \bm{w})^{1/2} W_1}{(\bm{w}^T \Sigma \bm{w})^{1/2}\sqrt{\int_{0}^1 (W_r - rW_1)^2 dr}} = \frac{W_1}{\sqrt{\int_{0}^1 (W_r - rW_1)^2 dr}}.
 \end{equation*}
 This gives us the theorem claim. 
\end{proof}

\section{Comparison of Our Covariance Estimators with Those for the ASGD Algorithm}
In this section, we compare our plug-in covariance estimator and the random-scaling estimator with their counterparts for the ASDG algorithm. 

\subsection{Plug-in Estimator Comparison}\label{app:plug_in_compare}
We compare our plug-in estimator with the plug-in estimator for ASGD in \citep{chen2020statistical} from two aspects: the convergence rate and the computational burden. 

For the convergence rate comparison, our plug-in estimator converges faster to the true asymptotic covariance of ROOT-SGD compared to the ASGD counterpart. 
Recall that for ASGD, the plug-in covariance estimator $\widehat{\Sigma}_{t,ASGD}$ converges to the true asymptotic covariance ${\Sigma}_{ASGD}$ as $\mathbbm{E}\|\Sigma_{ASGD} - \widehat{\Sigma}_{t,ASGD}\|_2\lesssim t^{-\alpha/2}$. 
Since $\alpha\in (1/2,1)$, the convergence rate is strictly slower than $\mathcal{O}(t^{-1/2})$. 
Our plug-in covariance estimator $\widehat{\Sigma}_{t}$ converges to the true asymptotic covariance ${\Sigma}$ of ROOT-SGD as $\mathbbm{E}\|\Sigma - \widehat{\Sigma}_{t}\|_2\lesssim t^{-1/2}$.
So the convergence speed of our estimator is strictly faster than that of the plug-in estimator in \citep{chen2020statistical}.
Moreover, this $\mathcal{O}(t^{-1/2})$ rate matches the optimal statistical rate in such a random sampling scheme. 

For the computational aspect of the plug-in estimator, both our estimator and that in \citep{chen2020statistical} can be computed in a fully online fashion. 
Suppose the algorithm is updated for $t$ steps, and the parameter is of dimension $p$. 
Our plug-in estimator takes $\mathcal{O}(p^6 + t p^4)$ total arithmetic computation, while their estimator takes $\mathcal{O}(p^3 + t p^2)$ total arithmetic computation. 
Our plug-in estimator takes more computation. 
This is due to the intrinsic structure of the asymptotic covariance of the ROOT-SGD estimator: 
the asymptotic covariance depends on the quantity \begin{equation}\label{eq:quantity_lambda}
    \mathbbm{E}(\nabla^2 f(\bm{\theta}^*;\bm{x})\Lambda \nabla^2 f(\bm{\theta}^*;\bm{x})) ;
\end{equation}
moreover, one needs to solve an equation that contains \eqref{eq:quantity_lambda} in $\Lambda$. 
Thus, in our plug-in estimator, we need to evaluate the empirical counter-part of \eqref{eq:quantity_lambda} for an unknown $\Lambda$, which is $ \widehat{P}(\Lambda) = \frac{1}{t}\sum_{i= 1}^{t}\nabla^2 f(\widehat{\bm{\theta}}_{i-1};\bm{x}_i)\Lambda \nabla^2 f(\widehat{\bm{\theta}}_{i-1};\bm{x}_i)$. 
To achieve this, we keep an online update of $\widehat{P} = \frac{1}{t}\sum_{i= 1}^{t}\nabla^2 f(\widehat{\bm{\theta}}_{i-1};\bm{x}_i)\otimes \nabla^2 f(\widehat{\bm{\theta}}_{i-1};\bm{x}_i)$. 
We further invert $\widehat{P}$ when solving the empirical counterpart of $\Lambda$. 
Computing $\widehat{P} \in \mathcal{R}^{p^2\times p^2}$ and its inverse then needs $\mathcal{O}(t p^4 + p^6)$ operations, which dominates the computation. 
As a comparison, the asymptotic covariance of the ASGD estimator does not contain any $\Lambda$ term; thus, its plug-in estimator is computed with fewer computations. 

The computation burden of our plug-in estimator for ROOT-SGD is the same as that for ASGD for the special case of $\nabla^2 f(\bm{\theta}^*;\bm{x}) \equiv \nabla^2 F(\bm{\theta}^*)$. 
Such a special case holds for the exponential family model. 
In this case, the asymptotic covariance of ROOT-SGD reduces to $A^{-1}S A^{-1}$, and the plug-in covariance of ROOT-SGD is then $\widehat{A}^{-1}\widehat{S} \widehat{A}^{-1}$. 
Then the computational complexity of our plug-in estimator becomes $\mathcal{O}(p^3 + t p^2)$. 

\subsection{Hessian-free Estimator Comparison}\label{app:hessian_free_compare}

We compare the computation of our random-scaling estimator with the random-scaling estimator for ASGD in \citep{lee2021fast}. 

Both estimators can be computed fully online. 
For $t$ steps of the algorithm on a problem of dimension $p$, our random-scaling estimator takes $\mathcal{O}(tp^2)$ arithmetic computations, which is the same as the random-scaling estimator for ASGD and is less than our plug-in estimator. 

As for the asymptotic convergence of the random-scaling estimators, both our estimator and that in \citep{lee2021fast} are asymptotically consistent. 
Unfortunately, there is no convergence rate result in either our work or \citep{lee2021fast}. 
Though the convergence rates of the random-scaling estimators are not guaranteed, considering that ROOT-SGD converges faster than SGD, there is still an advantage to use ROOT-SGD and our random-scaling estimator as compared to the ASGD counterpart. 
And we do see such an advantage in our experiment. 

\end{document}